\documentclass{article}
\usepackage{amsthm, amssymb, amsmath}
\usepackage{verbatim, float}
\usepackage{mathrsfs}
\usepackage{graphics}
\usepackage{subfig}
\usepackage[usenames,dvipsnames]{pstricks}
\usepackage{epsfig}
\usepackage{pst-grad} % For gradients
\usepackage{pst-plot} % For axes
\usepackage{cases}
\usepackage{multirow}
\usepackage{url}
\usepackage{array}
\usepackage{tabu}

\newcommand{\nc}{\newcommand}

\theoremstyle{definition}
 
\newtheorem{theorem}{Theorem}[section]

\newtheorem{definition}[theorem]{Definition}
\newtheorem{lemma}[theorem]{Lemma}

\theoremstyle{remark}

\newtheorem{example}{Example}[section]

\nc\remove[1]{}

\nc\bfa{{\boldsymbol a}}\nc\bfA{{\mathbf A}}\nc\cA{{\mathcal A}}
\nc\bfb{{\boldsymbol b}}\nc\bfB{{\mathbf B}}\nc\cB{{\mathcal B}}
\nc\bfc{{\boldsymbol c}}\nc\bfC{{\mathbf C}}\nc\cC{{\mathcal C}}
\nc\bfd{{\boldsymbol d}}\nc\bfD{{\mathbf D}}\nc\cD{{\mathcal D}}\nc\sD{{\mathscr D}}
\nc\bfe{{\boldsymbol e}}\nc\bfE{{\mathbf E}}\nc\cE{{\EuScript E}}
\nc\bff{{\boldsymbol f}}\nc\bfF{{\mathbf F}}\nc\cF{{\mathcal F}}
\nc\bfg{{\boldsymbol g}}\nc\bfG{{\mathbf G}}\nc\cG{{\mathcal G}}
\nc\bfh{{\boldsymbol h}}\nc\bfH{{\mathbf H}}\nc\cH{{\mathcal H}}
\nc\bfi{{\boldsymbol i}}\nc\bfI{{\mathbf I}}\nc\cI{{\mathcal I}}
\nc\bfj{{\boldsymbol j}}\nc\bfJ{{\mathbf J}}\nc\cJ{{\mathcal J}}
\nc\bfk{{\boldsymbol k}}\nc\bfK{{\mathbf K}}\nc\cK{{\mathcal K}}
\nc\bfl{{\boldsymbol l}}\nc\bfL{{\mathbf L}}\nc\cL{{\mathcal L}}\nc\sL{{\mathscr L}}
\nc\bfm{{\boldsymbol m}}\nc\bfM{{\mathbf M}}\nc\cM{{\mathcal M}}
\nc\bfn{{\boldsymbol n}}\nc\bfN{{\mathbf N}}\nc\cN{{\mathcal N}}
\nc\bfo{{\boldsymbol o}}\nc\bfO{{\mathbf O}}\nc\cO{{\mathcal O}}
\nc\bfp{{\boldsymbol p}}\nc\bfP{{\mathbf P}}\nc\cP{{\mathcal P}}
\nc\bfq{{\boldsymbol q}}\nc\bfQ{{\mathbf Q}}\nc\cQ{{\mathcal Q}}\nc\sQ{{\mathscr Q}}
\nc\bfr{{\boldsymbol r}}\nc\bfR{{\mathbf R}}\nc\cR{{\mathcal R}}
\nc\bfs{{\boldsymbol s}}\nc\bfS{{\mathbf S}}\nc\cS{{\mathcal S}}
\nc\bft{{\boldsymbol t}}\nc\bfT{{\mathbf T}}\nc\cT{{\mathcal T}}\nc\sT{{\mathscr T}}
\nc\bfu{{\boldsymbol u}}\nc\bfU{{\mathbf U}}\nc\cU{{\mathcal U}}
\nc\bfv{{\boldsymbol v}}\nc\bfV{{\mathbf V}}\nc\cV{{\mathcal V}}
\nc\bfw{{\boldsymbol w}}\nc\bfW{{\mathbf W}}\nc\cW{{\mathcal W}}\nc\sW{{\mathscr W}}
\nc\bfx{{\boldsymbol x}}\nc\bfX{{\mathbf Z}}\nc\cX{{\EuScript X}}
\nc\bfy{{\boldsymbol y}}\nc\bfY{{\mathbf Y}}\nc\cY{{\EuScript Y}}\nc\sY{{\mathscr Y}}
\nc\bfz{{\boldsymbol z}}\nc\bfZ{{\mathbf Z}}\nc\cZ{{\mathcal Z}}\nc\sZ{{\mathscr Z}}

\def\h_q{\qopname\relax{no}{h_q}}
\def\h{\qopname\relax{no}{h}}

 \usepackage[nonatbib, preprint]{nips_2018}

\usepackage[utf8]{inputenc} % allow utf-8 input
\usepackage[T1]{fontenc}    % use 8-bit T1 fonts
\usepackage{hyperref}       % hyperlinks
\usepackage{url}            % simple URL typesetting
\usepackage{booktabs}       % professional-quality tables
\usepackage{amsfonts}       % blackboard math symbols
\usepackage{nicefrac}       % compact symbols for 1/2, etc.
\usepackage{microtype}      % microtypography

\title{Revisiting Decomposable Submodular Function Minimization with Incidence Relations}

\author{
  Pan Li \\
  UIUC \\
  \texttt{panli2@illinois.edu} \\
  \And
    Olgica Milenkovic \\
  UIUC \\
  \texttt{milenkov@illinois.edu} \\
}

\begin{document}

\maketitle
\vspace{-0.5cm}
\begin{abstract}
We introduce a new approach to decomposable submodular function minimization (DSFM) that exploits incidence relations. Incidence relations describe which variables effectively influence the component functions, and when properly utilized, they allow for improving the convergence rates of 
DSFM solvers. Our main results include the precise parametrization of the DSFM problem based on incidence relations, the development of new scalable alternative projections and parallel coordinate descent methods and an accompanying rigorous analysis of their convergence rates.~\footnote{The code for this work can be found in https://github.com/lipan00123/DSFM-with-incidence-relations.}
\end{abstract}
\vspace{-0.5cm}
\section{Introduction}
\vspace{-0.2cm}
A set function $F: 2^{[N]} \rightarrow \mathbb{R}$ over a ground set $[N]$ is termed submodular if for all pairs of sets $S_1, S_2 \subseteq [N]$, one has $F(S_1) + F(S_2 )  \geq F(S_1\cap S_2 ) + F(S_1\cup S_2)$. Submodular functions capture the ubiquitous phenomenon of diminishing marginal costs~\cite{fujishige2005submodular} and they frequently arise as part of the objective function of various machine learning optimization problems~\cite{wei2015submodularity, li2017inhomogeneous, li2018submodular, kohli2009robust, lin2011class, krause2007near}.

%Submodularity is a property of set functions that captures the ubiquitous phenomenon of diminishing marginal costs~\cite{fujishige2005submodular}. As a result, submodular function optimization problems frequently arise in applications as diverse as machine learning, computer vision and signal processing~\cite{wei2015submodularity, li2017inhomogeneous, kohli2009robust, lin2011class, krause2007near}.

%~\cite{bach2013learning}, and in particular, graph and hypergraph partitioning and clustering~\cite{aissi2015strongly,chekuri2017computing}, active learning~\cite{wei2015submodularity}, MAP inference with higher-order potentials~\cite{kohli2009robust}, document summarization~\cite{lin2011class},  budget management~\cite{moulin2001strategyproof, 7756413} and sensor placement~\cite{krause2007near}. 

Among the various submodular function optimization problems, submodular function minimization (SFM), which may be stated as $\min_{S\subseteq[N]} F(S)$, is one of the most important and commonly studied questions. The current fastest known SFM algorithm has complexity $O(N^4 \log^{O(1)}N+ \tau N^3)$, where $\tau$ denotes the time needed to evaluate the submodular function~\cite{lee2015faster}.
Although SFM solvers operate in time polynomial in $N$, the high-degree of the underlying polynomial prohibits their use in practical large-scale settings. For this reason, a recent line of work has focused on developing scalable and parallelizable algorithms for solving the SFM problem by leveraging the property of \emph{decomposability}~\cite{stobbe2010efficient}. Decomposability asserts that the submodular function may be written as a sum of ``simpler'' submodular functions that may be optimized sequentially or in parallel. Formally, the underlying problem, referred to as decomposable SFM (DSFM), may be stated as: 
\begin{align}\label{DSFM}
\text{DSFM:} \quad \min_S \sum_{r\in [R]} F_{r}(S),
\end{align}
where $F_r: 2^{[N]} \rightarrow \mathbb{R}$ is a submodular function for all $r\in[R]$. Algorithmic solutions for the DSFM problem fall into two categories, combinatorial optimization approaches~\cite{kolmogorov2012minimizing,ene2017decomposable} and continuous function optimization methods~\cite{bach2013learning}.
In the latter setting, a crucial concept is the Lov$\acute{\text{a}}$sz extension of the submodular function which is 
convex~\cite{lovasz1983submodular} and lends itself to a norm-regularized convex optimization framework. %This continuous formulation appears by itself in a wide range of other applications, including semisupervised learning over hypergraphs~\cite{hein2013total, zhang2017re} and denoising~\cite{barbero2011fast}. 
Prior work in continuous DSFM has focused on devising efficient algorithms for solving the convex problem and deriving matching convergence results. The best known approaches include the alternating projection (AP) methods~\cite{jegelka2013reflection, nishihara2014convergence} and the coordinate descent (CD) methods~\cite{ene2015random}.
%Both algorithms require solving a collection of min-norm point problems over the base polytopes of the constituent functions in the decomposition, and the convergence rates of the techniques are characterized by the number of such problems one needs to solve in order to arrive at a $\epsilon$-optimal solution.  

Despite some simplifications offered through decomposibility, DSFM algorithms still suffer from scalability issues and have convergence guarantees that are suboptimal. To address the first issue, one needs to identify additional problem constraints that allow for parallel implementations. To resolve the second issue and more precisely characterize and improve the convergence rates, one needs to better understand how the individual submodular components jointly govern the global optimal solution. In both cases, it is crucial to utilize \emph{incidence relations} that describe which subsets of variables directly affect the value of any given component function. Often, incidences involve relatively small subsets of elements, which leads to desirable sparsity constraints. This is especially the case for min-cut problems on graphs and hypergraphs (where each submodular component involves two or several vertices)~\cite{karger1993global,chekuri2017computing} and MAP inference with higher-order potentials (where each submodular component involves variables corresponding to adjacent pixels)~\cite{stobbe2010efficient}. Although incidence relations have been used to parametrize the algorithmic complexity of combinatorial optimization methods for solving DSFM problems~\cite{kolmogorov2012minimizing}, they have been largely overlooked in continuous optimization methods. Some prior work considered merging decomposable parts with nonoverlapping  support into one submodular function, thereby creating a coarser decomposition that may be processed more efficiently~\cite{jegelka2013reflection, nishihara2014convergence, ene2015random}, but the accompanying algorithms were neither designed in a form that can optimally use this information nor analyzed precisely with respect to their convergence rates and merging strategies. In an independent work, Djolonga and Krause found that the variational inference problem in L-FIELD reduced to a DSFM problem with sparse incidence relations~\cite{djolonga2015scalable}, while their analysis only worked for regular cases and the obtained results were not as tight as those in this work.

Here, we revisit two benchmark algorithms for continuous DSFM -- AP and CD -- and describe how to modify them to exploit incidence relations that allow for significantly improved computational complexity. Furthermore, we provide a complete theoretical analysis of the algorithms parametrized by incidence relations with respect to their convergence rates. 
%to design the efficient algorithms and derive parameterization by exploiting incidence relations. The convergence rates of these algorithms are characterized as the numbers min-norm problems of one decomposed part to achieve an $\epsilon-$optimal solution~\cite{nishihara2014convergence, ene2015random, ene2017decomposable}. In this work,  a similar metric will be used. 
%We also observe that local incidence relations are important in parallel version of these algorithms where multiple min-norm point problems are allowed to be solved simultaneously within one iteration. So the convergence rates for parallel AP and CDM are also derived.
AP-based methods that leverage incidence relations achieve better convergence rates than classical AP algorithms both in the sequential and parallel optimization scenario. The random CD method (RCDM) and accelerated CD method (ACDM) that incorporate incidence information can be parallelized. The complexity of sequential CD methods cannot be improved using incidence relations, but the convergence rate of parallel CD methods strongly depends on how the incidence relations are used for coordinate sampling: while a new specialized combinatorial sampling based on equitable coloring~\cite{meyer1973equitable} is optimal, uniformly at random sampling produces a $2$-approximation. It also leads to a greedy method that empirically outperforms random sampling. A summary of these and other findings is presented in Table~\ref{tab:results}. 
 \begin{table*}[h] 
\begin{tabular}{|p{0.9cm}<{\centering}|p{1.3cm}<{\centering}|p{1.8cm}<{\centering}|p{1.8cm}<{\centering}|p{5.8cm}<{\centering}|}
\hline
\multirow{ 2}{*}{} & \multicolumn{2}{c|}{Prior work} & \multicolumn{2}{c|}{This work} \\
\cline{2-5}
& Sequential & Parallel & Sequential & Parallel \\
\hline
AP & $O(N^2R^2)$ & $O(N^2R^2/K)$ & $O(N\|\mu\|_1R)$ &  $O(N\|\mu\|_1R/K)$ \\
\hline
RCDM & $O(N^2R)$& - & $O(N^2R)$ & $O\left(\left(\frac{R-K}{R-1} N^2 +  \frac{K-1}{R-1} N\|\mu\|_1\right)R/K\right)$\\
\hline
ACDM & $O(NR)$ & - & $O(NR)$ & $O\left(\left(\frac{R-K}{R-1} N^2 +  \frac{K-1}{R-1} N\|\mu\|_1\right)^{1/2}R/K\right)$ \\
\hline
\end{tabular}
\centering 
\caption{Overview of known and new results: each entry contains the required number of iterations to achieve an $\epsilon$-optimal solution (the dependence on $\epsilon$ is the same for all algorithms and hence omitted). Here, $\|\mu\|_1 = \sum_{i\in[N]} \mu_i,$ where for all $i\in[N]$, $\mu_i$ equals the number of submodular functions that involve element $i$; $K$ is a parallelization parameter that equals the number of min-norm points problems that have to be solved within each iteration.}
\label{tab:results}
\vspace{-0.6cm}
\end{table*}
%~\cite{nishihara2014convergence}~\cite{nishihara2014convergence}~\cite{ene2015random}~\cite{ene2015random}  
%The paper is organized as follows. Section~\ref{sec:preliminaries} contains the mathematical preliminaries and describes the DSFM problem and related algorithmic approaches. The section also formally introduces the notion of an incidence relation and incidence matrix. Section~\ref{sec:algs} contains the main results of this work, a collection of new continuous DSFM algorithms that exploit the incidence structure of the problem. Experimental findings are reported in Section~\ref{sec:experiments}. The proofs of all claims can be found in the Appendix. 
%\vspace{-0.6cm}
\section{Background, Notation and Problem Formulation}  \label{sec:preliminaries}
%A submodular function $F$ is said to be decomposable if there exists a collection of submodular functions $\{F_r\}_{1\leq r \leq R}$, $F_r: 2^{[N]} \rightarrow \mathbb{R}$ such that $F(S) =  \sum_{r\in [R]} F_{r}(S)$. In this case, the SFM optimization problem may be written as
%\begin{align}\label{DSFM}
%\text{DSFM:} \quad \min_S \sum_{r\in [R]} F_{r}(S),
%\end{align}
%and we refer to it as decomposable submodular function minimization (DSFM). Without loss of generality, we tacitly assume that all submodular functions $F_r$ are normalized, i.e., that $F_r(\emptyset)=0$ for all $r \in [R]$.
\vspace{-0.1cm}
We start our exposition by reviewing several recent lines of work for solving the DSFM problem, and focus on approaches that transform the DSFM problem into a continuous optimization problem. Such approaches exploit the fact that the Lov$\acute{\text{a}}$sz extension of a submodular function is convex. Without loss of generality, we tacitly assume that all submodular functions $F_r$ are normalized, i.e., that $F_r(\emptyset)=0$ for all $r \in [R]$. Also, we define given a vector $z\in \mathbb{R}^{N}$ and $S\subseteq [N]$, $z(S) = \sum_{i\in S} z_i$. Then, the \emph{base polytope} of the $r$-th submodular function $F_r$ is defined as 
\begin{align*}
\mathcal{B}_r \triangleq \{y_r\in\mathbb{R}^N|&y_r(S)\leq F_r(S),\;\text{for any }S\subset [N], \text{and}\;y_r([N])=F_r([N])\}.
\end{align*}
The \emph{Lov$\acute{\text{a}}$sz extention}~\cite{lovasz1983submodular} $f_r(\cdot): \mathbb{R}^N\rightarrow \mathbb{R}$ of a submodular function $F_r$ is defined as $f_r(x) = \max_{y_r \in \mathcal{B}_r} \langle y_r, x\rangle,$
where $\langle \cdot, \cdot \rangle$ denotes the inner product of two vectors. The DSFM problem can be solved through continuous optimization, $\min_{x\in [0,1]^N} \sum_{r} f_{r}(x)$. To counter the nonsmoothness of the objective function, a proximal formulation of a generalization of the above optimization problem is considered instead~\cite{jegelka2013reflection},
 \begin{align}\label{smooth}
\min_{x\in \mathbb{R}^N} \sum_{r\in [R]} f_{r}(x)  + \frac{1}{2} \|x\|_2^2.
\end{align}
%~\cite{stobbe2010efficient,jegelka2013reflection,nishihara2014convergence,ene2015random,ene2017decomposable}
%The aforementioned proximal formulation~\eqref{smooth} is considered instead~\cite{jegelka2013reflection},
% \begin{align}\label{smooth}
%\min_{x\in \mathbb{R}^N} \sum_{r\in [R]} f_{r}(x)  + \frac{1}{2} \|x\|_2^2 = \min_{x\in \mathbb{R}^N} \sum_{r} (f_{r}(x)  + \frac{1}{2R} \|x\|_2^2).
%\end{align}
As the problem~\eqref{smooth} is strongly convex, it has a unique optimal solution, denoted by $x^*$. The exact discrete solution to the DSFM problem equals $S^* = \{i\in [N] | \, x_i^* > 0\}$. 
%As previous works do not take care of the incidence relations, so the used notations are basically same as the case when the incidence relations are full, i.e., $H$ is all-one matrix.  

For convenience, we denote the product of base polytopes as $\mathcal{B} = \otimes_{r=1}^R \mathcal{B}_r$, and write $y = (y_1, y_2, ...,y_R) \in \mathcal{B}$. Also, we let $A$ be a simple linear mapping $\otimes_{r=1}^R \mathbb{R}^N \rightarrow \mathbb{R}^N,$ which given a point $a =  (a_1, a_2, ...,a_R) \in \otimes_{r=1}^R \mathbb{R}^N$ outputs $Aa =\sum_{r\in [R]} a_r$. 
The AP and CD algorithms for solving~\eqref{smooth} use the dual form of the problem, described in the next lemma. 
\begin{lemma}[\cite{jegelka2013reflection}] \label{lem:jeg}
The dual problem of~\eqref{smooth} reads as 
\begin{align}\label{distance}
\min_{a,y} \|a-y\|_2^2 \quad \text{s.t.} \quad Aa= 0,\; y\in \mathcal{B}.
\end{align}
Moreover, problem \eqref{distance} may be written in the more compact form
\begin{align}\label{compact}
\min_{y} \|Ay\|_2^2 \quad \text{s.t.} \quad y\in \mathcal{B}.
\end{align}
For both problems, the primal and dual variables are related according to $x = - Ay$. In what follows, for notational simplicity, we write $g(y) = \frac{1}{2} \|Ay \|_2^2$.
\end{lemma}

The AP~\cite{nishihara2014convergence} and RCD algorithms~\cite{ene2015random} described below provide solutions to the problems~\eqref{distance} and~\eqref{compact}, respectively. They both rely on repeated projections $\Pi_{\mathcal{B}_r}(\cdot)$ onto the base polytopes $\mathcal{B}_r$, $r\in[R]$. These projections are typically less computationally intense than projections onto the complete base polytope of $F$ as they involve fewer data dimensions. The projection operation $\Pi_{\mathcal{B}_r}(\cdot)$ requires one to solve a min-norm problem by either exploiting the special forms of $F_r$ or by using the general purpose algorithm of Wolfe~\cite{wolfe1976finding}. The complexity of the method is typically characterized by the number of required projections $\Pi_{\mathcal{B}_r}(\cdot)$. 
 
\textbf{The AP algorithm.} Starting with $y=y^{(0)}$, iteratively compute a sequence $(a^{(k)}, y^{(k)})_{k=1,2,...}$ such that for all $r\in[R]$,
$a_r^{(k)} = y_r^{(k-1)} -  Ay^{(k-1)}/R$, $y_r^{(k)} = \Pi_{\mathcal{B}_r}(a_r^{(k)}),$ until a stopping criteria is met.

\textbf{The RCDM algorithm.} In each iteration $k$, chose uniformly at random a subset of elements in $y$ associated with one atomic function in the decomposition~\eqref{DSFM}, say the one with index $r_k$. Then, compute the sequence $(y^{(k)})_{k=1,2,...}$ according to $y_{r_k}^{(k)} = \Pi_{B_{r_k}}\left(- \sum_{r\neq r_k }y_r^{(k-1)}\right)$, $y_r^{(k)} = y_r^{(k-1)},$ for $r\neq r_k$. 

Finding an $\epsilon$-optimal solution for both the AP and RCD methods requires $O(N^2 R \log(\frac{1}{\epsilon}))$ iterations. In each iteration, the AP algorithm computes the projections onto all $R$ base polytopes, while the RCDM only computes one projection. Therefore, as may be seen from Table~\ref{tab:results}, the sequential AP solver, which computes one projection in each iteration, requires $O(N^2 R^2 \log(\frac{1}{\epsilon}))$ iterations. However, the projections within one iteration of the AP method can be generated in parallel, while the projections performed in the RCDM have to be generated sequentially.
\vspace{-0.2cm}
\subsection{Incidence Relations and Related Notations}
% \textcolor{blue}{Effective support or Incidence relation?}

We next formally introduce one of the key concepts used in this work: \emph{incidence relations} between elements of the ground set and the component submodular functions. 

We say that an element $i\in [N]$ is \emph{incident} to a submodular function $F$ iff there exists a $S\subseteq  [N]/\{i\}$ such that $F(S\cup \{i\}) \neq F(S)$; similarly, we say that the submodular function $F$ is \emph{incident} to an element $i$ iff $i$ is incident to $F$. To verify whether an element $i$ is incident to a submodular function $F$, one needs to verify that $F(\{i\}) = 0$ and that $F( [N]) = F([N]/\{i\})$ since for any $S\subseteq  [N]/\{i\}$
\begin{align*}
F(\{i\}) \geq F(S\cup \{i\}) - F(S) \geq F( [N]) - F([N]/\{i\}).
\end{align*}
Furthermore, note that if $i\in[N]$ is not incident to $F_r$, then for any $y_r\in \mathcal{B}_r$, one has $y_{r,i} = 0$. Let $S_r$ be the set of all elements incident to $F_r$. For each element $i$, denote the number of submodular functions that are incident to $i$ by $\mu_i = |\{r\in[R]: i \in S_r\}|$. We also refer to $\mu_i$ as the degree of element $i$. We find it useful to partition the set of submodular functions into different groups. Given a group $C\subseteq [R]$ of submodular functions, we define the degree of the element $i$ within $C$, $\mu^C_{i}$, as $\mu^C_{i} = |\{r\in C: i \in S_r\}|$.

We also define a skewed norm involving two vectors $w\in \mathbb{R}_{>0}^{N}$ and $z\in\mathbb{R}^{N}$ according to $\|z\|_{2,w} \triangleq \sqrt{\sum_{i\in [N]} w_{i}z_{i}^2}$. With a slight abuse of notation, for two vectors $\theta = (\theta_1, \theta_2, ..., \theta_R) \in \otimes_{r=1}^R\mathbb{R}_{> 0}^N$ and $y\in\otimes_{r=1}^R\mathbb{R}^N$, we also define the norm $\|y\|_{2,\theta} \triangleq \sqrt{\sum_{r\in [R]} \|y_{r}\|_{2,\theta_r}^2}$. Which of the norms we refer to should be clear from the context. In addition, we let $\|\theta\|_{1,\infty} = \sum_{i\in[N]} \max_{r\in[R]: i\in S_r}\theta_{r,i}$. For a closed set $\mathcal{K} \subseteq  \otimes_{r=1}^R \mathbb{R}^N$ and a positive vector $\theta \in \otimes_{r=1}^R\mathbb{R}_{> 0}^N$, the distance between $y$ and $\mathcal{K}$ is defined as $d_{\theta}(y, \mathcal{K}) = \min \{\|y- z\|_{2,\theta} | z \in \mathcal{K}\}$.  Also, given a set $\Omega\subseteq \mathbb{R}^N$, we let $\Pi_{\Omega, w}(\cdot)$ denote the projection operation onto $\Omega$ with respect to the norm $\|\cdot\|_{2, w}$. 

Given a vector $w\in \mathbb{R}_{>0}^N $, we also make use of an induced vector $I(w) \in \otimes_{r=1}^R\mathbb{R}^N$ whose $r$-th entry satisfies $(I(w))_r = w$. It is easy to check that $\|I(w)\|_{1,\infty}=\|w\|_{1}$. Of special interest are induced vectors based on pairs of  $N$-dimensional vectors, $\mu=(\mu_1, \mu_2,...,\mu_N)$, $\mu^C=(\mu^C_1, \mu^C_2,..., \mu^C_N)$. Finally, for $w, w'\in \mathbb{R}^{N}$, we denote the element-wise power of $w$ by $w^{\alpha} = (w_1^{\alpha}, w_2^{\alpha}, ..., w_N^{\alpha} )$, for some $\alpha\in \mathbb{R},$ and the element-wise product of $w$ and $w'$ by $w\odot w' = (w_1w'_1, w_2w'_2,...,w_Nw'_N)$.

Next, recall that $x^*$ is the unique optimal solution of the problem~\eqref{smooth} and let $\mathcal{Z} = \{\xi\in \otimes_{r=1}^R \mathbb{R}^N|A\xi =-x^*, \xi_{r,i} = 0, \forall i\in S_r, \forall r\in[R]\}$. Then, due to the duality relationship of Lemma~\ref{lem:jeg}, $\Xi =\mathcal{Z}  \cap \mathcal{B}$ is the set of optimal solutions $\{{y\}}$. 
%We make use of these entities in the derivations to follow. 

\vspace{-0.2cm}
\section{Continuous DSFM Algorithms with Incidence Relations} \label{sec:algs}
\vspace{-0.1cm}
In what follows, we revisit the AP and CD algorithms and describe how to improve their performance and analytically establish their convergence rates. Our first result introduces a modification of the AP algorithm~\eqref{distance} that exploits incidence relations so as to decrease the required number of iterations from $O(N^2R)$ to $O(N\|\mu\|_1)$.  Our second result is an example that shows that the convergence rates of CD algorithms~\cite{ene2017decomposable} cannot be directly improved by exploiting the functions' incidence relations even when the incidence matrix is extremely sparse. Our third result is a new algorithm that relies of coordinate descent steps but can be parallelized. In this setting, incidence relations are essential to the parallelization process. 
%We start with a preliminary discussion that establishes the basis of our analyses of the new DSFM algorithms.
%\vspace{-0.1cm}
%\subsection{Preliminary statements}

To analyze solvers for the continuous optimization problem~\eqref{smooth} that exploit the incidence structure of the functions, we make use of the skewed norm $\|\cdot\|_{2,w}$ with respect to some positve vector $w$ that accounts for the fact that incidences are, in general, nonuniformly distributed. In this context, the projection $\Pi_{\mathcal{B}_r, w}(\cdot)$ reduces to solving a classical min-norm problem after a simple transformation of the underlying space which does not incur significant complexity overheads. To see this, note that in order to solve a generic min-norm point problem, one typically uses either Wolfe's algorithm (continuous) or a divide-and-conquer procedure (combinatorial). The complexity of the former is at most quadratic in $F_{r,\text{max}}\triangleq\max_{v, S} |F_r(S\cup \{v\})-F_r(S)|$~\cite{chakrabarty2014provable}, while the complexity of the latter merely depends on $\log F_{r,\text{max}}$~\cite{jegelka2013reflection} (see Section~\ref{sec:discdiscreteopt} in the Supplement). It is unclear if including the weight vector $w$ into the projection procedure increases or decreases $F_{r,\text{max}}$. In either case, given that in our derivations all elements of $w$ are contained in $[1, \max_{i\in[N]}\mu_i]$ instead of $N$ or $R$, we do not expect to see significant changes in the complexity of the projection operation. Hence, throughout the remainder of our exposition, we regard the projection operation as an oracle and measure the complexity of all algorithms in terms of the number of projections performed.

Also, observe that one may avoid computing projections in skewed-norm spaces by introducing in~\eqref{smooth} a weighted rather than an unweighted proximal term. This gives another continuous objective that still provides a solution to the discrete problem~\eqref{DSFM}. Even in this case, we can prove that the numbers of iterations used in the different methods listed Table~\ref{tab:results} remain the same. Furthermore, by combining projections in skewed-norm spaces and weighted proximal terms, it is possible to actually reduce the number of iterations given in Table~\ref{tab:results}. However, for simplicity, we focus on the objective~\eqref{smooth} and projections in skewed-norm spaces. Methods using weighted proximal terms with and without skewed-norm projections are analyzed in a similar manner in Section~\ref{sec:orthproj} of the Supplement.

%We leave more investigation of balancing the complexity of this inner-loop projection and outer-loop iterations to the future work. 
%In a very recent paper, Ene et al. did an extensive discussion about comparing or combining discrete and continuous methods to solve DSFM problem~\cite{ene2017decomposable}. The authors refer this paper to the interested readers. 
We make frequent use of the following result which generalizes Lemma 4.1 of~\cite{ene2017decomposable}.\begin{lemma}\label{lemmakappabound}
Let $\theta\in \otimes_{r=1}^R\mathbb{R}_{> 0}^{N}, w\in \mathbb{R}_{> 0}^{N}$ be two positive vectors. Let $y\in \mathcal{B}$ and let $z$ be in the base polytope of the submodular function $F$. Then, there exists a point $\xi\in  \mathcal{B}$ such that $A\xi =z$ and $\|\xi-y\|_{2, \theta}\leq \sqrt{\frac{\|\theta\|_{1,\infty}}{2}} \|Ay - z\|_1$. Moreover, $\|\xi-y\|_{2, \theta}\leq \sqrt{\frac{\|\theta\|_{1,\infty}\|w^{-1}\|_1}{2}} \|Ay - z\|_{2, w}$.
% \textcolor{blue}{ As the base polytope of the submodular function $F$ is not the same as $\mathcal{B}$ that is a product of base polytopes. And here, $z$ is in the base polytope of the submodular function $F$ instead of $\mathcal{B}$}
\end{lemma}
\vspace{-0.1cm}
\subsection{The Incidence Relation AP (IAP)}
The following result establishes the basis of our improved AP method leveraging incidence structures. 
\begin{lemma}\label{dualprob}
The following problem is equivalent to problem~\eqref{distance}:
\begin{align}\label{newdistance}
\min_{a,y} \|a-y\|_{2, I(\mu)}^2 \quad \text{s.t.} \quad y\in \mathcal{B}, Aa = 0,\;\text{and}\; a_{r,i} = 0,\; \forall (r, i): i\notin 
S_r, r\in[R]. \; 
\end{align}
%where $\mathcal{A} = \{a \in \mathbb{R}^{\delta^s}| Aa = 0\}$. 
\end{lemma} 
%Equation~\eqref{newdistance} describes an approximation problem with respect to a skewed norm that depends on the the number of incident submodular functions of the elements, as opposed to formula~\eqref{distance}. 

Let $\mathcal{A} = \{a \in \otimes_{r=1}^R\mathbb{R}^{N}| Aa = 0, a_{r,i} = 0,\; \forall (r, i): i\notin S_r\} $ and $\mathcal{A'} = \{a \in \otimes_{r=1}^R\mathbb{R}^{N}| Aa = 0\}$.
The AP algorithm for problem~\eqref{newdistance} consists of alternatively computing projections between $\mathcal{A}$ and $\mathcal{B}$, as opposed to those between $\mathcal{A'}$ and $\mathcal{B}$ used in the problem~\eqref{distance}. However, as already pointed out, unlike for the classical AP problem~\eqref{distance}, the distance in~\eqref{newdistance} is not Euclidean, and hence the projections may not be orthogonal. 

The IAP method for solving~\eqref{newdistance} proceeds as follows. We begin with $a = a^{(0)}\in \mathcal{A}$, and iteratively compute a sequence $(a^{(k)}, y^{(k)})_{k=1,2,...}$ as follows: for all $r\in[R]$,
%\begin{align*}
%a^{(k)} = \Pi_{\mathcal{A},D'}(y^{(k-1)})=\arg\min_{a\in \mathcal{A}}\|a - y^{(k-1)}\|_{2,D'},\; y^{(k)} = \Pi_{\mathcal{B},D'}(a^{(k)})=\arg\min_{y\in \mathcal{B}}\|y - a^{(k)}\|_{2,D'}
$y_r^{(k)} = \Pi_{\mathcal{B}_r, \mu}(a_r^{(k)}), \; a_{r,i}^{(k)} = y_{r,i}^{(k-1)} - \mu_i^{-1}(Ay^{(k-1)})_i, \;\forall\; i\in S_r$.
%\end{align*}
%where $\Pi_{S,W}$ denotes the oblique projection operation onto the space $S$ with respect to the norm $\|\cdot\|_{2, W}$ (oblique projections, as opposed to orthogonal projections, minimize the Mahalanobis rather than Euclidean distance between a point and a set).%Specifically, the first projection is simple as $\Pi_{\mathcal{A},D'}(y) = (I - \tilde{A}^T \tilde{A})y$. The second projection can be transformed as a collection of oblique projections to the base polytopes for submodular functions: for $r\in[R]$, $y_r^{(k)}=\Pi_{\mathcal{B}_r,D_r'}(a_r^{(k)})$. 
The key difference between the AP and IAP algorithms is that the latter effectively removes ``irrelevant'' components of $y_r$ by fixing the irrelevant components of $a$ to $0$. In the AP method of Nishihara~\cite{nishihara2014convergence}, these components are never zero as they may be ``corrupted'' by other components during AP iterations. Removing irrelevant components results in projecting $y$ into a subspace of lower dimensions, which significantly accelerates the convergence of IAP. 

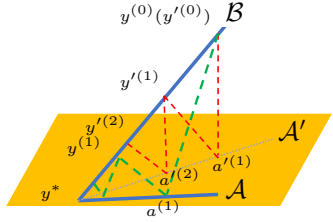
\begin{figure}[h]
\centering
\vspace{-0.2cm}
\includegraphics[trim={6.5cm 12cm 19.5cm 3cm},clip, width=0.33\textwidth]{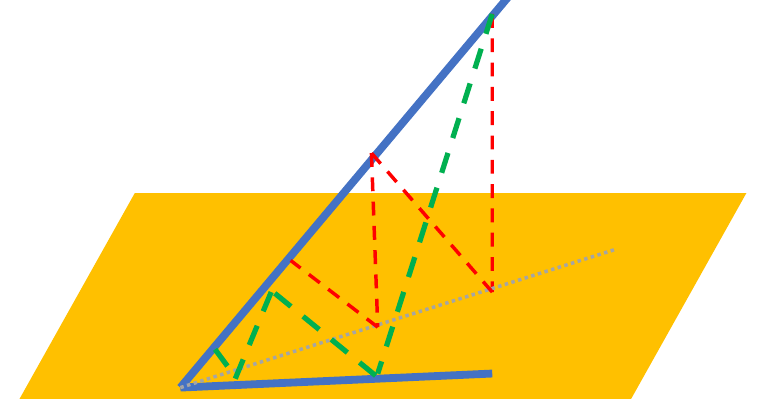}
\put(-25,25){$\mathcal{A'}$}
\put(-45,3){$\mathcal{A}$}
\put(-45,70){$\mathcal{B}$}
\put(-85,70){\tiny{$y^{(0)}(y'^{(0)})$}}
\put(-85,45){\tiny{$y'^{(1)}$}}
\put(-98,28){\tiny{$y'^{(2)}$}}
\put(-105,20){\tiny{$y^{(1)}$}}
\put(-115,3){\tiny{$y^*$}}
\put(-50,13){\tiny{$a'^{(1)}$}}
\put(-70,8){\tiny{$a'^{(2)}$}}
\put(-75,-4){\tiny{$a^{(1)}$}}
\caption{Illustration of the IAP method for solving problem~\eqref{newdistance}: The space $\mathcal{A}$ is a subspace of $\mathcal{A'}$, which leads to faster convergence of the IAP method when compared to AP.}
\label{APillu}
\vspace{-0.0cm}
\end{figure}

The analysis of the convergence rate of the IAP method follows a similar outline as that used to analyze~\eqref{distance} in~\cite{nishihara2014convergence}. Following Nishihara et al.~\cite{nishihara2014convergence}, we define the following parameter that plays a key role in determining the rate of convergence of the AP algorithm, $\kappa_* \triangleq \sup\limits_{y\in\mathcal{Z} \cup \mathcal{B} / \Xi}\frac{d_{I(\mu)}(y, \Xi)}{\max\{d_{I(\mu)}(y,\mathcal{Z}), d_{I(\mu)}(y, \mathcal{B})\}}$.
%\begin{align*}
%&\kappa(y)  \triangleq \frac{d_{D'}(y, \Xi)}{\max\{d_{D'}(y,\mathcal{Z}), d_{D'}(y, \mathcal{B})\}}, \\
%&\kappa_* \triangleq \sup\{\kappa(y) : y\in\mathcal{Z} \cup \mathcal{B} / \Xi \}.
%\end{align*}
\begin{lemma}[\cite{nishihara2014convergence}]
If $\kappa_*<\infty$, the AP algorithm converges linearly with rate $1- \frac{1}{\kappa_*^2}$. At the $k$-th iteration, the algorithm outputs a value $y^{(k)}$ that satisfies
\begin{align*}
d_{I(\mu)}(y^{(k)}, \Xi)  \leq 2d_{I(\mu)}(y^{(0)}, \Xi) \left(1- \frac{1}{\kappa_*^2}\right)^k.
\end{align*}
\end{lemma}
To apply the above lemma in the IAP setting, one first needs to establish an upper bound on $\kappa_{*}$. This bound is given in Lemma~\ref{kappabound} below. 
\begin{lemma}\label{kappabound}
The parameter $\kappa_*$ is upper bounded as $\kappa_*\leq \sqrt{N\|\mu\|_1 /2} +1$. 
\end{lemma}
By using the above lemma and the bound on $\kappa_{*}$, one can establish the following convergence rate for the IAP method.
\begin{theorem}\label{APconv}
After $O(N\|\mu\|_1 \log(1/\epsilon))$ iterations, the IAP algorithm for solving problem~\eqref{newdistance} outputs a pair of points $(a , y)$ that satisfies $d_{I(\mu)}(y, \Xi) \leq \epsilon$.
\end{theorem}
Note that in practice, one often has $\|\mu\|_1\ll NR,$ which shows that the convergence rate of the AP method for solving the DSBM problem may be significantly improved. 
\vspace{-0.2cm}
\subsection{Sequential Coordinate Descent Algorithms}
\vspace{-0.1cm}
Unlike the AP algorithm, the CD algorithms by Ene et al.~\cite{ene2015random} remain unchanged given~\eqref{compact}. Our first goal is to establish whether the convergence rate of the CD algorithms can be improved using a parameterization that exploits incidence relations.

The convergence rate of CD algorithms is linear if the objective function is component-wise smooth and $\ell$-strong convex. In our case, $g(y)$ is component-wise smooth as for any $y, z\in \mathcal{B}$ that only differ in the $r$-th block (i.e., $y_r\neq z_r$, $y_{r'}= z_{r'}$ for $r'\neq r$), one has
\begin{align}\label{smoothness}
 \| \nabla_{r} g(y)- \nabla_r g(z) \|_2 \leq \|y - z\|_{2}.
\end{align}
Here, $\nabla_{r} g$ denotes the gradient vector associated with the $r$-th block. 

\begin{definition}
We say that the function $g(y)$ is \emph{$\ell$-strongly convex} in $\|\cdot\|_{2,}$, if for any $y\in \mathcal{B}$ 
\begin{align}\label{strongconv}
g(y^*) \geq g(y) + \langle \nabla g(y), y^*-y \rangle + \frac{\ell}{2}\|y^*- y\|_{2}^2, \; \text{or equivalently,} \quad \|Ay - Ay^*\|_2^2 \geq \ell \|y^*- y\|_{2}^2, \nonumber
\end{align}
where $y^* = \arg\min\limits_{z\in\Xi}\|z- y\|_{2}^2$. Moreover, we let $\ell_{*} = \sup\{\ell: \text{$g(y)$ is $\ell$-strongly convex in } \;\|\cdot\|_{2}\}$. 
\end{definition}
Note that the above definition essentially establishes a form of weak-strong convexity~\cite{karimi2016linear}. %, since $(y,\,y^*)$ is not an arbitrary pair of points from the feasible region. 
%However, this difference does not hurt the convergence rate of CDMs. %Moreover, we emphasize here the strong convexity is in $\|\cdot\|_{2, \mathcal{L}}$, as 
Then, using standard analytical tools for CD algorithms~\cite{nesterov2012efficiency}, we can prove the following result~\cite{ene2015random}. 
\begin{theorem}\label{RCDBconv}
The RCDM for problem~\eqref{compact} outputs a point $y$ that satisfies $ \mathbb{E}[g(y)] \leq  g(y^*) + \epsilon$ after $O(\frac{R}{\ell_{*}}\log(1/\epsilon))$ iterations. The ACDM applied to the problem~\eqref{compact} outputs a point $y$ that satisfies $\mathbb{E}[g(y)] \leq  g(y^*) + \epsilon$ after $O(\frac{R}{\sqrt{\ell_{*}}}\log(1/\epsilon))$ iterations. 
\end{theorem}

To precisely characterize the convergence rate, we need to find an accurate estimate of $\ell_*$. Ene et al.~\cite{ene2017decomposable} derived $\ell_* \geq \frac{1}{N^2}$ without taking into account the incidence structure. As sparse incidence side information improves the performance of the AP method, it is of interest to determine if the same can be accomplished for the CD algorithms. Example~\ref{boundofell} establishes that this is not possible in general if one only relies on $\ell_{*}$.  
 \begin{example}\label{boundofell}
Consider a DSFM problem with a extremely sparse incidence structure with $|S_r|= 2$. More precisely, let $N=2n+1$, $R = 2n$, and $\|\mu\|_1 = \sum_{r\in[R]}|S_r|= 4n\ll NR$. Let $F_r$ be incident to the elements $\{r, r+1\},$ for all $r\in[R],$ and be such that $F_r(\{r\}) = F_r(\{r+1\}) = 1, F_r(\emptyset) = F_r(\{r, r+1\}) = 0$. Then, $\ell_{*}< \frac{7}{N^2}$. 
\end{example}

Note that the optimal solution of problem~\eqref{compact} for this particular setting equals $y^* = 0$. Let us consider a point $y\in \mathcal{B}$ specified as follows. First, due to the given incidence relations, the block $y_r$ has two components corresponding to the elements indexed by $r$ and $r+1$. For any $r\in [R]$,
\begin{equation} \label{exampleinit}
   y_{r, r} = - y_{r, r+1} = \left\{\begin{array}{cc} \frac{r}{n} & r\leq n, \\
   \frac{2n+1-r}{n} &  r\geq n+1.
    \end{array} \right.
\end{equation}
Therefore, $g(y) = \frac{1}{n},\;\|y\|_2^2 > \frac{4}{3}n$, which results in $\ell_* < \frac{3}{2n^2}\leq \frac{7}{N^2}$ for all $n\geq 3$. 

Example~\ref{boundofell} only illustrates that an important parameter of CDMs cannot be improved using incidence information; but this does not necessarily imply that a sequential RCDM that uses incidence structures cannot offer better convergence rates than $O(N^2R)$. In Section~\ref{sec:lowerboundexp} of the Supplement, we present additional experimental evidence that supports our observation, using the setting of Example~\ref{boundofell}.  

As a final remark, note that Nishihara et al.~\cite{nishihara2014convergence} also proposed a lower bound that does not make use of sparse incidence structures and only works for the AP method.

%Example~\ref{boundofell} only illustrates that an important parameter of CDMs cannot be improved using incidence information; but this does not necessarily imply that a sequential RCDM that uses incidence structures cannot offer better convergence rates than $O(N^2R)$. Therefore, we provide additional empirical evidence that the convergence result suggested by the bound on $\ell_* \leq \frac{7}{N^2}$ is correct. We constructed a DSFM problem following Example 3.1 and initialized $y$ according to equation (11). We used the number of iterations $k$ required to attain $g(y^{(k)}) \leq \epsilon g(y^{(0)})$ as a measure for the speed of convergence. We ran the simulations for $n\in[5,50]$ and averaged the results for each $n$ over $10$ independent runs. Figure~\ref{lower-bound-figure} shows the results. The values next to the curves are their slopes obtained via a linear regression involving $\ln (\text{\# Iterations})\sim \ln(N)$. As the accuracy threshold increases, the slope approaches the value $3$, which indicates the required number of iterations equals $O(N^2 R)$.     
%\begin{figure*}[t]
%\centering
%\includegraphics[trim={0cm 0cm 0cm 0cm},clip,width=.4\textwidth]{lowerboundfigure-eps-converted-to.pdf}
%\caption{Simulations accompanying Example 3.1: $\ln$(the number of iterations) vs $\ln(N)$. }
%\label{lower-bound-figure}
%\end{figure*}

\vspace{-0.1cm}
\subsection{New Parallel CD methods}\label{sec:CDMs}
%Recall that the AP method described in the previous section requires performing projections onto all base polytopes $\mathcal{B}_r$ during each iteration; but these projections can be computed in parallel. 
%Unlike IAP, CDM do not offer better convergence rates. If $H$ is full, no matter how to parallelize CDMs, one can better convergence rate. I mentioned this as I want to emphasize that incidence relation not only is critical parameter in parallel CDMs but also parallelization only works if there is sparse incidence relation. By comparing the convergence rates of these two types of algorithms, it is possible to conclude that if one is allowed to execute more than $\bar{d} \triangleq \sum_{v} d_v/N$ projections in parallel, AP will converge faster than RCDM. 
\vspace{-0.1cm}
In what follows, we propose two CDMs which rely on parallel projections and incidence relations.

The following observation is key to understanding the proposed approach. Suppose that we have a nonempty group of blocks $C\subseteq R$. Let $y, h\in \otimes_{r=1}^R\mathbb{R}^{N}$. If $h_{r,i}$ is nonzero only for block $r\in C$ and $i\in S_r$, then,
\begin{align}
g(y+h) = g(y) + \langle \nabla g(y), h\rangle + \frac{1}{2}\|Ah\|_{2}^2 \leq\; g(y) + \sum_{r\in C} \langle \nabla_r g(y),  h_{r}\rangle + \sum_{r\in C} \frac{1}{2}\| h_{r}\|_{2, I(\mu^C)}^2 . \label{paradescent}
\end{align}
Hence, for all $r\in C$, if we perform projections onto $\mathcal{B}_r$ with respect to the norm $\|\cdot\|_{2, \mu^C}$ simultaneously in each iteration of the CDM, convergence is guaranteed as the value of the objective function remains bounded. The smaller the components of $\mu^C$, the faster the convergence. Note that the components of $\mu^C$ are the numbers of incidence relations of elements restricted to the set $C$. Hence, in each iteration, blocks that ought to be updated in parallel are those that correspond to submodular functions that have supports with smallest possible intersections. 

One can select blocks that are to be updated in parallel in a combinatorially specified fashion or in a randomized fashion, as dictated by what we call an $\alpha$-proper distribution. To describe our parallel RCDM, we first introduce the notion of an $\alpha$-proper distribution. %Define the required distribution to sample group of blocks. 
\begin{definition}
Let $P$ be a distribution used to sample a group of $C$ blocks. Define $\theta^P = (\theta_1^P, \theta_2^P, ..., \theta_R^P)$ such that for $r\in [R]$, $\theta_r^P \triangleq \mathbb{E}_{C\sim P}\left[\mu^C| r\in C\right]$. We say that $P$ is an $\alpha$-proper distribution, if for any $r\in [R]$ and a given $\alpha \in (0,1)$, we have $\mathbb{P}(r\in C) =\alpha$. 
\end{definition}

We are now ready to describe the parallel RCDM algorithm -- Algorithm~1; the description of the parallel ACDM is postponed to Section~\ref{sec:ACDM} of the Supplement. 
 \begin{table}[htb]
\centering
\begin{tabular}{l}
\hline
\label{PRCD}
\textbf{Algorithm 1: } \textbf{Parallel RCDM for Solving ~\eqref{compact}} \\
\hline
\ \textbf{Input}: $\mathcal{B}$, $\alpha$ \\
\ 0: Initialize $y^{(0)}\in\mathcal{B}$, $k\leftarrow 0$\\
\ 1: Do the following steps iteratively until the dual gap $<\epsilon$:\\
%\ 3: \quad Generate one group $C$ by following categorical distribution with $\mathbb{P}[C = C_{i}] = w_i$, $i\in[m]$ .  \\
%\ 2: \quad Pick one group $C_k\subseteq [R]$ with size $K$ uniformly randomly  \\
\ 2: \;\quad Sample $C_{ i_k}$ using some $\alpha$-proper distribution $P$\\
\ 3: \;\quad For $r\in C_{i_k}$: \\
%\ 4: \;\quad $z_{r, i}^{(k)} \leftarrow y_{r,i}^{(k)} - (\nabla_r g(y^{(k)}))_i/\theta_{r,i}^P$ for $i\in S_r$,\; $z_{r, i}^{(k)}\leftarrow 0$ for $i\notin S_r$\\
\ 4: \;\quad \quad $y_r^{(k+1)}\leftarrow \Pi_{\mathcal{B}_r, \theta_r^P} (y_{r}^{(k)} - (\theta_{r}^P)^{-1} \odot \nabla_r g(y^{(k)}))$ \\
\ 5: \;\quad Set $y_r^{(k+1)}\leftarrow y_r^{(k)}$ for $r\not \in C_{i_k}$, $k\leftarrow k+1$ \\
\ 6:\; Output $y^{(k)}$  \\
\hline
\end{tabular}
\end{table}
\\Next, we establish strong convexity results for the space $\|\cdot \|_{2, \theta^P}$ by invoking Lemma~\ref{lemmakappabound}.
\begin{lemma}~\label{skewstrongconv}
For any $y\in \mathcal{B}$, let $y^* = \arg\min_{\xi\in\Xi}\|\xi- y\|_{2,\theta^P}^2$. Then,
$$\|Ay - Ay^*\|_2^2 \geq \frac{2}{N\|\theta^P\|_{1,\infty}}\|y-y^*\|_{2,\theta^P}^2.$$ 
\end{lemma}  
The convergence rate of Algorithm 1 is established in the next theorem. 
\begin{theorem} \label{PCDrate}
At each iteration of Algorithm 1, $y^{(k)}$ satisfies 
\begin{align*}
\mathbb{E}\left[g(y^{(k)})- g(y^*) + \frac{1}{2}d_{\theta^P}^2(y^{k}, \xi) \right] \leq \left[1- \frac{4\alpha}{(N\|\theta^P\|_{1,\infty} + 2)}\right]^{k}\left[g(y^{(0)})- g(y^*) + \frac{1}{2}d_{\theta^P}^2(y^{0}, \xi) \right].
\end{align*}
\end{theorem}
The parameter $N\|\theta^P\|_{1,\infty}$ is obtained by combining the strong convexity constant and the properties of the sampling distribution $P$. 
%As $K$ cores are available to compute projection in parallel, we may choose some $\mathcal{P}$ that is $\frac{K}{R}$-proper. 
Small values of $\|\theta^P\|_{1,\infty}$ ensure better convergence rates, and we next bound this value. 

\begin{lemma}\label{lowerbound}
For any $\alpha$-proper distribution $P$ and an element $i\in[N]$, $\max\limits_{r\in[R]: i\in S_r} \theta_{r,i}^P\geq \max\{ \alpha \mu_i ,1\}$. Consequently, $\|\theta^P\|_{1,\infty} \geq \max\{\alpha\|\mu\|_1, N\}$.
\end{lemma}
Without considering incidence relations, i.e., by setting $\|\mu\|_1= NR$, one always has $\|\theta^P\|_{1,\infty} \geq \alpha NR$, which shows that parallelization cannot improve the convergence rate of the RCDM.
The next lemma characterizes an achievable $\|\theta^P\|_{1,\infty}$ obtained by choosing $P$ to be a uniform distribution, which, when combined with Theorem~\ref{PCDrate}, proves the result of the last column in Table~\ref{tab:results}.
\begin{lemma}\label{eqnorm}
If $C$ is a set of size $0<K \leq R$ obtained by sampling the $K$-subsets of $[R]$ uniformly at random, then $\theta_{r}^P= \frac{K-1}{R-1} \mu + \frac{R-K}{R-1} 1$. Moreover, $\|\theta^P\|_{1,\infty} = \frac{K-1}{R-1} \|\mu\|_1 + \frac{R-K}{R-1} N$.
\end{lemma}

Comparing Lemma~\ref{lowerbound} and Lemma~\ref{eqnorm}, we see that the $\|\theta^P\|_{1,\infty}$ achieved by sampling uniformly at random is at most a factor of two of the lower bound since $\alpha = K/R$. A natural question is if it is possible to devise a better sampling strategy. This question is addressed in Section~\ref{sec:comb} of the Supplement, where we related the sampling problem to equitable coloring~\cite{meyer1973equitable}. By using Hajnal-Szemer{\'e}di's Theorem~\cite{hajnal1970proof}, we derived a sufficient condition under which an $\alpha$-proper distribution $P$ that achieves the lower bound in Lemma~\ref{lowerbound} can be found in polynomial time. We also described a greedy algorithm for minimizing $\|\theta^P\|_{1,\infty} $ that empirically convergences faster than sampling uniformly at random.

\section{Experiments} \label{sec:experiments}
\begin{figure*}[t]
\centering
\includegraphics[trim={0cm 0cm 0.0cm 0.0cm},clip,width=.24\textwidth]{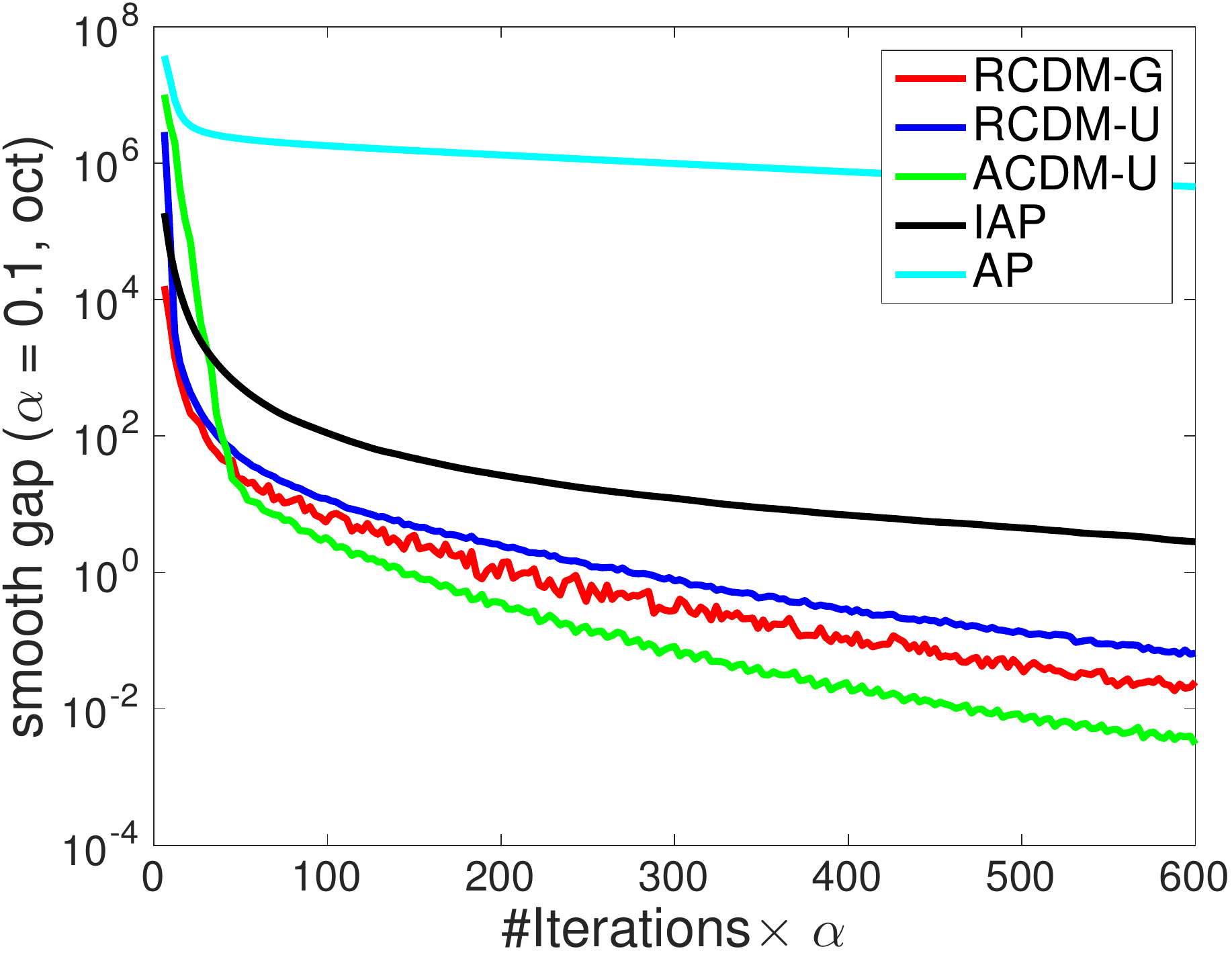}
\includegraphics[trim={0cm 0cm 0.0cm 0.0cm},clip, width=.24\textwidth]{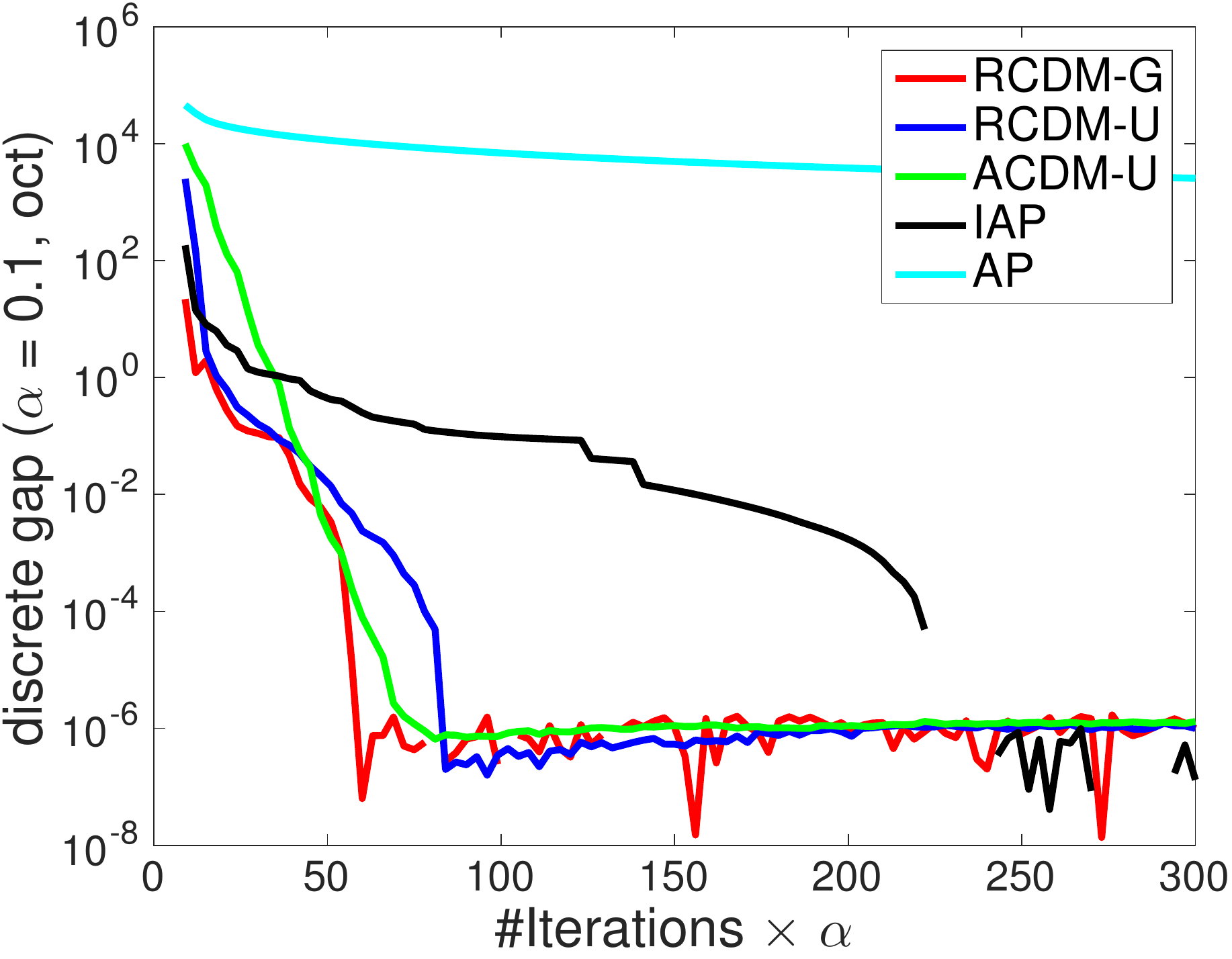}
\includegraphics[trim={0cm 0cm 0.0cm 0.0cm},clip,width=.24\textwidth]{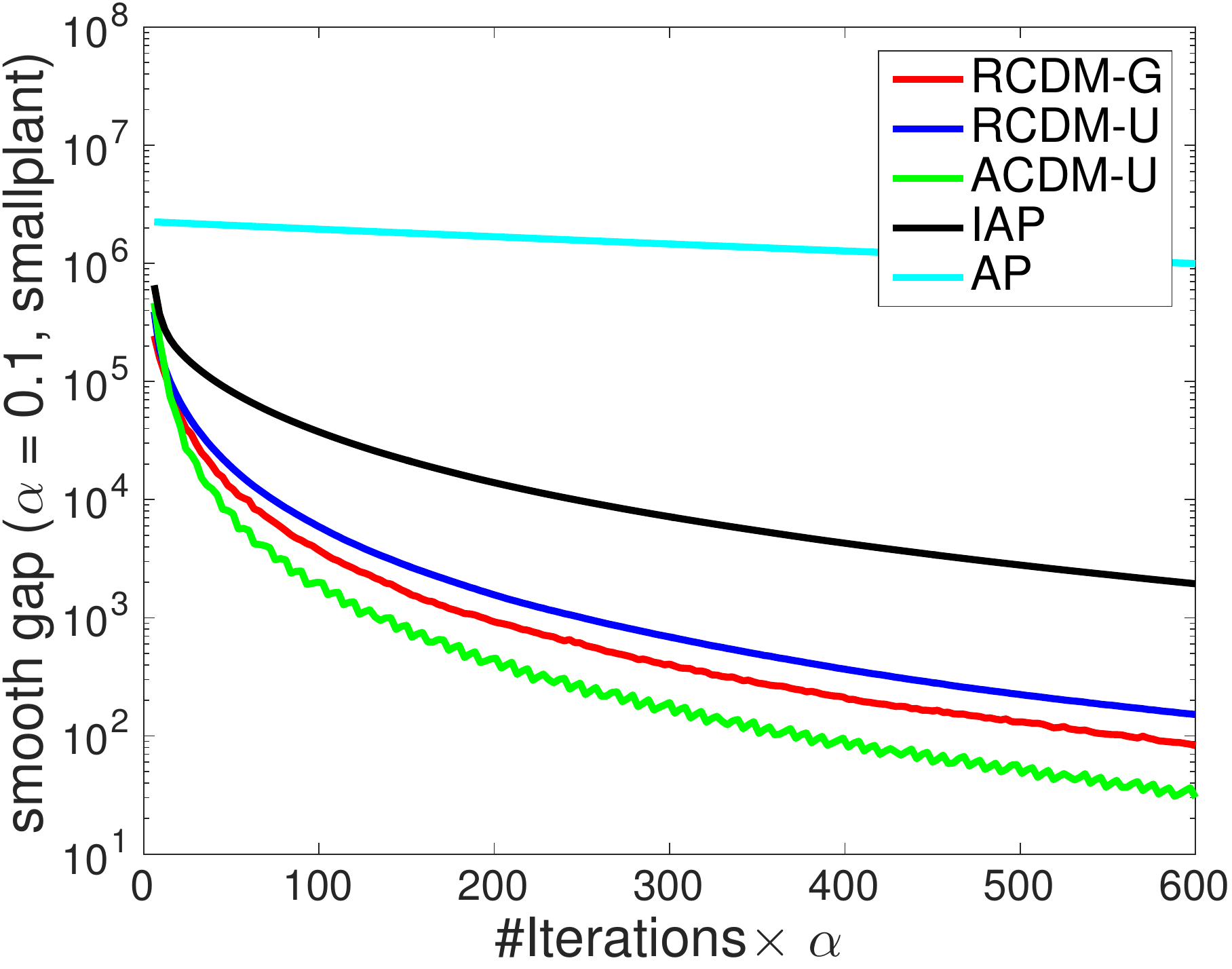}
\includegraphics[trim={0cm 0cm 0.0cm 0.0cm},clip, width=.24\textwidth]{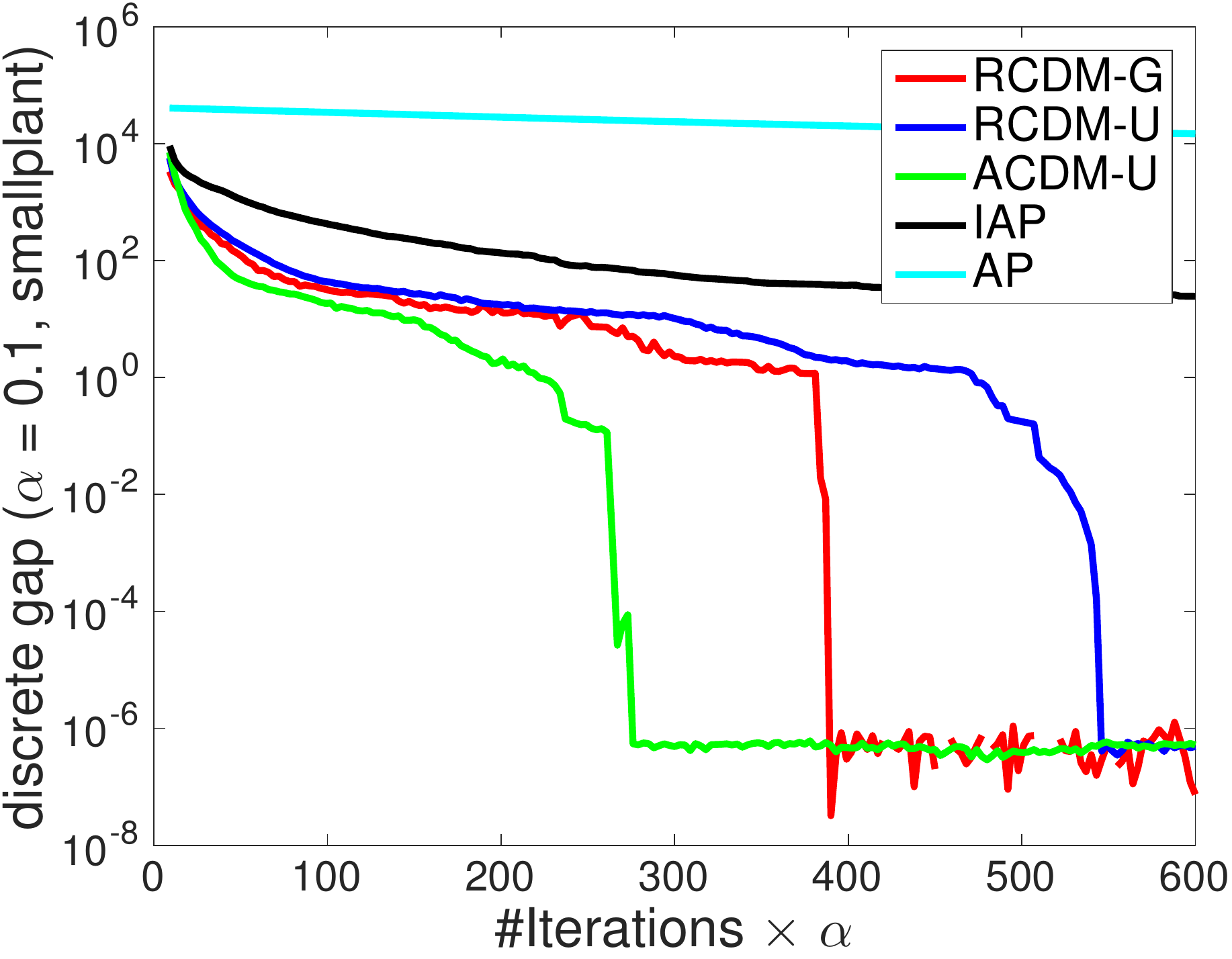}
\includegraphics[trim={0cm 0cm 0.0cm 0.0cm},clip,width=.24\textwidth]{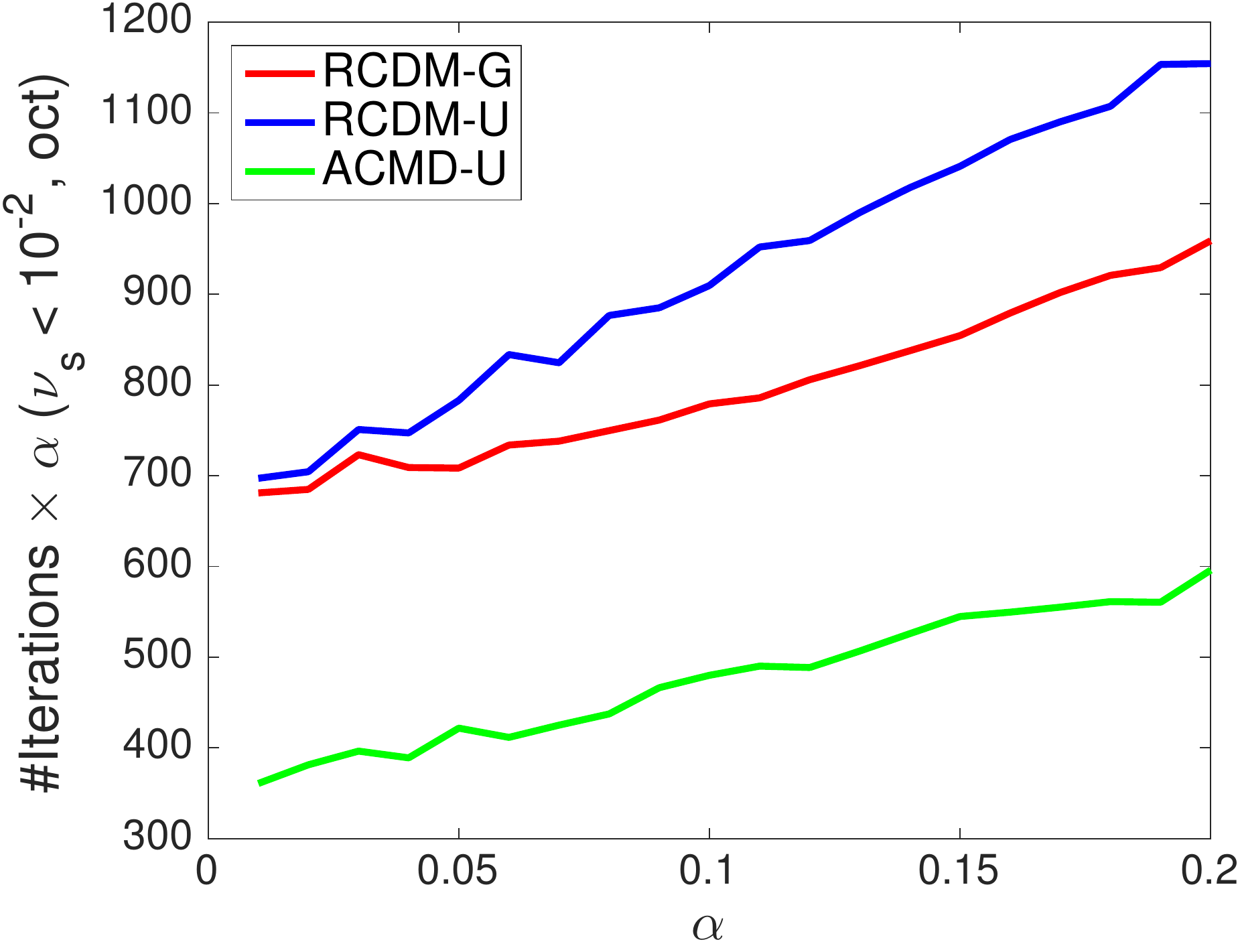}
\includegraphics[trim={0cm 0cm 0.0cm 0.0cm},clip, width=.24\textwidth]{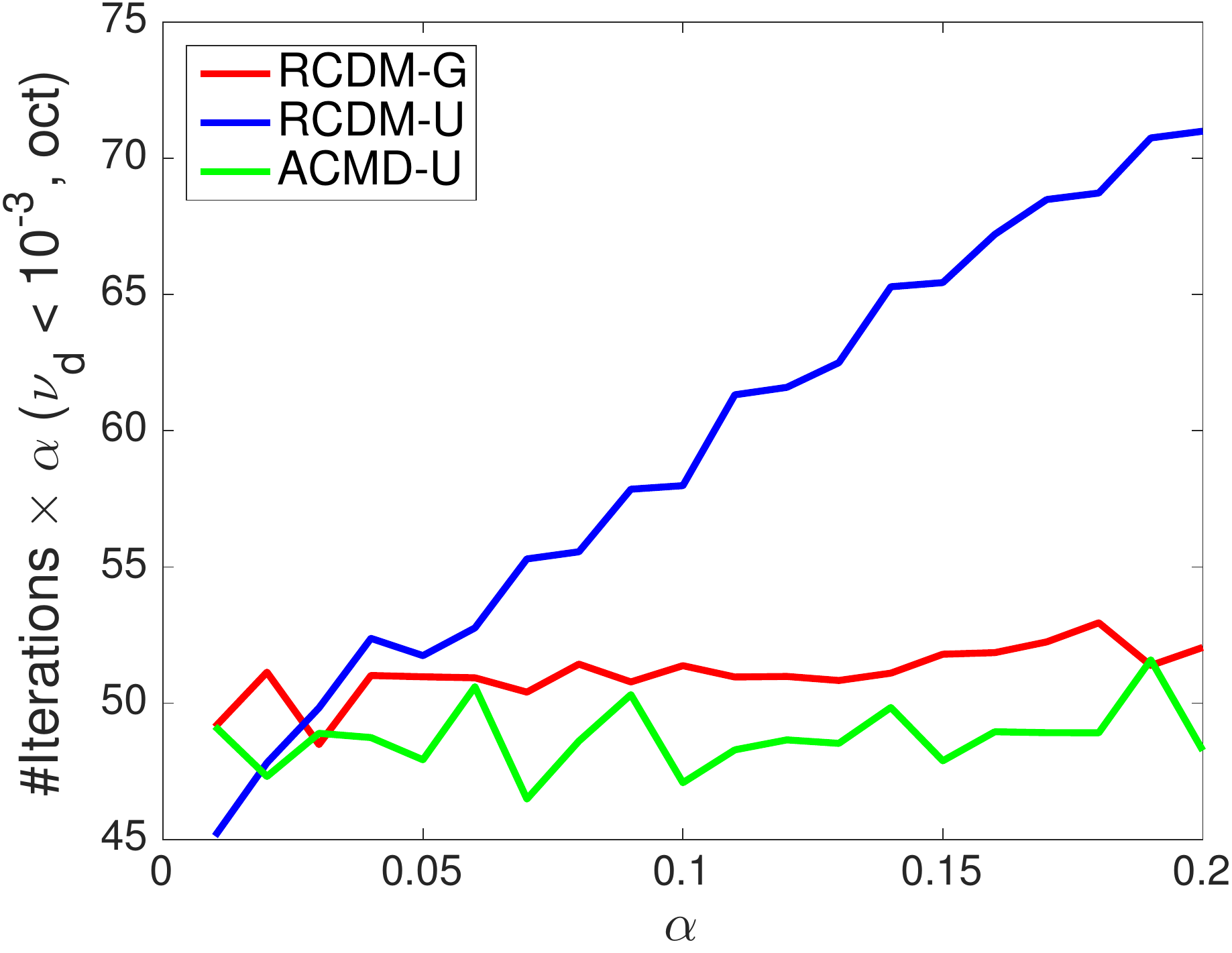}
\includegraphics[trim={0cm 0cm 0.0cm 0.0cm},clip,width=.24\textwidth]{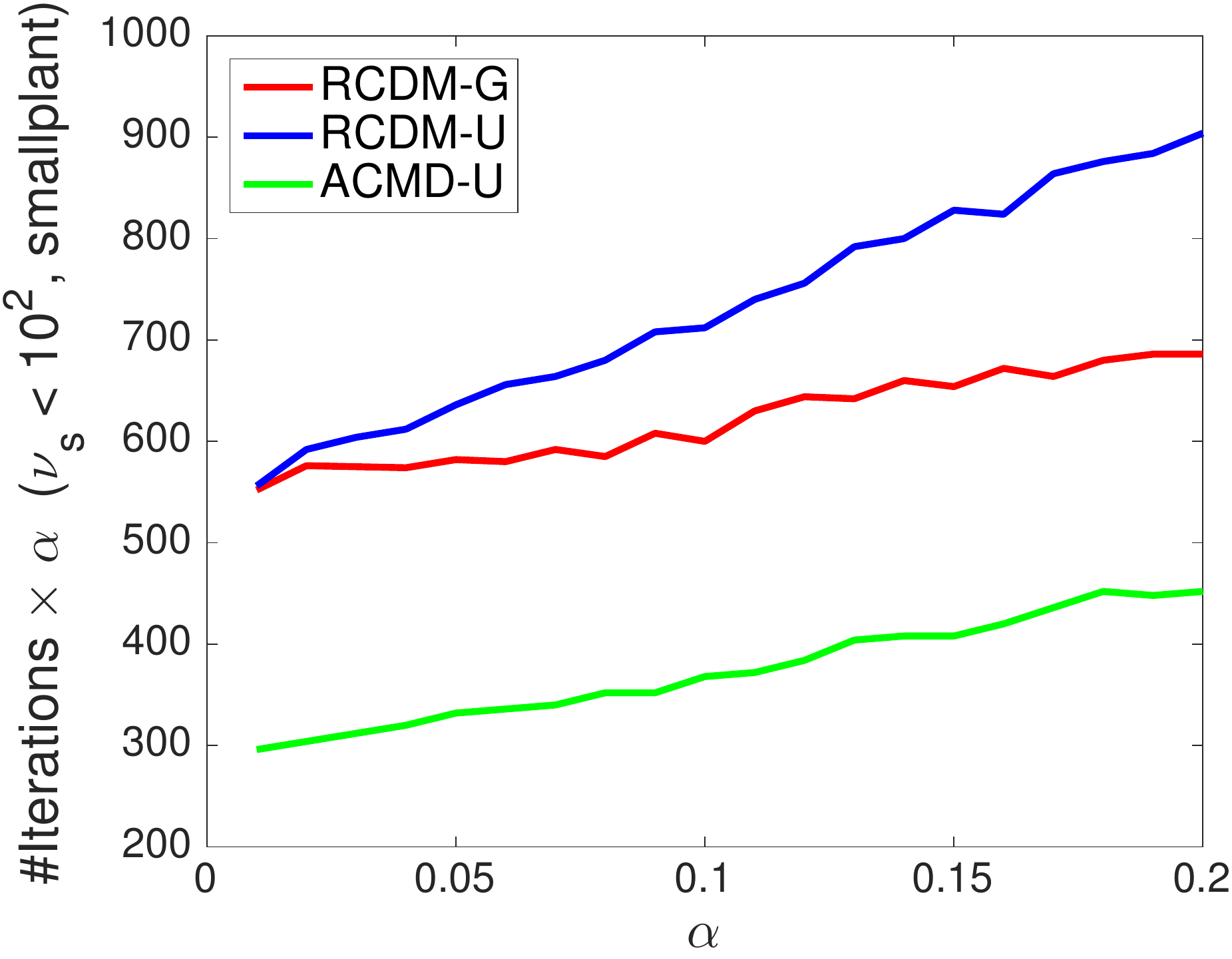}
\includegraphics[trim={0cm 0cm 0.0cm 0.0cm},clip, width=.24\textwidth]{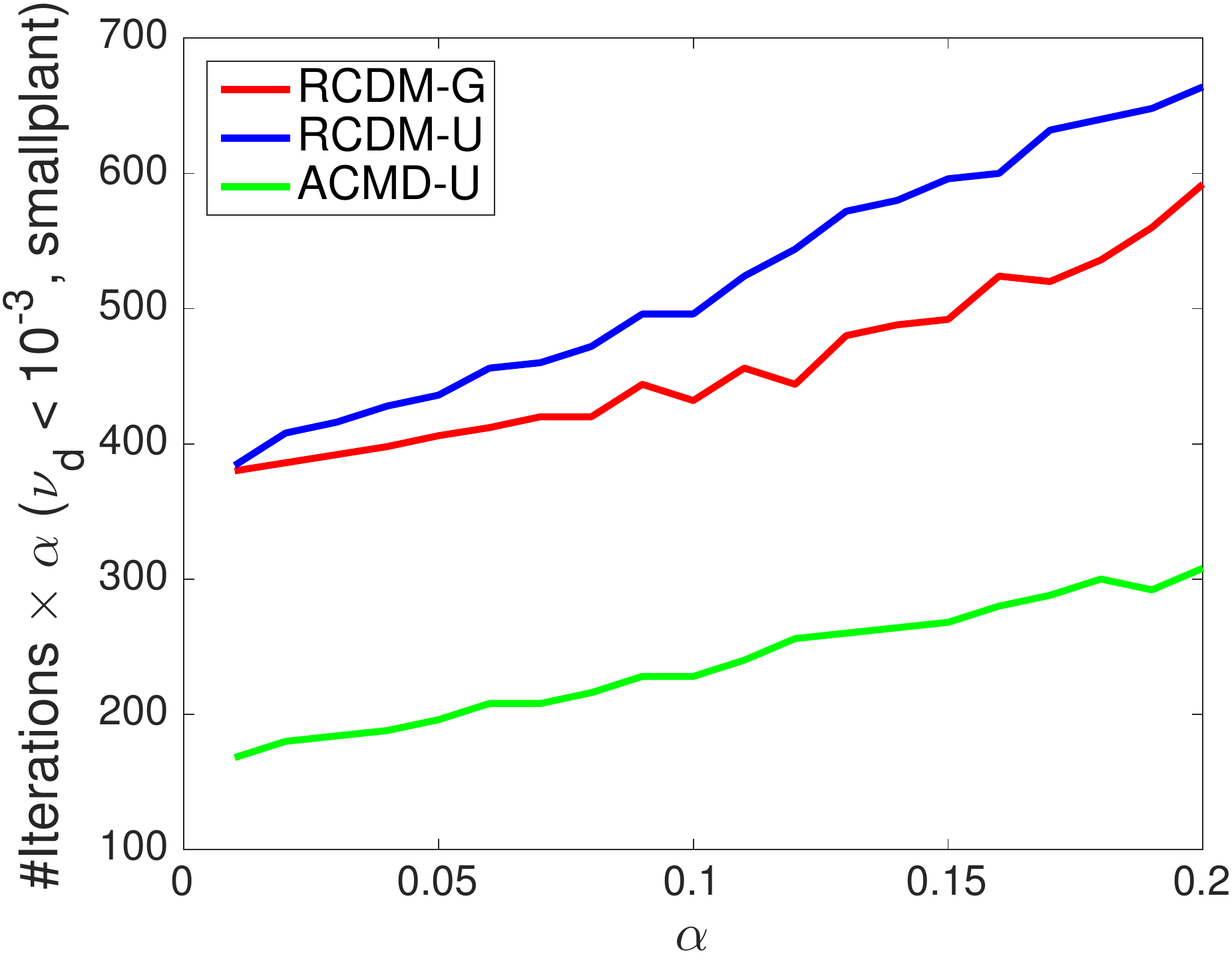}
\vspace{-0.1cm}
\caption{Image segmentation example. First row: Gap vs the number of iterations $\times \alpha$. Second row: The number of iterations $\times \alpha$ vs $\alpha$. Here, $\alpha$ is the parallelization parameter, while $K = \alpha R$ equals the number of projections that have to be computed in each iteration.}
\label{imageseg}
\vspace{-0.5cm}
\end{figure*}

In what follows, we illustrate the performance of the newly proposed DSFM algorithms on a benchmark datasets used for MAP inference in image segmentation~\cite{stobbe2010efficient} and used for semi-supervised learning over graphs. More experiments on semi-supervised learning over hypergraphs can be found in Section~\ref{supexp} of the Supplement.

In all the experiments, we evaluated the convergence rate of the algorithms by using the smooth duality gap $\nu_s$ and the discrete duality gap $\nu_d$. The primal problem solution equals $x = -Ay$ so that the smooth duality gap can be computed according to $\nu_s = \sum_{r} f_r(x) + \frac{1}{2}\|x\|^2 - (-\frac{1}{2}\|Ay\|^2)$. Moreover, as the level set $S_{\lambda} = \{v\in [N] | x_v>\lambda\}$ can be easily found based on $x$, the discrete duality gap can be written as $\nu_d = \min_{\lambda} F(S_{\lambda}) - \sum_{v\in [N]} \min\{-x_v, 0\} $. 

\textbf{MAP inference}. We used two images -- \emph{oct} and \emph{smallplant} -- adopted from~\cite{jegelka2013reflection}\footnote{Downloaded from the website of Professor Stefanie Jegelka: http://people.csail.mit.edu/stefje/code.html}. The images comprise $640\times 427$ pixels so that $N = 273,280$. The decomposable submodular functions are constructed following a standard procedure. The first class of functions arises from the $4$-neighbor grid graph over the pixels. Each edge corresponds to a pairwise potential between two adjacent pixels $i,j$ that follows the formula $\exp(-\|v_i - v_j\|_2^2),$ where $v_i$ is the RGB color vector of pixel $i$. We split the vertical and horizontal edges into rows and columns that result in $639+426=1065$ components in the decomposition. Note that within each row or each column, the edges have no overlapping pixels, so the projections of these submodular functions onto the base polytopes reduce to projections onto the base polytopes of edge-like submodular functions. The second class of submodular functions contain clique potentials corresponding to the superpixel regions; specifically, for region $r$, $F_r(S) = |S|(|S_r| -|S|)$~\cite{levinshtein2009turbopixels}. These functions give another $500$ decomposition components. We apply the divide and conquer method in~\cite{jegelka2013reflection} to compute the projections required for this type of submodular functions. Note that in each experiment, all components of the submodular function are of nearly the same size, and thus the projections performed for different components incur similar computational costs. As the projections represent the primary computational units, for comparative purposes we use the number of iterations (similarly to~\cite{jegelka2013reflection,ene2015random}).

We compared five algorithms: RCDM with a sampling distribution $P$ found by the greedy algorithm (RCDM-G), RCDM with uniform sampling (RCDM-U), ACDM with uniform sampling (ACDM-U), AP based on~\eqref{newdistance} (IAP) and AP based on~\eqref{distance} (AP). Figure~\ref{imageseg} depicts the results. In the first row, we compared the convergence rates of different algorithms for a fixed parallelization parameter $\alpha=0.1$. The values on the horizontal axis correspond to $\#$ iterations $ \times \, \alpha$, the total number of projections performed divided by $R$. The results are averaged over $10$ independent experiments. We observe that the CD-based methods outperform AP-based methods, and that ACDM-U is the best performing CD-based method. IAP significantly outperforms AP. Similarly, RCDM-G outperforms RCDM-U. We also investigated the relationship between the number of iterations and the parameter $\alpha$. %We recorded the number of iterations needed to achieve a smooth gap below $10^{-2}$ for the image ``oct'', the smooth gap $10^{2}$ for the image ``smallplant'', and the discrete gap $10^{-3}$ for both images. 
We recorded the number of iterations needed to achieve a smooth and discrete gap below a certain given threshold. The results are shown in the second row of Figure~\ref{imageseg}. We did not plot the curves for the AP-based methods as they are essentially horizontal lines. Among the CD-based methods, ACDM-U performs best. RCDM-G offers a much better convergence rate than RCDM-U since the sampling probability $P$ produced by the greedy algorithm leads to a smaller value of $\|\theta^P\|_{1,\infty}$ compared to uniform sampling. The reason behind this finding is that the supports of the components in the decomposition are localized, which makes the sampling $P$ obtained from the greedy algorithm highly effective. For RCDM-U, the total number of iterations increases almost linearly with $\alpha$ ($=K/R$), which confirms the results of Lemma~\ref{eqnorm}. 
%Comparing RCDM-G and RCDM-U, the superiority becomes obvious when $\alpha$ increases to a moderate value. The reason is that   % The reason for this is as following: We may compare the lower bound of $D_{\mathcal{P}}'$ in Theorem~\ref{lowerbound} and the formula of $D_{\mathcal{P}}'$ in  Lemma~\eqref{eqnorm}. As $\alpha$ increases from 0 to 1, the gap between two formulas first increases then decreases and the greatest one is achieved when $\alpha\delta^s = N$. The larger that gap is, the more space that  the greedy balanced-partition can improve. According to the setting of this experiment, as almost all pixels occur in two horizontal edges, two vertical edges and a superpixel, we have $\delta^s \approx 5N$. So in this case where $\alpha < \frac{N}{\delta^s} \approx 0.2$,  one may observe such a phenomenon.
%Moreover, to see the parallelizability, one compares the smooth gaps for $\alpha = 0.02$ and $\alpha=0.1$ cases when $\#$iterations $\times$ $\alpha$ achieves $600$. As it shows, both ACDM-U and RCDM-U hold a little bit larger gaps while the other three algorithms almost keep unchanged. This is caused by the high possibility that different decomposed parts with intersection are sampled in one iteration via uniformly sampling.   

Note that in the above examples of MAP inference, another way to decompose the submodular functions is available: as there are three natural layers of non-overlapping incidence sets, we can merge all vertical edges, all horizontal edges, and all superpixel regions into three components respectively. Then, each of this component is incident to all pixels, and the derived results in this work will reduce to those of the former works~\cite{jegelka2013reflection,ene2015random}. However, such a way to decompose submodular function strongly depends on the particular structure and thus is not general for DSFM problems. The following example on semi-supervised learning over graphs does not contain natural layers for decomposition.   

\begin{figure}[t]
\centering
\includegraphics[trim={0cm 0cm 0cm 0cm},clip,width=.24\textwidth]{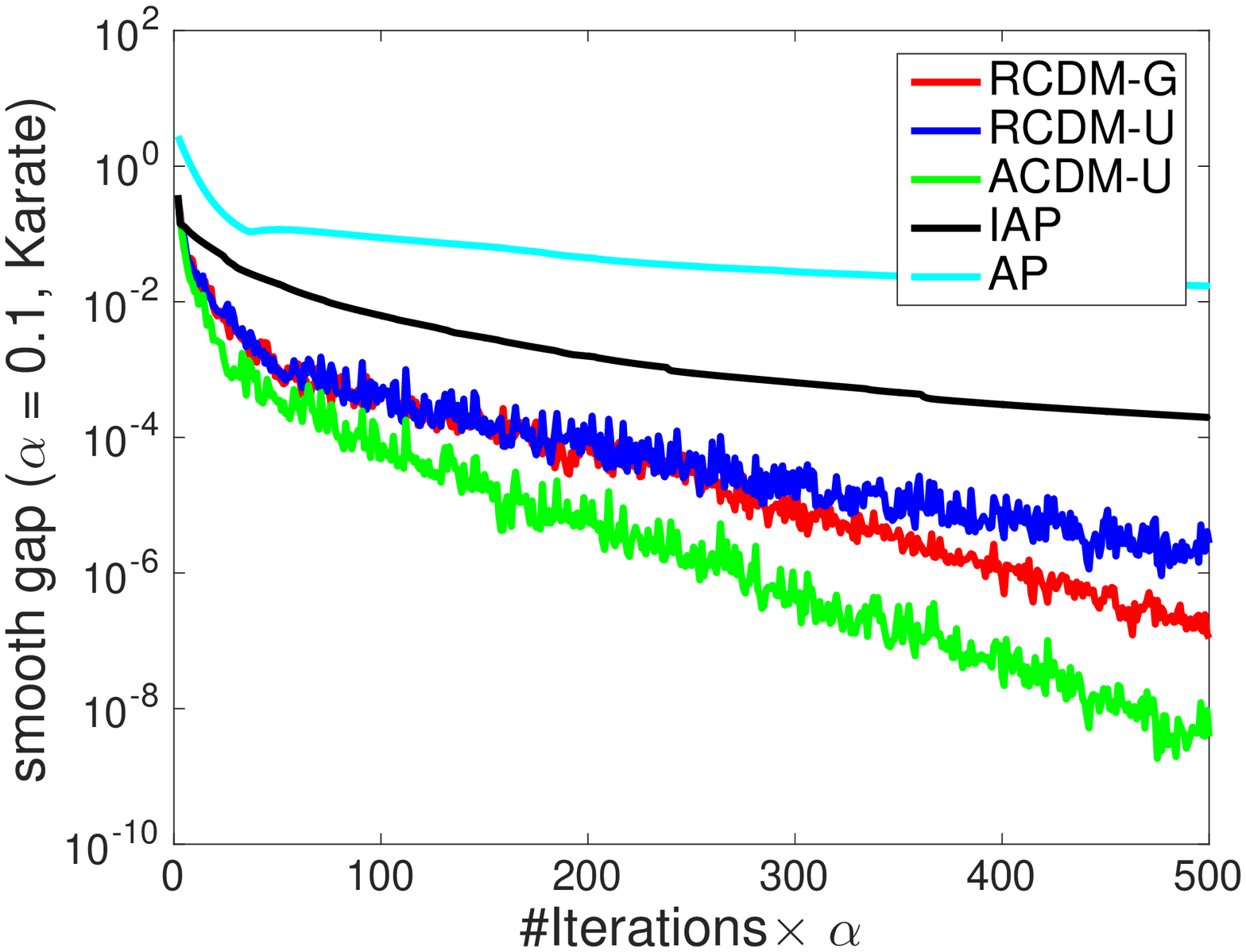}
\includegraphics[trim={0cm 0cm 0cm 0cm},clip, width=.24\textwidth]{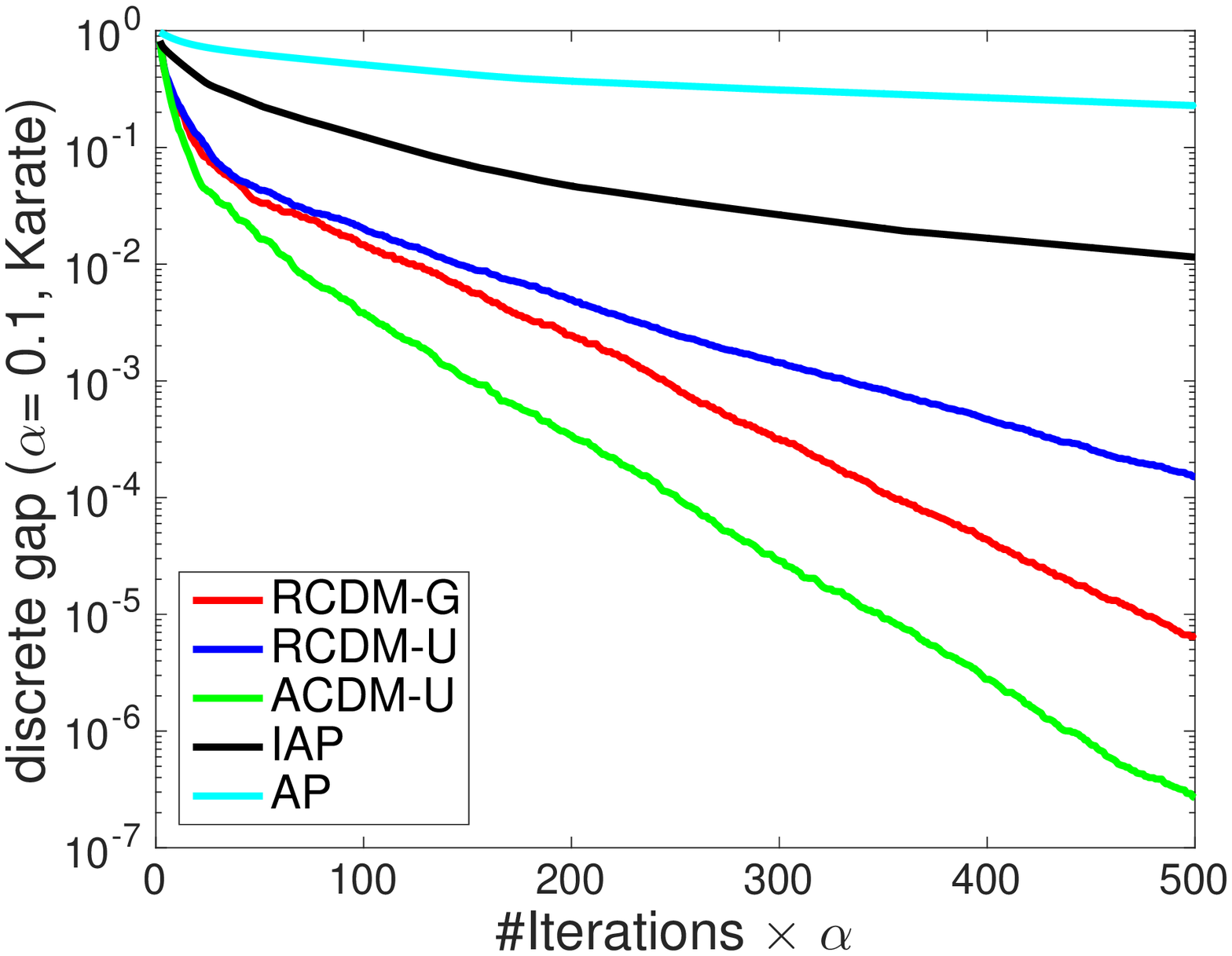}
\includegraphics[trim={0cm 0cm 0cm 0cm},clip,width=.24\textwidth]{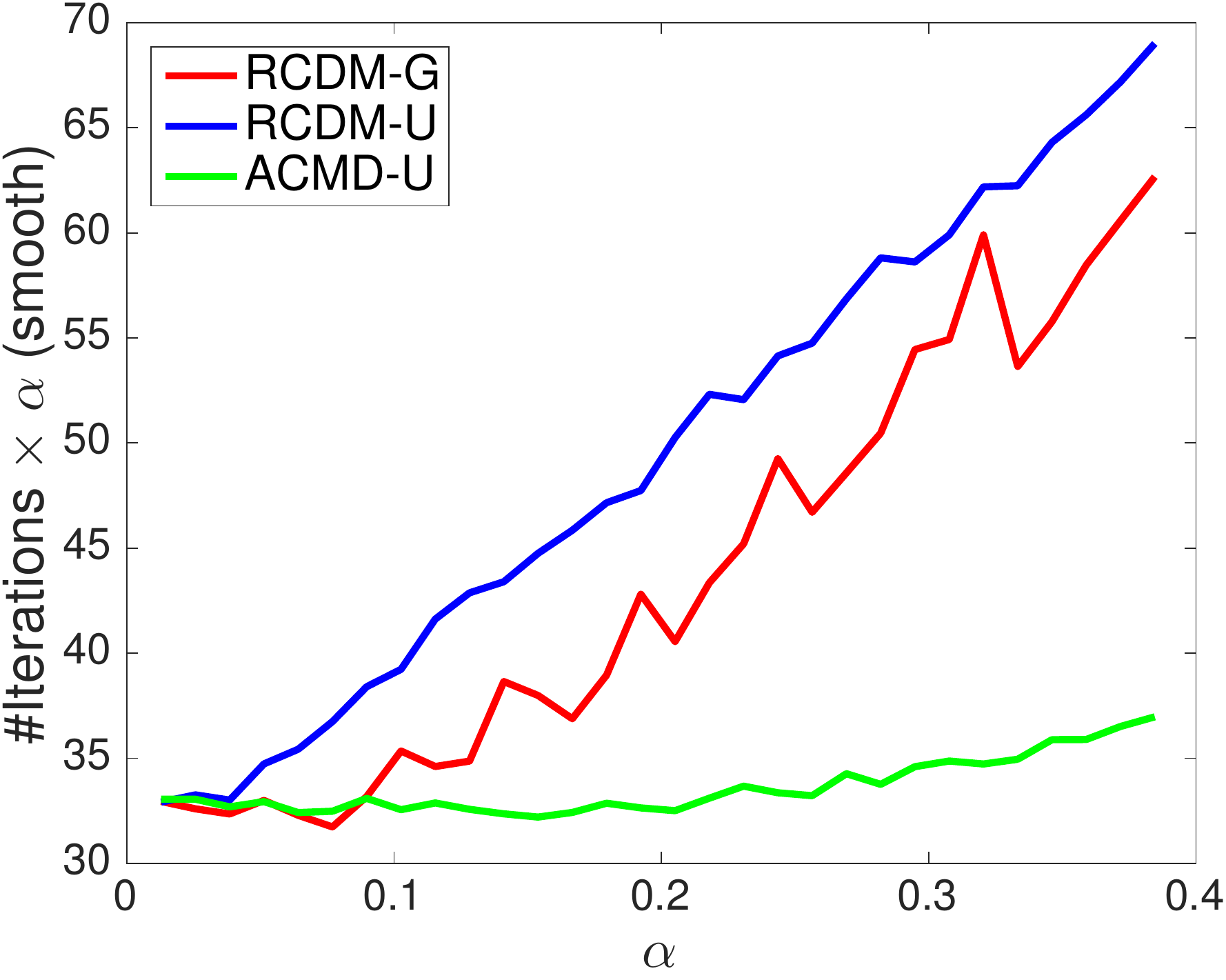}
\includegraphics[trim={0cm 0cm 0cm 0cm},clip, width=.24\textwidth]{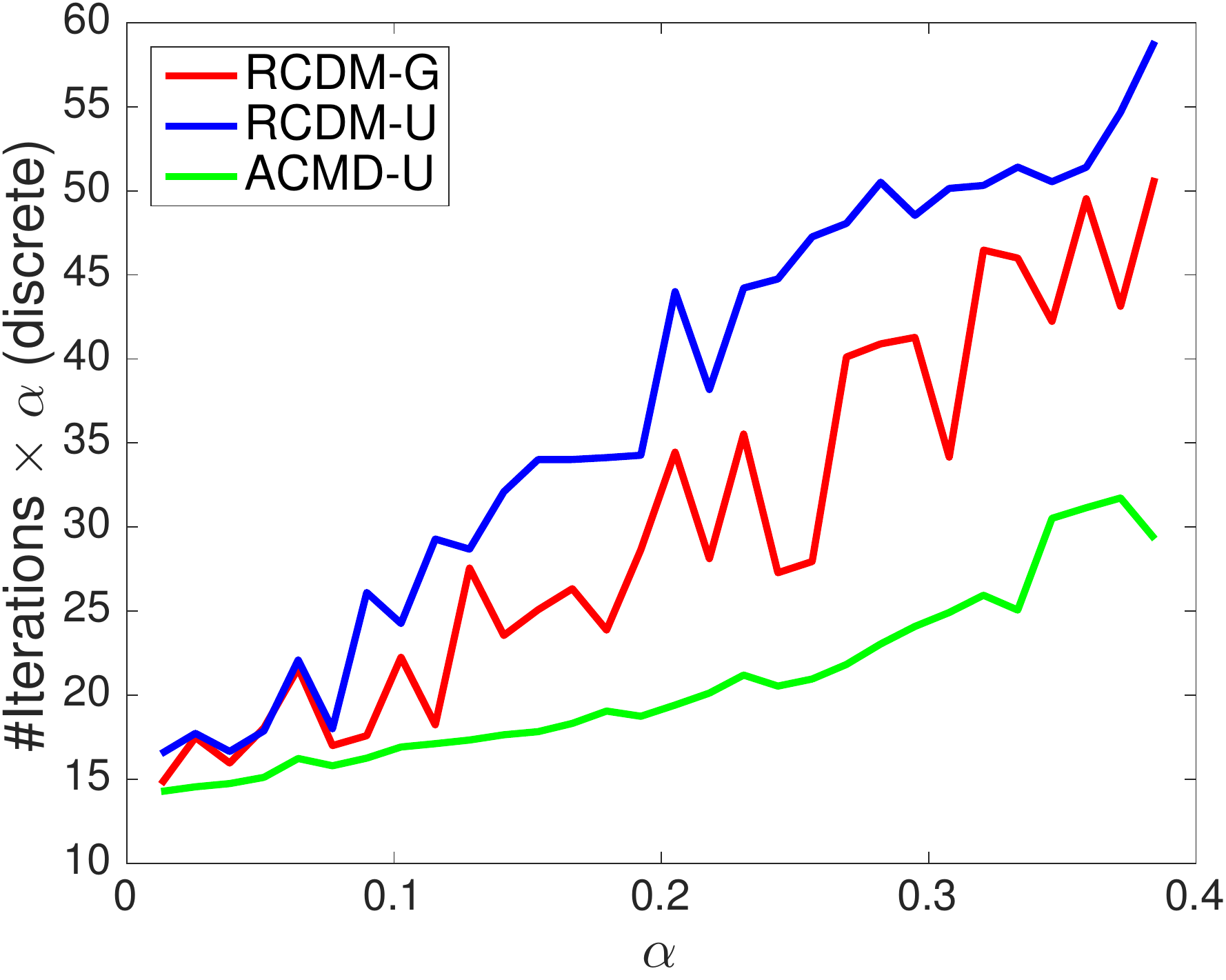}\\
\caption{Zachary's Karate Club. Left two: Gap vs the number of iterations $\times \alpha$. Right two: The number of iterations $\times \alpha$ vs $\alpha$. Here, $\alpha$ is the parallelization parameter, while $K = \alpha R$ equals the number of projections that have to be computed in each iteration.}
\label{karate}
\end{figure}

\textbf{Semi-supervised learning}.  We tested our algorithms over the dataset of Zachary's karate club~\cite{zachary1977information}. This dataset is used as a benchmark example for evaluating semisupervised learning algorithms over graphs~\cite{kipf2016semi}. It includes $N =34 $ vertices and $R=78$ submodular functions in the decomposition, each corresponding to one edge in the network. The objective function of both semi-supervised learning problems may be written as 
\begin{align}\label{semisupervise}
\min_{x} \tau\sum_{r\in [R]} f_r(x) + \frac{1}{2}\|x- x_0\|_2^2
\end{align}
where $\tau$ is a parameter that needs to be tuned, and $x_0 \in \{-1,0,1\}^N,$ so that the nonzero components correspond to the labels that are known a priori. In our case, as we are only concerned with the convergence rate of the algorithm, we fix $\tau = 0.1$. In the experiments for Zachary's karate club, we set $x_0(1) = 1$, $x_0(34) = -1$ and let all other components of $x_0$ be equal to zero. 

Figure~\ref{karate} shows the results of the experiments pertaining to Zachary's karate club. In the left two subfigures, we compared the convergence rates of different algorithms for a fixed parallelization parameter $\alpha=0.1$. The values on the horizontal axis correspond to $\#$ iterations $ \times \, \alpha$, the total number of projections performed divided by $R$. In the right two subfigures, we controlled the numbers of projections executed within one iteration by tuning the parameter $\alpha$ and recorded the number of iterations needed to achieve smooth/discrete gaps below $10^{-3}$. The values depicted on the vertical axis correspond to $\#$ iterations $\times \alpha$, describing the total number of projections needed to achieve the given accuracy. In all cases, we see the similar tendency to that of the MAP inference. As may be seen, AP-based methods require more projections than CD-based methods, but IAP consistently outperforms AP, which is consistent with our theoretical results. Among the CD-based methods, ACDM-U offers the best performance in general, and RCDM-G slightly outperforms RCDM-U, since the greedy algorithm used for sampling produces a smaller $\|\theta^P\|_{1,\infty}$ than uniform sampling. As the AP-based methods are completely parallelizable, and increasing the parameter $\alpha$ does not increase the total number of projections. However, for RCDM-U, the total number of iterations required increases almost linearly with $\alpha$, which is supported by the result in Lemma 3.12. The performance curve for RCDM-G exhibits large oscillations due to the discrete problem component, needed for finding a balanced partition. 

\section{Acknowledgement}
The authors gratefully acknowledge many useful suggestions by the reviewers. This work was supported in part by the NSF grant CCF 15-27636, the NSF Purdue 4101-38050 and the NFT STC center Science of Information. 

\bibliographystyle{IEEETran}
\bibliography{example_paper}
 
 \newpage
 
\appendix

\begin{appendix}

\begin{center}
{\Large \textbf{Supplement}}
\end{center}
\end{appendix}

\section{Discrete Optimization Approach for Computing the Projections $\Pi_{\mathcal{B}_r, w}(\cdot)$}\label{sec:discdiscreteopt}
The following Lemma~\ref{skewedprim} describes how the projections $\Pi_{\mathcal{B}_r, w}(\cdot)$ can be performed via discrete optimization. Discrete methods are especially useful when $F_r(S)$ is concave in $|S|$, as in this case they have much smaller complexity than the min-norm algorithm of Wolfe~\cite{wolfe1976finding}. 
\begin{lemma}\label{skewedprim}
The optimization problem $\min_{y_r\in \mathcal{B}_r} \|z - y_r\|_{2,w}^2$ is the dual of the problem $\min_{x\in \mathbb{R}^{N}} f_r(x) - \langle x , z\rangle + \frac{1}{2}\|x\|_{2,w^{-1}}^2$. A solution with coordinate accuracy $\epsilon$ for the latter setting can be obtained  by solving the discrete problem
$$\min_S F_r(S) - z(S) + \lambda \sum_{i\in S_r\cap S}  w_i^{-1},$$ 
where 
\begin{align*}
\lambda\in &\left[ \min_{i\in[N]} [- F_r(\{i\}) + z(\{i\})]w_{i} ,\; \max_{i\in[N]} [F_r([N]/\{i\})- F_r([N]) + z(\{i\})]w_{i} \right],
\end{align*}
at most $\min\{|S_r|, \log{1/\epsilon} \}$ times. The parameter $\lambda$ is chosen based on a binary search procedure which requires solving the discrete problem $O(\log 1/\epsilon)$ times.
\end{lemma}
\begin{proof}
The first statement follows from $f_r(x) = \max_{y_r\in \mathcal{B}_r} \langle y_r, x\rangle$ and some simple algebra. The second claim follows from the divide and conquer algorithm described in Appendix B of~\cite{jegelka2013reflection}.
\end{proof}

\section{Proof of Lemma~\ref{lemmakappabound}}
The first part of the proof follows along the same line as the corresponding proof of Ene et al.~\cite{ene2017decomposable} which is based on a submodular auxiliary graph and the path-augmentation algorithm~\cite{fujishige1992new}, described in what follows. 

Let $G = (V, E)$ be a directed graph such that the vertex set $V$ corresponds to the elements in $[N]$, and where the arc set may be written as $E = \cup_{r\in[R]} E_r$, with $E_r$ corresponding to a complete directed graph on the set of elements $S_r$ incident to $F_r$. With each arc $(u,v)$, we associate a capacity value based on a $y'\in\mathcal{B}$ according to $c(u,v)\triangleq \min\{f_{r}(S) - y_r'(S): S\subseteq S_r, u\in S, v\notin S\}$. 

Next, we consider a procedure termed path augmentations over $G$ that sequentially transforms $y'$ from $y'=y$ to a point in $\mathcal{B}$ that satisfies $Ay'= z$; the vector $y'$ is kept within $\mathcal{B}$ during the whole procedure. Let the set of source and sink nodes of the graph be defined as $N \triangleq \{v\in [N] | (Ay')_v < z_v\}$ and $P \triangleq \{v\in[N] |(Ay')_v > z_v\}$, where $z$ is as defined in the statement of the lemma. If $N = P =\emptyset $, we have $Ay' = z$. It can be shown that there always exists a directed path with positive capacity from $N$ to $P$ unless $N = P =\emptyset$~\cite{ene2017decomposable}. In each step, we find the shortest directed path, denoted by $\mathcal{Q}$, with positive capacity from $N$ to $P$. For each arc $(u,v)$ in $\mathcal{Q}$, if the arc belongs to $E_r$, we set $y'_{r,u} \leftarrow y'_{r,u} + \rho$, $y'_{r,v} \leftarrow y'_{r,v} - \rho,$ where $\rho$ denotes the smallest capacity of any arc in $\mathcal{Q}$. This procedure ensures that $y'\in\mathcal{B}$ and that the procedure terminates in a finite number of steps, with $N = P =\emptyset$~\cite{fujishige1992new}. 
  
The second part of the proof differs from the derivations of Ene et al.~\cite{ene2017decomposable}. 
Suppose that $\{y'^{(0)} =y, y'^{(1)}, ..., y'^{(t)}\}$ is a sequence such that $y'^{(i)}$ equals the vector $y'$ after the $i$-th step of the above procedure. We also assume that $Ay'^{(t)} = z,$ implying that the algorithm terminated at step $t$. Hence, the point $y'^{(t)}$ is the desired value of $\xi$. During path-augmentation, no element appears in more than two updated arcs. Hence, 
$$\|y'^{(i)} - y'^{(i-1)}\|_{2, \theta}\leq \sqrt{2\sum_{v} \max_{r\in[R]: v\in S_r}\theta_{r,v}} \rho=\sqrt{2\|\theta\|_{1,\infty}} \rho.$$ 
As $\|Ay'^{(i)} - Ay'^{(i-1)}\|_{1} = 2 \rho$, we have 
$$\|y'^{(i)} - y'^{(i-1)}\|_{2, \theta} \leq \sqrt{\frac{\|\theta\|_{1,\infty}}{2}} \|Ay'^{(i)} - Ay'^{(i-1)}\|_{1}.$$ 
An important observation is that during the path-augmentation procedure, for each component $v\in[N]$, the updated sequence $\{(Ay'^{(i)})_v\}_{i=1,2,..,t}$ converges monotonically to $z_v$. Hence, $\|Ay'^{(t)} - Ay'^{(0)}\|_{1} = \sum_{i=1}^t \|Ay'^{(i)} - Ay'^{(i-1)}\|_{1}$. By using the triangle inequality for the norm $\|\cdot\|_{2,\theta}$, we obtain 
\begin{align*}
\sqrt{\frac{\|\theta\|_{1,\infty}}{2}} \|z - Ay\|_{1} & =  \sqrt{\frac{\|\theta\|_{1,\infty}}{2}} \|Ay'^{(t)} - Ay'^{(0)}\|_{1} \geq  \sum_{i=1}^t\|y'^{(i)} - y'^{(i-1)}\|_{2, \theta}\\
& \geq \|y'^{(t)} - y'^{(0)}\|_{2, \theta} = \|y'^{(t)} - y\|_{2, \theta}. 
\end{align*}
Invoking the Cauchy-Schwarz inequality establishes $\|z - Ay\|_{1} \leq \sqrt{\|w^{-1}\|_1}\|z - Ay\|_{2,w}$, which concludes the proof.

\section{Proof for Lemma~\ref{dualprob}}
The equivalence between problem~\eqref{newdistance} and problem~\eqref{compact} is easy to establish, as $y$ is obtained from $y'$ by simply removing its zero components. The second statement is proved as follows: 
 \begin{align}
 & \min_{y\in \mathcal{B}}\min_{a: Aa= 0, a_{r,i} = 0,\; \forall (r, i): i\notin S_r} \frac{1}{2}\|  y - a\|_{2, I(\mu)}^2 \nonumber \\
 =& \min_{y\in \mathcal{B}} \min_{a: a_{r,i} = 0,\; \forall (r, i): i\notin S_r} \max_{\lambda \in \mathbb{R}^N}\frac{1}{2}\|  y - a\|_{2, I(\mu)}^2 - \langle \lambda,  Aa  \rangle \nonumber \\
  \stackrel{1)}{=}& \min_{y\in \mathcal{B}}  \max_{\lambda \in \mathbb{R}^N}  \min_{a\in \otimes_{r=1}^R R^N} \frac{1}{2}\sum_{r\in [R]}\sum_{i\in S_r}[\mu_i(y_{r,i} - a_{r,i})^2 - 2\lambda_i a_{r,i}] \nonumber \\
 =&  \min_{y\in \mathcal{B}}  \max_{\lambda \in \mathbb{R}^N} \frac{1}{2}\sum_{r\in [R]}\sum_{i\in S_r}[\mu_i^{-1}\lambda_i^2 - 2\lambda_i (\mu_i^{-1}\lambda_i + y_{r,i})]  \nonumber \\
  =&  \min_{y\in \mathcal{B}}  \max_{\lambda \in \mathbb{R}^N} \frac{1}{2}\sum_{r\in [R]}\sum_{i\in S_r}(-\mu_i^{-1}\lambda_i^2 - 2\lambda_iy_{r,i})  \nonumber \\
 \stackrel{2)}{=}& \min_{y\in \mathcal{B}}  \max_{\lambda \in \mathbb{R}^N} - \frac{1}{2} \|\lambda \|_2^2 -  \langle \lambda,  Ay  \rangle \nonumber \\
 =& \min_{y\in \mathcal{B}}  \|Ay \|_2^2, \nonumber
\end{align}
where $1)$ is obtained using the incidence relations $y_{r,i} = a_{r,i} = 0$ if $i\notin S_r$ and $2)$ holds because $\mu_i = |\{r\in [R] | i\in S_r\}|$. The optimal $y, a, \lambda$ satisfy $a_{r,i} = y_{r,i} + \mu_i^{-1}\lambda_i$ for all $i\in S_r$, $r\in[R]$ and $ \lambda = - Ay$. 

\section{Proof of Lemma~\ref{kappabound}}
First, consider a $y\in \mathcal{B}/\Xi$. We have $d_{I(\mu)}(y, \mathcal{Z}) = \|Ay+x^* \|_{2}$, since
\begin{align*}
\frac{1}{2}d_{I(\mu)}(y, \mathcal{Z})^2 &= \min_{a\in \mathcal{Z}} \frac{1}{2}\|y- a \|_{2,I(\mu)}^2  \\
& =\min_{a: a_{r,i} = 0,\; \forall (r, i): i\notin S_r} \max_{\lambda \in \mathbb{R}^N}\frac{1}{2}\|  y - a\|_{2, I(\mu)}^2 - \langle \lambda,  Aa +x^* \rangle \\
 &  \stackrel{1)}{=} \max_{\lambda \in\mathbb{R}^N}\min_{a \in\otimes_{r=1}^R\mathbb{R}^N}  \frac{1}{2}\sum_{r\in [R]}\sum_{i\in S_r}[\mu_i(y_{r,i} - a_{r,i})^2 - 2\lambda_i a_{r,i}] - \langle \lambda, x^* \rangle   \\
 & = \max_{\lambda \in\mathbb{R}^N}\frac{1}{2}\sum_{r\in [R]}\sum_{i\in S_r}[-\mu_i^{-1}\lambda_i^2 - \lambda_iy_{r,i}] - \langle \lambda, x^* \rangle  \\
 &  \stackrel{2)}{=} \max_{\lambda \in\mathbb{R}^N} -\frac{1}{2}\|\lambda \|_{2}^2 - \lambda^T (Ay+x^*) \\
 & = \frac{1}{2}\|Ay+x^* \|_{2}^2.
\end{align*}
where $1)$ is obtained using the incidence relations $y_{r,i} = a_{r,i} = 0$ if $i\notin S_r$ and $2)$ holds because $\mu_i = |\{r\in [R] | i\in S_r\}|$.
Based on Lemma~\ref{lemmakappabound}, we know that there exists a $\xi\in \mathcal{B}$ such that 
$A\xi = -x^*$ and 
$$\|y-\xi\|_{2,I(\mu)} \leq \sqrt{\frac{N\|I(\mu)\|_{1,\infty}}{2}} \|Ay + x^*\|_2 =  \sqrt{\frac{N\|\mu\|_1}{2}} \|Ay + x^*\|_2.$$ 
Therefore, $\kappa(y) = \frac{d_{I(\mu)}(y, \Xi)}{d_{I(\mu)}(y, \mathcal{Z})} \leq \sqrt{\frac{N\|\mu\|_1}{2}}$.

Next, consider a $y\in \mathcal{Z}/\Xi$. As $\mathcal{B}$ is compact, there exists a $y'\in\mathcal{B}$ that achieves $d_{I(\mu)}(y, \mathcal{B}) =  \|y-y'\|_{2,I(\mu)} $. Based on Lemma 3.1, we also know that there exists a $\xi\in \mathcal{B}$ such that $A\xi = -x^*$ and 
$$ \|\xi-y'\|_{2,I(\mu)} \leq \sqrt{\frac{\|I(\mu)\|_{1,\infty}}{2}} \|Ay' + x^*\|_1 = \sqrt{\frac{\|\mu\|_{1}}{2}} \|Ay' + x^*\|_1 .$$ 
Moreover, we have 
\begin{align*}
\|Ay' + x^*\|_1 = & \|Ay' -A y\|_1 \leq \|y' -y\|_1 =  \sum_{v\in[N]}\sum_{r: v\in S_r } |y'_{r,v} - y_{r,v}| \\
\leq &\sum_{v\in[N]} \left[ \mu_v \sum_{r: v\in S_r } (y'_{r,v} - y_{r,v})^2\right]^{\frac{1}{2}} \leq \sqrt{N}\|y'-y\|_{2,I(\mu)}.
\end{align*}
As $\xi\in \Xi$, it holds that $d_{I(\mu)}(y,\Xi) \leq \|\xi-y\|_{2,I(\mu)}\leq \|y'-y\|_{2,I(\mu)} +  \|y'-\xi\|_{2,I(\mu)} $. In addition, as 
$$\|y'-\xi\|_{2,I(\mu)} \leq\sqrt{\frac{\delta^s}{2}} \|Ay' + x^*\|_1\leq \sqrt{\frac{N\|\mu\|_1}{2}} \|y'-y\|_{2,I(\mu)},$$ 
we know that $d_{I(\mu)}(y,\Xi) \leq (1+\sqrt{\frac{N\|\mu\|_1}{2}})\|y'-y\|_{2,I(\mu)}$. Therefore,  
$$\kappa(y) = \frac{d_{I(\mu)}(y,\Xi)}{d_{I(\mu)}(y, \mathcal{B})} \leq \left(1+\sqrt{\frac{N\|\mu\|_1}{2}}\right),$$ 
which concludes the proof.

\section{Simulation for Example~\ref{boundofell}}\label{sec:lowerboundexp}
We provide additional empirical evidence that the convergence result suggested by the bound on $\ell_* \leq \frac{7}{N^2}$ is correct. We constructed a DSFM problem following Example~\ref{boundofell} and initialized $y$ according to equation~\eqref{exampleinit}. We used the number of iterations $k$ required to attain $g(y^{(k)}) \leq \epsilon g(y^{(0)})$ as a measure for the speed of convergence. We ran the simulations for $n\in[5,50]$ and averaged the results for each $n$ over $10$ independent runs. Figure~\ref{lower-bound} shows the results. The values next to the curves are their slopes obtained via a linear regression involving $\ln (\text{\# Iterations})\sim \ln(N)$. As the accuracy threshold increases, the slope approaches the value $3$, which indicates that the required number of iterations equals $O(N^2 R)$.     
\begin{figure}[t]
\centering
\includegraphics[trim={0cm 0cm 0cm 0cm},clip,width=.4\textwidth]{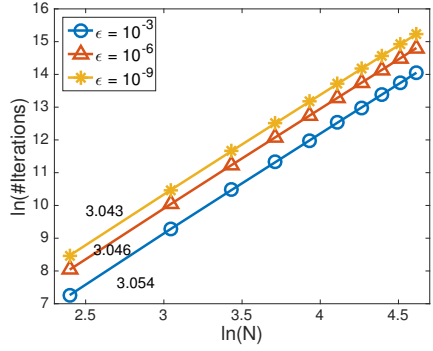}
\caption{Simulations accompanying Example 3.1: $\ln$(the number of iterations) vs $\ln(N)$. }
\label{lower-bound}
\end{figure}

\section{Proof of Lemma~\ref{skewstrongconv}}
Choose $z = Ay^*$ in Lemma~\ref{lemmakappabound}. Then, there is a $\xi\in \mathcal{B}$ such that $\|Ay - Ay^*\|^2 \geq \frac{2}{N\|\theta^P\|_{1,\infty}}\|y-\xi\|_{2,\theta^P}^2$. Moreover as $A\xi = z = Ay^*=-x^*$, we also  have $\xi\in\Xi$. Therefore, $\|y-\xi\|_{2,\theta^P}^2 \geq \|y-y^*\|_{2,\theta^P}^2$, which concludes the proof.

\section{Proof for Theorem~\ref{PCDrate}}
First, given a group of blocks $C$ and $y\in \otimes_{r=1}^R \mathbb{R}^{N}$, we define $y_{[C]}\in \otimes_{r=1}^R \mathbb{R}^{N}$ as
\begin{align*}
(y_{[C]})_r = \left\{ \begin{array}{cc} y_r & \quad\text{if $r\in C$,}\\
0 & \quad\text{if $r\not\in C$}.
\end{array} \right.
\end{align*}
The following lemma holds. 
\begin{lemma}\label{eqnorm2}
Let $C$ be a group of blocks sampled according to a $\alpha-$proper distribution $P$. Then, for any $y\in \otimes_{r=1}^R \mathbb{R}^{N}$ and $y_{r,i} = 0,$ whenever $i\notin S_r$, one has
\begin{align*}
\mathbb{E}_{C\sim P}(\|y_{[C]}\|_{2, I(\mu^C)}^2) = \mathbb{E}_{C\sim P}(\|y_{[C]}\|_{2,\theta^P}^2).
\end{align*}
\end{lemma}
\begin{proof}
%For $r\not \in C$, $\theta_r^P$ is an all zero matrix, so that
%\begin{align*}
%\mathbb{E}_{C\sim P}(\|y_{[C]}\|_{2, D_{C}'}^2) =\mathbb{E}_{C\sim P}(\|y\|_{2, D_{C}'}^2).  %\stackrel{b)}{=} \|y\|_{2, \frac{K}{R}(\frac{K-1}{R-1}(D'-I) +I)}^2 = \mathbb{E}_C(\|h_{[C]}\|_{2, \frac{K-1}{R-1}(D'-I) +I}^2) 
%\end{align*}
%Furthermore, for $r \in C$,
\begin{align*}
&\mathbb{E}_{C\sim P}(\|y_{[C]}\|_{2, I(\mu^C)}^2) =\mathbb{E}_{C\sim P}(\sum_{r\in C}\| y_{r}\|_{2, \mu^C}^2) = \sum_{r\in[R]}\mathbb{E}_{C\sim P}\left[ \| y_{r}\|_{2, \mu^C}^2 1_{r\in C}\right] \\
= & \sum_{r\in[R]}\mathbb{E}\left[1_{r\in C}\mathbb{E}_{C\sim P}\left[ \| y_{r}\|_{2, \mu^C}^2 |r\in C\right]\right]  = \sum_{r\in[R]}\mathbb{E}\left[1_{r\in C} \| y_{r}\|_{2, \theta_r^P}^2\right] = \mathbb{E}_{C\sim P}(\|y_{[C]}\|_{2,\theta^P}^2).
\end{align*}
\end{proof}

Next, we turn our attention to the proof of the theorem. For this purpose, suppose that $y^* = \arg\min_{y\in\Xi } \|y - y^{(k)}\|_{2,\theta^P}$. 
\subsection{Algorithm 1}
We start with by establishing the following results.
\begin{lemma} It can be shown that
\begin{align}
\langle \nabla g(y^{(k)}), y^*- y^{(k)}\rangle &\stackrel{1)}{\leq} g(y^*) - g(y^{(k)}) - \frac{1}{N\|\theta^P\|_{1,\infty}}\|y^{(k)}-y^*\|_{2,\theta^P}^2  \nonumber \\
& \stackrel{2)}{\leq} \frac{4}{N\|\theta^P\|_{1,\infty} + 2}\left[g(y^*) - g(y^{(k)})-  \frac{1}{2} \|y^{(k)}-y^*\|_{2,\theta^P}^2\right].\label{sc4}
\end{align} 
\end{lemma}
\begin{proof}
From Lemma 3.8 we can infer that
\begin{align}
& \|Ay^{(k)} -Ay^*\|_2^2 \geq \frac{2}{N\|\theta^P\|_{1,\infty}} \|y^{(k)}-y^*\|_{2,\theta^P}^2 \Rightarrow \nonumber  \\
& g(y^*) \geq g(y^{(k)}) + \langle \nabla g(y^{(k)}), y^*- y^{(k)}\rangle + \frac{1}{N\|\theta^P\|_{1,\infty}}\|y^{(k)}-y^*\|_{2,\theta^P}^2, \label{sc1} \\
& g(y^{(k)})  \geq g(y^*) +\langle \nabla g(y^{*}),  y^{(k)} -  y^*\rangle + \frac{1}{N\|\theta^P\|_{1,\infty}}\|y^{(k)}-y^*\|_{2,\theta^P}^2. \label{sc2} 
\end{align} 
As $\langle \nabla g(y^{*}),  y^{(k)} -  y^*\rangle \geq 0$, \eqref{sc2} gives 
\begin{align}
 g(y^*) -  g(y^{(k)})\leq  -\frac{1}{N\|\theta^P\|_{1,\infty}}\|y^{(k)}-y^*\|_{2,\theta^P}^2. \label{sc3}
\end{align} 
The inequality~\eqref{sc1} yields claim 1) in~\eqref{sc4}. Claim 2) in~\eqref{sc4} follows from~\eqref{sc3}.
%Combining \eqref{sc1}, \eqref{sc2} and  \eqref{sc3}, we can prove \eqref{sc4}. 
\end{proof}

The following lemma is a direct consequence of the optimality of $y_r^{(k+1)}$ for an oblique projection. 
\begin{lemma}\label{optcond}
\begin{align*}
\langle \nabla_r g(y^{(k)}), y_r^{(k+1)} - y_r^*\rangle \leq \langle y_r^{(k)} - y_r^{(k+1)}, y_r^{(k+1)} - y_r^*\rangle_{\theta_r^P}. 
\end{align*}
\end{lemma}
The following lemma follows from a simple manipulation of the Euclidean norm.
\begin{lemma}\label{threepoints}
\begin{align*}
\frac{1}{2}\|y_r^{(k+1)} - & y_r^{(k)}\|_{2,\theta_r^P}^2 = \frac{1}{2}\|y_r^{(k+1)} - y_r^{*}\|_{2,\theta_r^P}^2  +\frac{1}{2} \|y_r^{*} - y_r^{(k)}\|_{2,\theta_r^P}^2 + \langle y_r^{(k+1)} - y_r^{*},  y_r^{*} - y_r^{(k)}\rangle_{\theta_r^P} \nonumber \\
 = &- \frac{1}{2}\|y_r^{(k+1)} - y_r^{*}\|_{2,\theta_r^P}^2   + \frac{1}{2} \|y_r^{*} - y_r^{(k)}\|_{2,\theta_r^P}^2 +  \langle y_r^{(k+1)} - y_r^{*},  y_r^{(k+1)} - y_r^{(k)}\rangle_{\theta_r^P} 
\end{align*}
\end{lemma}
%The following lemma shows that the norm $\|\cdot\|_{2, D_{C}'}^2$ that depends on the picked groung $C$ performs equivalently to the norm $\|\cdot\|_{2,\theta^P}^2$ in expectation.

Let us analyze next the amount by which the objective function decreases in each iteration. The following expectation is with respect to $C_{i_k}\sim P$. 
\begin{align}
&\mathbb{E}\left[g(y^{(k+1)})\right]  \\
& \stackrel{1)}{\leq}  g(y^{(k)}) +\mathbb{E}\left\{ \sum_{r\in C_{i_k}} \left[\langle \nabla_r g(y^{(k)}), y_r^{(k+1)}-y_r^{(k)}\rangle + \frac{1}{2}\| y_r^{(k+1)}-y_r^{(k)}\|_{2, \mu_r^{C_{i_k}}}^2\right]\right\} \nonumber \\
& \stackrel{2)}{=}  g(y^{(k)}) +\mathbb{E}\left\{ \sum_{r\in C_{i_k}} \left[\langle \nabla_r g(y^{(k)}), y_r^{(k+1)}-y_r^{(k)}\rangle + \frac{1}{2}\| y_r^{(k+1)}-y_r^{(k)}\|_{2, \theta_r^P}^2\right]\right\} \nonumber \\
& = g(y^{(k)}) + \mathbb{E}\left\{\sum_{r\in C_{i_k}} \left[\langle \nabla_r g(y^{(k)}), y_r^{*}-y_r^{(k)}\rangle + \langle \nabla_r g(y^{(k)}), y_r^{(k+1)}-y_r^{*}\rangle  + \frac{1}{2}\| y_r^{(k+1)}-y_r^{(k)}\|_{2,\theta_r^P}^2 \right]\right\} \nonumber \\
& \stackrel{3)}{\leq} g(y^{(k)}) +  \mathbb{E}\left\{\sum_{r\in C_{i_k}} \left[\langle \nabla_r g(y^{(k)}), y_r^{*}-y_r^{(k)}\rangle - \frac{1}{2}\|y_r^{(k+1)} - y_r^{*}\|_{2,\theta_r^P}^2   + \frac{1}{2} \|y_r^{*} - y_r^{(k)}\|_{2,\theta_r^P}^2  \right]\right\} \nonumber \\
&= g(y^{(k)}) + \alpha\langle \nabla g(y^{(k)}), y^{*}-y^{(k)}\rangle - \mathbb{E}\left[ \frac{1}{2}\|y_{[C_{i_k}]}^{(k+1)} - y_{[C_{i_k}]}^{*}\|_{2,\theta^P}^2 \right]+ \mathbb{E}\left[ \frac{1}{2}\|y_{[C_{i_k}]}^{(k)} - y_{[C_{i_k}]}^{*}\|_{2,\theta^P}^2\right] \nonumber \\
&\stackrel{4)}{=} g(y^{(k)}) + \alpha\langle \nabla g(y^{(k)}), y^{*}-y^{(k)}\rangle - \mathbb{E}\left[ \frac{1}{2}\|y^{(k+1)} - y^{*}\|_{2,\theta^P}^2 \right]+ \mathbb{E}\left[ \frac{1}{2}\|y^{(k)} - y^{*}\|_{2,\theta^P}^2\right] \nonumber \\
& \stackrel{5)}{\leq} g(y^*) - \mathbb{E}\left[ \frac{1}{2}\|y^{(k+1)} - y^{*}\|_{2,\theta^P}^2 \right] + \left[1- \frac{4\alpha}{N\|\theta^P\|_{1,\infty} + 2}\right]\left\{g(y^{(k)}) - g(y^*) -  \frac{1}{2}\|y^{(k)} - y^{*}\|_{2,\theta^P}^2\right\}, \label{onestep}
\end{align}
where $1)$ follows from inequality~\eqref{paradescent}, $2)$ holds due to Lemma~\ref{eqnorm2}, $3)$ is a consequence of Lemma~\ref{optcond} and Lemma~\ref{threepoints}, $4)$ is due to $y^{(k+1)}_r = y^{(k)}_r$ for $r\notin C_{i_k}$, and $5)$ may be established from \eqref{sc4}.

Equation~\eqref{onestep} further establishes that 
\begin{align*}
\mathbb{E}\left[g(y^{(k+1)})- g(y^*) + \frac{1}{2}d_{\theta^P}^2(y^{k+1}, \xi) \right] \leq& \mathbb{E}\left[g(y^{(k+1)})- g(y^*) + \frac{1}{2}\|y^{(k+1)} - y^{*}\|_{2,\theta^P}^2  \right] \\
\leq & \left[1- \frac{4\alpha}{N\|\theta^P\|_{1,\infty} + 2}\right] \mathbb{E}\left[g(y^{(k)})- g(y^*) + \frac{1}{2}d^2_{\theta^P}(y^{k}, \xi) \right] .
\end{align*}
The proof follows by repeating the derivations for all $k$. 

\section{Proof of Lemma~\ref{lowerbound}}
According to the definition of $\theta^P$, we have 
\begin{align}
\max_{r\in[R]: i\in S_r} \theta_{r,i}^P &= \max_{r\in[R]: i\in S_r} \mathbb{E}_{C\sim P}\left[\mu_i^C| r\in C\right] \nonumber \\
&= \max_{r\in[R]: i\in S_r}  \mathbb{E}_{C\sim P}\left[\sum_{r'\in[R] : i\in S_{r'}} 1_{r'\in C} | r\in C\right] \nonumber \\
& =    \max_{r\in[R]: i\in S_r} \sum_{r'\in[R] : i\in S_{r'}} \mathbb{P}_{C\sim P}\left[r'\in C | r\in C\right] \label{greater1} \\
& =  \frac{1}{\alpha} \max_{r\in[R]: i\in S_r} \sum_{r'\in[R] : i\in S_{r'}} \mathbb{P}_{C\sim P}\left[r'\in C, r\in C\right]  \nonumber \\
  &\geq \frac{1}{\alpha \mu_i}  \sum_{r,r'\in[R] : i\in S_{r}, S_{r'}} \mathbb{P}_{C\sim P}\left[r'\in C, r\in C\right]  \nonumber \\
  & = \frac{1}{\alpha \mu_i} \mathbb{E}_{C\sim P}\left[|\{(r,r')\in C\times C: i\in S_{r}, i\in S_{r'}\}|\right] \nonumber\\
& =  \frac{1}{\alpha \mu_i} \mathbb{E}_{C\sim P}\left[(\mu_{i}^C)^2\right]  \geq \frac{1}{\alpha d_i}  \left[\mathbb{E}_{C\sim P}(\mu_{i}^C)\right]^2   = \frac{1}{\alpha d_i}  \left(\sum_{C} \sum_{r\in[R]:i\in S_r} 1_{r\in C} \mathbb{P}(C)\right)^2  \nonumber\\
& =  \frac{1}{\alpha \mu_i}  \left( \sum_{r\in[R]:i\in S_r} \mathbb{P}_{C\sim P}[r\in C]\right)^2   \nonumber \\
& =\frac{1}{\alpha \mu_i}  \left(\alpha \mu_i\right)^2 =\alpha \mu_i. \nonumber
\end{align}
From~\eqref{greater1}, we also have $\sum_{r'\in[R]: i\in S_{r'}} \mathbb{P}_{C\sim P}\left[r'\in C | r\in C\right] \geq \mathbb{P}_{C\sim P}\left[r\in C | r\in C\right] =1$, which proves the claimed result. 

\section{Proof of Lemma~\ref{eqnorm}}
Similar to what was established for~\eqref{greater1}, one can show that $\theta_{r,i}^P = \sum_{r'\in[R] : i\in S_{r'}} \mathbb{P}_{C\sim P}\left[r'\in C | r\in C\right]$. 

Consider next the right hand side of this equation for $\alpha = \frac{K}{R}$. In this case, for some $r$ and some $i\in S_r$, we have
\begin{align*}
\sum_{r'\in[R] : i\in S_{r'}} \mathbb{P}_{C\sim P}\left[r'\in C | r\in C\right] &= \mathbb{P}_{C\sim P}\left[r\in C | r\in C\right] \; + \sum_{r'\in[R] : i\in S_{r'}, r'\neq r} \mathbb{P}_{C\sim P}\left[r'\in C | r\in C\right] \\
&= 1 +  \frac{R}{K}\sum_{r' : i\in S_{r'}, r'\neq r} \mathbb{P}_{C\sim P}\left[r'\in C , r\in C\right] \\
&  = 1+ \frac{R}{K} (\mu_i - 1) \frac{{R-2 \choose K-2}}{{R \choose K}} = 1 + \frac{K-1}{R-1}(\mu_i - 1).
\end{align*} 
Therefore, $\theta_{r,i}^P= \frac{K-1}{R-1} \mu_i + \frac{R-K}{R-1} $ when $P$ is a uniform distribution. 

\section{Analysis of the Accelerated Coordinate Descend Method}\label{sec:ACDM}
In the ACDM setting, we used the APPROX framework proposed by Fercoq and Richt$\acute{\text{a}}$rik in~\cite{fercoq2015accelerated} and adapted it to this particular problem. In the general APPROX framework, the norm in \eqref{paradescent} is chosen as follows: consider an arbitrary function $\phi$ with the component-wise smoothness and strong convexity property. For block $r$, one has $|\nabla_r \phi(x) -\nabla_r \phi(y)| \leq L_r \| x_r-y_r\|_{\nu_r},$ where $\|\cdot\|_{\nu_r}$ is a norm associated with the $r$-th block. If one wants to process multiple blocks simultaneously, say those in a group $C$, one first needs to find a constant $L_C$ such that for any $h$ as defined in~\eqref{paradescent}, it holds that
$$\phi(y+h) \;\leq \; \phi(y) + \sum_{r\in C} \langle \nabla_r \phi(y),  h_{r}\rangle + \sum_{r\in C} \frac{L_C}{2}\| h_{r}\|_{2, \nu_r}^2.$$ 
The smaller the value of the multiplier $L_C$, the faster the convergence. Typically, $L_C$ lies in $[\max_{r\in C} L_r, \sum_{r\in C} L_r]$.  

Recall the smoothness property of $g$ from equation~\eqref{smoothness}. A direct application of APPROX to our problem gives
\begin{align*}
g(y+h) \;\leq \; g(y) + \sum_{r\in C} \langle \nabla_r g(y),  h_{r}\rangle + \sum_{r\in C} \frac{\max_{i\in[N]} \mu_{i}^C}{2}\| h_{r}\|_{2}^2 . 
\end{align*}
As $(\max_{i\in[N]} \mu_{i}^C) \geq \mu_{j}^C$ for all $j\in[N]$, we obtain convergence rates worse than those implied by inequality~\eqref{paradescent}. To actually obtain the guarantees in~\eqref{paradescent}, one needs to dispose with the $\|\cdot \|_{2}$ norm at the block level and break the blocks into components corresponding to the individual elements. The elements are evaluated independently through the use of the norm $\|\cdot \|_{2,\mu^C}$.
%Second, in both PRCD and PACD, we suggest the dynamic norm for PCAD as well while Fercoq and Richt's ACD framework makes use of the average norm $\|\cdot\|_{2, D_{\mathcal{P}}'}$ instead. In experiments, we find it improves the convergence although requires some effort to compute $D_{C}'$ in each iteration. In our case, the computation of $D_{C}'$ is fairly simple and efficient compared to the projection/descent step. According to the definition, $D_{C}'$ contains the degree of items within the corresponding subgraph generated by group $C$. If each $S_r$ is given as a list which is suitable for sparse case, it becomes just a problem to count the frequency of characters in lists. If $S_r$ is given a column of $H$ (for non-sparse case), this is just the sum of $K$ columns of $H$. We think this idea of dynamic norm should be suitable for more general parallel coordinate descent algorithms but the analysis of this is out of the scope of this paper so we leave it to the future. 

\begin{table}[htb]
\centering
\begin{tabular}{l}
\hline
\textbf{Algorithm 2: } \textbf{Parallel ACDM for Solving~\eqref{compact}} \\
\hline
\label{PACD}
\ \textbf{Input}: $\mathcal{B}$, $\alpha$, some constant $c > 0$ \\
\ 0: Initialize $y^{(0)}\in\mathcal{B}$, $k\leftarrow 0$\\
\ 1: $c' \leftarrow \left\lceil(1+c)\frac{\sqrt{2N\|\theta^P\|_{1,\infty}}}{\alpha} + c\right\rceil $ \\
\ 2: Do the following steps iteratively until the dual gap $<\epsilon$:\\
\ 3: \;\quad If $k = lc'$ for some $l\in \mathbb{Z}$, $z^{(k)} \leftarrow y^{(k)}$, $\lambda_k \leftarrow 1$ \\
%\ 3: \quad Generate one group $C$ by following categorical distribution with $\mathbb{P}[C = C_{i}] = w_i$, $i\in[m]$ .  \\
\ 4: \;\quad $p^{(k)} \leftarrow (1- \lambda_k)y^{(k)} + \lambda_kz^{(k)}$ \\
\ 5: \;\quad Sample $C_{i_k}$ using some  $\alpha$-proper distribution $P$  \\
\ 6: \;\quad $z^{(k+1)}\leftarrow z^{(k)}$ \\
\ 7: \;\quad For $r\in C_{i_k}$: \\
\ 8: \;\;\;\;\;\quad  $z_r^{(k+1)} \leftarrow \Pi_{\mathcal{B}_r,\theta_r^P} (z_r^{(k)} - \frac{\alpha}{\lambda_k} (\theta_r^P)^{-1} \odot \nabla_r g(p^{(k)}))$ \\
\ 9: \;\quad $y^{(k+1)}\leftarrow p^{(k)} + \frac{\lambda_k}{\alpha}(z^{(k+1)} - z^{(k)})$  \\
\ 10:\;\quad $\lambda_{k+1} \leftarrow \frac{\sqrt{\lambda_k^4 + 4 \lambda_k^2}-\lambda_k^2}{2}$ \\
\ 11:\;\quad $k\leftarrow k+1$ \\
\ 12: Output $y^{(k)}$\\
\hline
\end{tabular}
\end{table}
%\textcolor{red}{stopped here}
Similar to the APPROX method~\cite{fercoq2015accelerated}, the parallel ACDM can also be implemented to avoid full-dimensional vector operations (see Section~\ref{blockversion}). The following theorem characterizes the convergence property of Algorithm 2. 
\begin{theorem}\label{paraACD}
Given $c>0$, let $c'=\left\lceil(1+c)\frac{\sqrt{2N\|\theta^P\|_{1,\infty} }}{\alpha} + c\right\rceil$. 
Consider the iterations $k=lc'$ for $l\in \mathbb{Z}_{\geq 0}$. Then, $y^{(k)}$ of Algorithm 2 satisfies
\begin{align*}
\mathbb{E}\left[g(y^{(k)}- g(y^*)\right] \leq  \frac{1}{(1+c)^l}\left[g(y^{(0)}) - g(y^*)\right] .
\end{align*}
\end{theorem}

\subsection{Proof of Theorem~\ref{paraACD}}
We start by establishing a number of background results. 

The following lemma is due to the optimality of $z_r^{(k+1)}$. 
\begin{lemma} \label{PACDpoptcond}
\begin{align*}
\langle \nabla_r g(p^{(k)}), z_r^{(k+1)} - y_r^*\rangle \leq \frac{\lambda_k}{\alpha}\langle z_r^{(k)} - z_r^{(k+1)}, z_r^{(k+1)} - y_r^*\rangle_{\theta_r^P}.
\end{align*}
\end{lemma}
Once again, one can easily establish the following result pertaining to the Euclidean norm.
\begin{lemma}\label{PACDpthreepoints}
\begin{align*}
\frac{1}{2}\|z_r^{(k+1)} - & z_r^{(k)}\|_{2,\theta_r^P}^2 = \frac{1}{2}\|z_r^{(k+1)} - y_r^{*}\|_{2,\theta_r^P}^2  +\frac{1}{2} \|y_r^{*} - z_r^{(k)}\|_{2,\theta_r^P}^2 + \langle z_r^{(k+1)} - y_r^{*},  y_r^{*} - z_r^{(k)}\rangle_{\theta_r^P} \nonumber \\
 = &- \frac{1}{2}\|z_r^{(k+1)} - y_r^{*}\|_{2,\theta_r^P}^2   + \frac{1}{2} \|y_r^{*} - z_r^{(k)}\|_{2,\theta_r^P}^2 +  \langle z_r^{(k+1)} - y_r^{*},  z_r^{(k+1)} - z_r^{(k)}\rangle_{\theta_r^P}. 
\end{align*}
\end{lemma}

The next result follows from the convexity property of the function $g$. 
\begin{lemma}  \label{PACDpeq3}
\begin{align}
\lambda_k\langle \nabla g(p^{(k)}), y^{*}-z^{(k)}\rangle &=  \langle \nabla g(p^{(k)}), \lambda_k y^{*}- \lambda_k z^{(k)}\rangle = \langle \nabla g(p^{(k)}), \lambda_k y^{*}-(p^{(k)} - (1-\lambda_k)y^{(k)}) \rangle \nonumber \\ 
&= \lambda_k\langle \nabla g(p^{(k)}), y^{*}-p^{(k)}\rangle + (1-\lambda_k)\langle \nabla g(p^{(k)}),y^{(k)}-p^{(k)}\rangle \nonumber \\ 
&\leq \lambda_k\left[g(y^{*})- g(p^{(k)})\right] + (1-\lambda_k)\left[g(y^{(k)})- g(p^{(k)})\right].\nonumber 
\end{align}
\end{lemma}
We are now ready to analyze the decrease of the objective function in each iteration of Algorithm 2. The expectation in the following equations is performed with respect to $C_{i_k}\sim P$. 
\begin{align}
&\quad  \mathbb{E}\left[g(y^{(k+1)})\right] \nonumber \\
& \stackrel{1)}{\leq}  g(p^{(k)}) +\frac{\lambda_k}{\alpha}\mathbb{E}\left\{ \sum_{r\in C_{i_k}} \left[\langle \nabla_r g(p^{(k)}), z_r^{(k+1)}-z_r^{(k)}\rangle + \frac{\lambda_k}{2\alpha}\| z_r^{(k+1)}-z_r^{(k)}\|_{2,\mu^C}^2\right]\right\} 
\nonumber \\
&\stackrel{2)}{=}  g(p^{(k)}) +\frac{\lambda_k}{\alpha}\mathbb{E}\left\{ \sum_{r\in C_{i_k}} \left[\langle \nabla_r g(p^{(k)}), z_r^{(k+1)}-z_r^{(k)}\rangle + \frac{\lambda_k}{2\alpha}\| z_r^{(k+1)}-z_r^{(k)}\|_{2,\theta_r^P}^2\right]\right\} 
\nonumber \\
& = g(p^{(k)}) + \frac{\lambda_k}{\alpha}\mathbb{E}\left\{\sum_{r\in C_{i_k}} \left[\langle \nabla_r g(p^{(k)}), y_r^{*}-z_r^{(k)}\rangle + \langle \nabla_r g(p^{(k)}), z_r^{(k+1)}-z_r^{*}\rangle  + \frac{\lambda_k}{2\alpha}\| z_r^{(k+1)}-z_r^{(k)}\|_{2,\theta_r^P}^2 \right]\right\} \nonumber \\
& \stackrel{3)}{\leq} g(p^{(k)}) +  \frac{\lambda_k}{\alpha}\mathbb{E}\left\{\sum_{r\in C_{i_k}} \left[\langle \nabla_r g(p^{(k)}), y_r^{*}-z_r^{(k)}\rangle - \frac{\lambda_k}{2\alpha}\|z_r^{(k+1)} - y_r^{*}\|_{2,\theta_r^P}^2   + \frac{\lambda_k}{2\alpha} \|y_r^{*} - z_r^{(k)}\|_{2,\theta_r^P}^2  \right]\right\} \nonumber \\
&= g(p^{(k)}) + \lambda_k\langle \nabla g(p^{(k)}), y^{*}-z^{(k)}\rangle +  \frac{\lambda_k^2}{2\alpha^2} \mathbb{E}\left[\|z_{[C_{i_k}]}^{(k)} - y_{[C_{i_k}]}^{*}\|_{2,\theta^P}^2 -\|z_{[C_{i_k}]}^{(k+1)} - y_{[C_{i_k}]}^{*}\|_{2,\theta^P}^2\right] \nonumber \\
&\stackrel{4)}{=} g(p^{(k)}) + \lambda_k\langle \nabla g(p^{(k)}), y^{*}-z^{(k)}\rangle +  \frac{\lambda_k^2}{2\alpha^2} \mathbb{E}\left[\|z^{(k)} - y^{*}\|_{2,\theta^P}^2 -\|z^{(k+1)} - y^{*}\|_{2,\theta^P}^2\right] \nonumber \\
& \stackrel{5)}{=} g(y^{*}) +(1-\lambda_k)\left[g(y^{(k)})- g(y^{*})\right] +  \frac{\lambda_k^2}{2\alpha^2} \left\{\|z^{(k)} - y^{*}\|_{2,\theta^P}^2 -\mathbb{E}\left[\|z^{(k+1)} - y^{*}\|_{2,\theta^P}^2\right]\right\},
\label{onestep2}
\end{align}
where $1)$ follows from~\eqref{paradescent}, $2)$ may be deduced from Lemma~\ref{eqnorm2}, $3)$ is a consequence of Lemma~\ref{PACDpoptcond} and Lemma~\ref{PACDpthreepoints}, $4)$ is due to the fact that $y^{(k+1)}_r = y^{(k)}_r$ for $r\notin C_{i_k},$ and $5)$  follows from Lemma~\ref{PACDpeq3}. 

Based on the definition of $\{\lambda_k\}_{k\geq 0}$, we also have
\begin{align} \label{theta} 
\frac{1-\lambda_k}{\lambda_k^2} = \frac{1}{\lambda_{k-1}^2},\quad 0<\lambda_{k+1}\leq\lambda_{k}\leq \frac{2}{k+2/\lambda_0} = \frac{2}{k+2}. 
\end{align}
Hence, combining the above expression with~\eqref{onestep2}, for $k\in[1,\frac{2}{\alpha}\lceil\sqrt{N\|\theta^P\|_{1,\infty}}\rceil+1]$, we have
\begin{align}
&\mathbb{E}\left[\frac{1-\lambda_{k}}{\lambda_k^2}\left[g(y^{(k)}) - g(y^{*})\right] +  \frac{1}{2\alpha^2}\|z^{(k)} - y^{*}\|_{2,\theta^P}^2 \right] \nonumber \\
  & = \mathbb{E}\left[\frac{1}{\lambda_{k-1}^2}\left[g(y^{(k)}) - g(y^{*})\right] +  \frac{1}{2\alpha^2}\|z^{(k)} - y^{*}\|_{2,\theta^P}^2 \right] \nonumber\\
 &  \leq  \mathbb{E}\left[\frac{1-\lambda_{k-1}}{\lambda_{k-1}^2}\left[g(y^{(k-1)}) - g(y^{*})\right] +  \frac{1}{2\alpha^2}\|z^{(k-1)} - y^{*}\|_{2,\theta^P}^2 \right]\nonumber \\
 & \leq \dots \leq  \frac{(1-\lambda_0)}{\lambda_0^2}\left[g(y^{(0)}) - g(y^{*})\right] +  \frac{1}{2\alpha^2}\|z^{(0)} - y^{*}\|_{2,\theta^P}^2. \label{PACDpbound}
\end{align}
Lemma~\ref{skewstrongconv} implies the strong convexity property as
\begin{align}
 \|Ay^{(k)} -Ay^*\|_2^2& \geq \frac{2}{N\|\theta^P\|_{1,\infty}} \|y^{(k)}-y^*\|_{2,\theta^P}^2 \Rightarrow \nonumber  \\
 g(y^{(k)}) - g(y^*)  &\geq \langle \nabla g(y^{*}),  y^{(k)} -  y^*\rangle + \frac{1}{N\|\theta^P\|_{1,\infty}}\|y^{(k)}-y^*\|_{2,\theta^P}^2  \nonumber \\
&\stackrel{1)}{\geq}  \frac{1}{N\|\theta^P\|_{1,\infty}}\|y^{(k)}-y^*\|_{2,\theta^P}^2. \label{sc5} 
\end{align} 
Here, $1)$ holds since $y^*$ is an optimal solution of $\min_y g(y)$ and thus $\langle \nabla g(y^{*}),  y^{(k)} -  y^*\rangle \geq 0$.

Combining \eqref{theta}, \eqref{PACDpbound} and \eqref{sc5}, we obtain  
\begin{align*}
\mathbb{E}\left[g(y^{(k)}) - g(y^{*})\right] & \leq \lambda_{k-1}^2\left[\frac{1-\lambda_0}{\lambda_0^2}(g(y^{(0)}) - g(y^{*}))+  \frac{1}{2\alpha^2}\|y^{(0)} - y^{*}\|_{2,\theta^P}^2\right] \\
&\leq\left( \frac{2}{k+1}\right)^2  \frac{1}{2\alpha^2}\|y^{(0)} - y^{*}\|_{2,\theta^P}^2  \\
&\leq \left( \frac{2}{k+1}\right)^2 \frac{N\|\theta^P\|_{1,\infty}}{2\alpha^2}(g(y^{(0)}) - g(y^{*})).
\end{align*}
Therefore, for $k =\left\lceil(1+c)\frac{\sqrt{2N\|\theta^P\|_{1,\infty} }}{\alpha} + c\right\rceil$, we have
\begin{align*}
\mathbb{E}\left[g(y^{\left(\left\lceil(1+c)\frac{\sqrt{2N\|\theta^P\|_{1,\infty} }}{\alpha} + c\right\rceil \right)}) - g(y^{*})\right] \leq \frac{1}{1+c}(g(y^{(0)}) - g(y^{*})).
\end{align*}
For each value of $k = l \times \left\lceil(1+c)\frac{\sqrt{2N\|\theta^P\|_{1,\infty} }}{\alpha} + c\right\rceil$, $l\in \mathbb{Z}_{\geq 0}$, the values $z^{(k)}$, $\lambda_k$ are reinitialized. Using a similar proof as above, we have
\begin{align*}
\mathbb{E}\left[g(y^{\left((l+1)\times\left\lceil(1+c)\frac{\sqrt{2N\|\theta^P\|_{1,\infty} }}{\alpha} + c\right\rceil \right)}) - g(y^{*})\right] \leq \frac{1}{1+c}\left[g(y^{\left(l\times\left\lceil(1+c)\frac{\sqrt{2N\|\theta^P\|_{1,\infty} }}{\alpha} + c\right\rceil \right)})- g(y^{*})\right].
\end{align*}
Therefore, 
\begin{align*}
\mathbb{E}\left[g(y^{\left(l\left\lceil(1+c)\frac{\sqrt{2N\|\theta^P\|_{1,\infty} }}{\alpha} + c\right\rceil \right)}) - g(y^{*})\right] \leq \frac{1}{(1+c)^l}(g(y^{(0)}) - g(y^{*})).
\end{align*}
This concludes the proof. 

\subsection{Avoiding Full-Dimensional Vector Operations}\label{blockversion}
\begin{table}[htb]
\centering
\begin{tabular}{l}
\hline
\textbf{Algorithm 3:} \textbf{Parallel ACDM for Solving Problem (7) (an efficient implementation)} \\
\hline
\label{PACD}
\ \textbf{Input}: $\mathcal{B}$, $\alpha$ \\
\ 0: Initialize $z^{(0)}\in\mathcal{B}$, $u^{(0)}\leftarrow 0 \in \mathbb{R}^N$, $k\leftarrow 0$.\\
\ 1: Do the following steps iteratively until the dual gap $<\epsilon$:\\
\ 2: \quad If $k = l\left\lceil(1+c)\sqrt{\frac{2N\|\theta^P\|_{1,\infty}}{\alpha}} +c\right\rceil$ for some $l\in \mathbb{Z}$ and $c>0$, \\
\quad\quad\quad\quad $z^{(k)} \leftarrow z^{(k)} + \lambda_{k-1}^2 u^{(k)}$, $u^{(k)}\leftarrow 0$, $\lambda_k \leftarrow 1$ \\
%\ 3: \quad Generate one group $C$ by following categorical distribution with $\mathbb{P}[C = C_{i}] = w_i$, $i\in[m]$ .  \\
\ 3: \quad Sample one set $C_{i_k}$ according to a $\alpha$-proper distribution $P$  \\
\ 4: \quad For $r\in C_{i_k}$: \\
\ 5: \quad\quad $\triangle z_r\leftarrow \arg\min_{\triangle z + z_r^{(k)}\in B_r} \|\triangle z + \frac{\alpha}{\lambda_k}(\theta_r^P)^{-1}\odot\nabla_r g(z^{(k)}+ \lambda_k^2 u^{(k)})\|_{2,\theta_r^P}^2$ \\
\ 6: \quad\quad $z_r^{(k+1)} \leftarrow z_r^{(k)}+\triangle z_r $ \\
\ 7: \quad \quad $u_r^{(k+1)} \leftarrow u_r^{(k)} +\frac{\lambda_k-\alpha}{\alpha\lambda_k^2}\triangle z_r$ \\
\ 8: \quad $\lambda_{k+1} \leftarrow \frac{\sqrt{\lambda_k^4 + 4 \lambda_k^2}-\lambda_k^2}{2}$  \\
\ 9: \quad $k\leftarrow k+1$ \\
\ 10: Output $y^{(k)}$ \\
% $y^{(k+1)}\leftarrow p^{(k)} + \frac{\lambda_k}{\alpha}(z_r^{(k+1)} - z_r^{(k)})$  \\
\hline
\end{tabular}
\end{table}
Algorithm 2 can be implemented without full-dimensional vector operations. In each step, only those coordinates within the blocks in $C$ are updated. Consequently, one only needs to replace $p^{(k)}$ and $y^{(k)}$ with $p^{(k)} = z^{(k)} + \lambda_{k}^2 u^{(k)}$ and $y^{(k)} = z^{(k)} + \lambda_{k-1}^2 u^{(k)}$, where $u^{(k)}$ is a new variable described in Algorithm 3. 

\section{Minimization of $\|\theta^P\|_{1,\infty}$}\label{sec:comb}
 \begin{table}[htb]
\centering
\begin{tabular}{l}
\hline
\label{Greedy}
\textbf{Algorithm 4: } \textbf{A Greedy Algorithm for Finding a Balanced-Partition Distribution} \\
\hline
\ \textbf{Input}: $\{S_r\}_{r\in[R]}$, $K$ \\
\ 0: Initialize the partition $\mathcal{C} = \{C_i\}_{1\leq i \leq m}$, $C_i\leftarrow \emptyset$, vectors $\{\mu^{C_i}\}_{1\leq i\leq m}$, $\mu^{C_i}\in \mathbb{R}^N$,  \\
\quad and $\mu^{\max}\in \mathbb{R}^N$, $\mu^{\max}\leftarrow 0$. \\
\ 1: For $r$ from 1 to $R$:\\
\ 2: \quad For $i$ from 1 to $m$:\\
\ 3: \quad\quad If $|C_i|< K$: \\
\ 4: \quad \quad \quad $\triangle \mu^{C_i} \leftarrow 0$  \\
\ 5: \quad \quad \quad For $v$ in $S_r$, if $\mu_v^{C_i}$ is equal to $\mu_v^{\max}$, $\triangle \mu^{C_i} \leftarrow \triangle \mu^{C_i} + 1$ \\
\ 6: \quad \quad else: $\triangle \mu^{C_i}\leftarrow \infty$ \\
\ 7: \quad $i^* \leftarrow \arg\min_{i} \triangle \mu^{C_i}$ \\
\ 8: \quad $C_{i^*} \leftarrow C_{i^*}\cup \{r\}$  \\
\ 9: \quad For $v$ in $S_r$, $\mu_v^{C_{i^*}}\leftarrow \mu_v^{C_{i^*}} + 1$, $\mu_v^{\max} \leftarrow \max\{\mu_v^{\max} , \mu_v^{C_{i^*}}\}$. \\
\ 10: Output $\mathcal{C}$. \\
\hline
\end{tabular}
\end{table}
We first define $\triangle_* \triangleq  \max_{r\in [R]} |\{r'\in [R]|S_{r'} \cap S_{r}\neq \emptyset\}|$, which we use in our subsequent derivations.

As shown in Theorem~\ref{PCDrate} and Theorem~\ref{paraACD}, $\|\theta^P\|_{1,\infty}$ plays an important role in the convergence rate of CDMs. Hence, we are interested in identifying the optimal sampling strategy $P$ that minimizes $\|\theta^P\|_{1,\infty}$. 

For this purpose, consider a partition of $[R]$ into $m=\lceil\frac{1}{\alpha}\rceil$ parts $\{C_i\}_{1\leq i \leq m},$ such that $|C_{i}| \in \{K-1, K\}$. We refer to such a partition as a \emph{balanced partition}. In this case, every block $r$ is in exactly one component $C_i$ and $\|\theta^P\|_{1,\infty} = \sum_{v\in[N]} \max_{i\in[m]} \mu_v^{C_i}$. As a result, the problem of minimizing $\|\theta^P\|_{1,\infty}$ is closely related to the so called \emph{equitable coloring problem} first proposed by Meyer~\cite{meyer1973equitable}.
\begin{definition}[Meyer~\cite{meyer1973equitable}]
Given a graph, an \emph{equitable coloring} is an assignment of colors to the vertices that satisfies the following two properties: no two adjacent vertices share the same color and the number of vertices in any two color classes differs by at most one. Moreover, the minimum number of colors in any equitable coloring is termed the \emph{equitable coloring number}.
\end{definition}
Hajnal-Szemer{\'e}di's Theorem~\cite{hajnal1970proof} established one of the most important results in equitable graph coloring: a graph is equitably $k$-colorable if $k$ is strictly greater than the maximum vertex degree. This bound is tight. We can construct a graph based on the incidence structure of DSFM problem so that a vertex corresponds to a component submodular function and two vertices are connected iff the corresponding submodular functions are incident to at least one common point. An equitable coloring of this graph can be used to assign submodular functions of the same color class to a set $C_i$ in $\mathcal{C}$. This guarantees that $\mu_v^{C_i}\leq 1$ for all $C_i$ and all $v\in[N]$. Note that the maximal degree of this graph is $\triangle_*$. By directly applying Hajnal-Szemer{\'e}di's Theorem, we have the following lemma. 
\begin{lemma}\label{equitablecoloring}
There exists a balanced-partition distribution $P$ such that $\|\theta^P\|_{1,\infty} = N$, provided that $\lceil\frac{1}{\alpha}\rceil \geq \triangle_*+1$. 
\end{lemma}
As in many applications, such as image segmentation~\cite{stobbe2010efficient}, the value of $\triangle_*$ is small, and hence using a balanced-partition instead of one obtained through sampling uniformly at random may produce significantly better results. Unfortunately, finding the equitable coloring number is an NP-hard problem; still, a polynomial time algorithm for finding $\triangle_*+1$ equitable colorings was described in~\cite{kierstead2010fast}, with complexity $O(\triangle_*R^2)$. We describe a greedy algorithm that outputs a balanced-partition distribution and aims to minimize $\|\theta^P\|_{1,\infty}$ in Algorithm 4. According to our experimental results, the sampling strategy $P$ found by Algorithm 4 works better than sampling uniformly at random.
%According to our experiments, when $S_r$ is moderate, this algorithm can yield a distribution better than uniformly randomly sampling with little overhead compared to the gradient descent procedure. 
%\textcolor{blue}{Pan: If pages are limited, Algorithm 3 can be moved to appendix as well. }

\section{Using Weighted Proximal Terms} \label{sec:orthproj} 
The AP and RCDM solvers discussed in the main text are designed to solve the convex optimization~\eqref{smooth}, but also produce a solution to the discrete optimization problem~\eqref{DSFM}. To solve the discrete optimization problem~\eqref{DSFM}, another convex optimization formulation may be considered instead:
\begin{align}\label{weightedsmooth}
\min_{x\in\mathbb{R}^N}\sum_{r\in[R]} f_r(x) + \frac{1}{2} \|x\|_{2, w}^2,
\end{align}
where the choice of $w\in\mathbb{R}_{> 0}^N$ will be described later. By using the arguments in~\cite{chambolle2009total} or in Chapter 8.1-8.2 of~\cite{bach2013learning}, we know that the solution of the discrete optimization problem~\eqref{DSFM} can be obtained as $S = \{i\in[N] | x_i^* > 0\}$, where $x^*$ is a solution of~\eqref{weightedsmooth}. 

Next, we describe how a proper choice of $w$ allows one to avoid compute oblique projections in the AP and parallel CDM algorithms. If oblique projections are allowed, a good choice for $w$ may also decrease the computational complexities listed in Table~\ref{tab:results}. The results obtained based on weighted proximal terms are summarized in Table~\ref{tab:results2}. 
 \begin{table*}[h] 
\begin{tabular}{|p{1cm}<{\centering}|p{4.0cm}<{\centering}|p{5.2cm}<{\centering}|}
\hline
\multirow{ 2}{*}{} & \multicolumn{2}{c|}{Using Orthogonal Projection $\Pi_{\mathcal{B}_r}(\cdot)$}  \\
\cline{2-3}
& The Value of $w$ & Complexity  \\
\hline
AP & $w = \mu$  & $O(N\|\mu\|_1\frac{R}{K})$ \\
\hline
RCDM & $w = \frac{R-K}{R-1} 1 +  \frac{K-1}{R-1} \mu$ & $O\left(\left(\frac{R-K}{R-1} N^2 +  \frac{K-1}{R-1} N\|\mu\|_1\right)\frac{R}{K}\right)$  \\
\hline
ACDM & $w = \frac{R-K}{R-1} 1 +  \frac{K-1}{R-1} \mu$ & $O\left(\left(\frac{R-K}{R-1} N^2 +  \frac{K-1}{R-1} N\|\mu\|_1\right)^{\frac{1}{2}}\frac{R}{K}\right)$  \\
\hline
\multirow{ 2}{*}{} & \multicolumn{2}{c|}{Using Oblique Projection $\Pi_{\mathcal{B}_r, w^{1/2}}(\cdot)$} \\
\cline{2-3}
& The Value of $w$ & Complexity  \\
\hline
AP  & $w = \mu^{\frac{1}{2}}$ &  $O(\|\mu^{\frac{1}{2}}\|_1^2 \frac{R}{K})$  \\
\hline
RCDM & $w = \left(\frac{R-K}{R-1} 1 +  \frac{K-1}{R-1} \mu\right)^{\frac{1}{2}}$& $O\left( \left\|\left(\frac{R-K}{R-1} 1 +  \frac{K-1}{R-1} \mu\right)^{\frac{1}{2}}\right\|_1^2\frac{R}{K}\right)$ \\
\hline
ACDM & $w = \left(\frac{R-K}{R-1} 1 +  \frac{K-1}{R-1} \mu\right)^{\frac{1}{2}}$ & $O\left(\left\|\left(\frac{R-K}{R-1} 1 +  \frac{K-1}{R-1} \mu\right)^{\frac{1}{2}}\right\|_1\frac{R}{K}\right)$ \\
\hline
\end{tabular}
\centering 
\caption{New complexity results based on weighted proximal terms: here, complexity refers to the required number of iterations needed to achieve an $\epsilon-$optimal solution (the dependence on $\epsilon$ is the same for all algorithms and hence omitted). As before, $K$ is the parallelization parameter and it equals the number of min-norm points problems that are solved within each iteration; $K=1$ reduces to the sequential case.}
\label{tab:results2}
\end{table*}

We now analyze the new objective~\eqref{weightedsmooth} in more detail. The proof techniques used in the main text carry over to the setting involving weighted proximal terms.

By using a dual strategy similar to those described in Lemma~\ref{lem:jeg} and Lemma~\ref{dualprob}, we arrive at the dual formulation of problem~\eqref{weightedsmooth} described in the next lemma. Note that the derivation of~\eqref{weighteddistance} takes into account the underlying incidence relations. 
\begin{lemma}
The dual problem of~\eqref{weightedsmooth} reads as 
\begin{align}\label{weighteddistance}
\min_{a,y} \|a-y\|_{2, I(w^{-1}\odot\mu)}^2 \quad \text{s.t.} \quad y\in \mathcal{B}, Aa = 0,\;\text{and}\; a_{r,i} = 0,\; \forall (r, i): i\notin S_r, r\in[R].  
\end{align}
Moreover, problem \eqref{distance} may be written in a more compact form as
\begin{align}\label{weightedcompact}
\min_{y} \|Ay\|_{2, w^{-1}}^2 \quad \text{s.t.} \quad y\in \mathcal{B}.
\end{align}
For both problems, the primal and dual variables are related according to $x = - w^{-1} \odot Ay$.
\end{lemma}

\subsection{The Incidence Relations AP (IAP) Method for Solving~\eqref{weighteddistance}}
The steps of the IAP method are listed in Algorithm 5. 
\begin{table}[htb]
\centering
\begin{tabular}{l}
\hline
\label{APM}
\textbf{Algorithm 5: } \textbf{The IAP Method for Solving~\eqref{weighteddistance}} \\
\hline
\ 0: For all $r$, initialize $y_r^{(0)} \in \mathcal{B}_r$, and $k\leftarrow 0$\\
\ 1: In iteration $k$:\\
\ 2: \; For all $r\in [R]$:\\
\ 3: \quad \; $a_{r, i}^{(k+1)} \leftarrow y_{r,i}^{(k)} - \mu_i^{-1} (Ay^{(k)})_i$ for all $i\in S_r$ \\
\ 4: \quad \; $y_r^{(k+1)}\leftarrow  \Pi_{\mathcal{B}_r, w^{-1}\odot\mu}(a_r^{(k+1)})$\\
\hline
\end{tabular}
\end{table}
\\The convergence properties of Algorithm 5 can be characterized similarly as those of IAP for solving~\eqref{newdistance}. The latter relies on a finite upper bound for $\kappa_* \triangleq \sup\limits_{y\in\mathcal{Z} \cup \mathcal{B} / \Xi}\frac{d_{I(w^{-1}\odot \mu)}(y, \Xi)}{\max\{d_{I(w^{-1}\odot \mu)}(y,\mathcal{Z}), d_{I(w^{-1}\odot \mu)}(y, \mathcal{B})\}}$. 
\begin{lemma}\label{weightedkappabound}
One has $\kappa_* \leq \sqrt{\frac{\|w^{-1}\odot \mu\|_1 \|w\|_1}{2}} + 1$. When $w=\mu$, $\kappa_* \leq \sqrt{\frac{N\|\mu\|_1}{2}} + 1$.
\end{lemma}
\begin{proof}
The result follows using the same strategy as the one used to prove Lemma~\ref{kappabound}. Note that when using Lemma~\ref{lemmakappabound}, one should set $\theta$ to $I(w^{-1}\odot \mu)$ and replace $w$ by $w^{-1}$. 
\end{proof}
By setting $w = \mu$, Step 4 of Algorithm 5 reduces to orthogonal projections. In this case, based on Lemma~\ref{weightedkappabound}, Algorithm 5 requires $O(N\|\mu\|_1\log \frac{1}{\epsilon})$ iterations to achieve an $\epsilon$-optimal solution. By setting $w = \mu^{\frac{1}{2}}$ for all $i\in [R]$, Step 4 of Algorithm 5 reduces to the projections $\Pi_{\mathcal{B}_r, w^{\frac{1}{2}}}(\cdot)$. In this case, Algorithm 5 requires $O\left(\left\|\mu^{\frac{1}{2}}\right\|^2\log \frac{1}{\epsilon}\right)$ iterations to achieve an $\epsilon$-optimal solution. The latter result is slightly better because $\left\|\mu^{\frac{1}{2}}\right\|^2\leq N \|\mu\|_1$.

\subsection{A Parallel RCD Method for Solving~\eqref{weightedcompact} with Uniform Sampling Strategies}
As discussed in Section~\ref{sec:CDMs}, RCDM strongly depends on an $\alpha$-proper distribution $P$ that characterizes the parallel coordinate sampling strategy. In what follows, we choose $P$ to be a uniform distribution. From Lemma~\ref{eqnorm}, we know that when $P$ is uniform, one has $\theta_r^P = \frac{K - 1}{R-1}\mu + \frac{R - K}{R-1}1$ for all $r\in[R]$, where $K$ denotes the number of projections computed in parallel as part of each iteration. In Algorithm 1, $\theta_r^P$ defines the normed space over which to minimize $g(y)$. As our goal is to minimize $g_{w}(y) = \frac{1}{2} \|Ay\|_{2,w^{-1}}^2$, the vector used to define the normed space is 
$$\nu= w^{-1}\odot \theta_r^P = w^{-1} \odot (\frac{K - 1}{R-1}\mu + \frac{R - K}{R-1}1).$$ 
The parallel RCDM procedure in this setting is described in Algorithm 6. 
 \begin{table}[htb]
\centering
\begin{tabular}{l}
\hline
\label{PRCD}
\textbf{Algorithm 6: } \textbf{Parallel RCDM with Uniform Sampling for Solving ~\eqref{weightedcompact}} \\
\hline
\ \textbf{Input}: $\mathcal{B}$, $K$ \\
\ 0: Initialize $y^{(0)}\in\mathcal{B}$, $k\leftarrow 0$\\
\ 1: Do the following steps iteratively until the dual gap $<\epsilon$:\\
\ 2: \;\quad Uniformly sample $C_{ i_k} \subseteq [R]$ so that $|C_{ i_k}| = K$. \\
\ 3: \;\quad For $r\in C_{i_k}$: \\
%\ 4: \;\quad $z_{r, i}^{(k)} \leftarrow y_{r,i}^{(k)} - (\nabla_r g(y^{(k)}))_i/\theta_{r,i}^P$ for $i\in S_r$,\; $z_{r, i}^{(k)}\leftarrow 0$ for $i\notin S_r$\\
\ 4: \;\quad\quad  $y_r^{(k+1)}\leftarrow \Pi_{\mathcal{B}_r, \nu} (y_r^{(k)} - (\nu^{-1}) \odot \nabla_r g_{w}(y) \ )$ \\
\ 5: \;\quad Set $y_r^{(k+1)}\leftarrow y_r^{(k)}$ for $r\not \in C_{i_k}$ \\
\ 6: \;\quad $k\leftarrow k+1$\\
\ 7: Output $y^{(k)}$ \\
\hline
\end{tabular}
\end{table}

Similarly to what was done in Lemma~\ref{skewstrongconv}, we can establish weak strong convexity of $g_w(y)$ with respect to the norm $\|\cdot\|_{2,\nu}$ by invoking Lemma~\ref{lemmakappabound}.
\begin{lemma}~\label{weightedskewstrongconv}
For any $y\in \mathcal{B}$, let $y^* = \arg\min_{\xi\in\Xi}\|\xi- y\|_{2,\nu}^2$. Then,
$$\|Ay - Ay^*\|_{2,w^{-1}}^2 \geq \frac{2}{\|w\|_1\|\nu\|_{1}}\|y-y^*\|_{2, \nu}^2.$$ 
\end{lemma}  
Therefore, using a strategy similar to the one outlined in the proof of Theorem~\ref{PCDrate}, the convergence rates of Algorithm 6 can be derived as summarized in the next theorem. 
\begin{theorem} \label{weightedPCDrate}
At each iteration of Algorithm 6, $y^{(k)}$ satisfies 
\begin{align*}
&\mathbb{E}\left[g_w(y^{(k)})- g_w(y^*) + \frac{1}{2}d_{I(\nu)}^2(y^{k}, \xi) \right]  \\
 &\quad \leq \left[1- \frac{4K}{R(\|w\|_1\|\nu\|_{1} + 2)}\right]^{k}\left[g_w(y^{(0)})- g_w(y^*) + \frac{1}{2}d_{I(\nu)}^2(y^{0}, \xi) \right].
\end{align*}
\end{theorem}

By setting $w = \frac{K - 1}{R-1}\mu + \frac{R - K}{R-1}1$, we reduce the projections in Step 4 of Algorithm 6 to orthogonal projections. In this case, based on Theorem~\ref{weightedPCDrate}, Algorithm 6 requires $O\left(\left(\frac{K - 1}{R-1}N\|\mu\|_1 + \frac{R - K}{R-1}N^2\right)\frac{R}{K}\log\frac{1}{\epsilon}\right)$ iterations to achieve an $\epsilon$-optimal solution. 

By setting $w = \left(\frac{K - 1}{R-1}\mu + \frac{R - K}{R-1}\right)^{1/2}$ for all $i\in [R]$, the projections in Step 4 of Algorithm 6 reduce to oblique projections $\Pi_{\mathcal{B}_r, w^{\frac{1}{2}}}(\cdot)$. In this case, Algorithm 6 requires $O\left(\left\| \left(\frac{K - 1}{R-1}\mu_i + \frac{R - K}{R-1}\right)^{1/2}\right\|_1^2\log \frac{1}{\epsilon}\right)$ iterations to achieve an $\epsilon$-optimal solution, which is slightly better than the previous case. The accelerated methods can be analyzed in the same manner. 

\subsection{Simulations}
We now describe simulation results that empirically evaluate Algorithms 5 and 6. The DSFM problem is designed as follows. We consider $N=100$ vertices. The unary potentials of different elements are iid standard Gaussian variables. We construct a network over these vertices based on the Barab$\acute{a}$si-Albert model (BA)~\cite{albert2002statistical}, initialized with a single edge between vertices 1 and 2. Each edge in the network gives a pairwise potential for the corresponding vertices. We use the BA model so that the number of incidence relations corresponding to different vertices vary to a large extent. As we are using weighted proximal terms, the continuous objectives are not consistent for different $w$. However, here, we are only interested in generating solutions for the discrete problem~\eqref{DSFM} and thus regard the discrete gap $\nu_d$ as the relevant metric for characterizing convergence properties. The following results are obtained from $100$ independent experiments.

In IAP (Algorithm 5), we set $w\in\{1, \mu, \mu^{1/2}\}$, corresponding to three cases: \emph{unweighted proximal term + oblique projections}, \emph{weighted proximal term + orthogonal projections}, \emph{weighted proximal term + oblique projections}, respectively. In RCDM-U (Algorithm 6), we set $w\in\{1, \frac{K-1}{R-1}\mu + \frac{R-K}{R-1}1, (\frac{K-1}{R-1}\mu + \frac{R-K}{R-1}1)^{1/2}\}$, corresponding to the same three cases. We control the number of parallel projection operations in each iteration by choosing $K\in\{10, 20, 30, 40, 50\}$. Figure~\ref{fig:weightedprim} shows the convergence curve of the discrete gap for different solvers and different choices of $w$. We only plotted results for $K=10, 50$ as other values of $K$ produce similar patterns. For both IAP and RCDM-U, when $w$ corresponds to the weighted proximal term + orthogonal projections case, we obtain the best convergence rates. The value $w=1$, corresponding to the case unweighted proximal term + oblique projections, results in the worst convergence rates. Albeit somewhat inconsistent with the results listed in Table~\ref{tab:results2}, the simulations simply imply that using weighted proximal terms can reduce the complexity of the algorithms at hand and that the weighted proximal term with orthogonal projections in the inner loop may represent the best choice in practice. 

\begin{figure}[t]
\centering
\includegraphics[trim={0cm 0cm 0cm 0cm},clip,width=.24\textwidth]{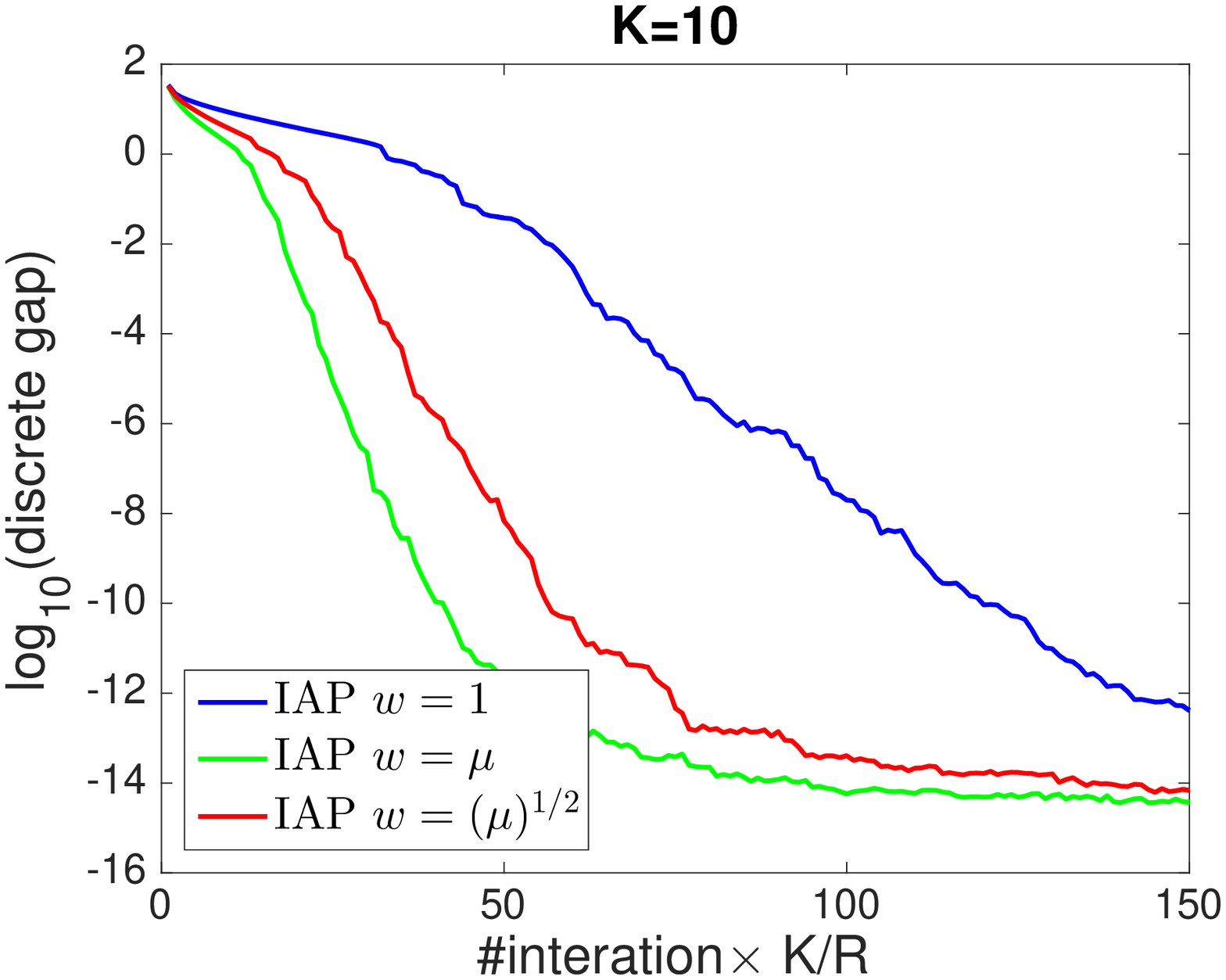}
\includegraphics[trim={0cm 0cm 0cm 0cm},clip, width=.24\textwidth]{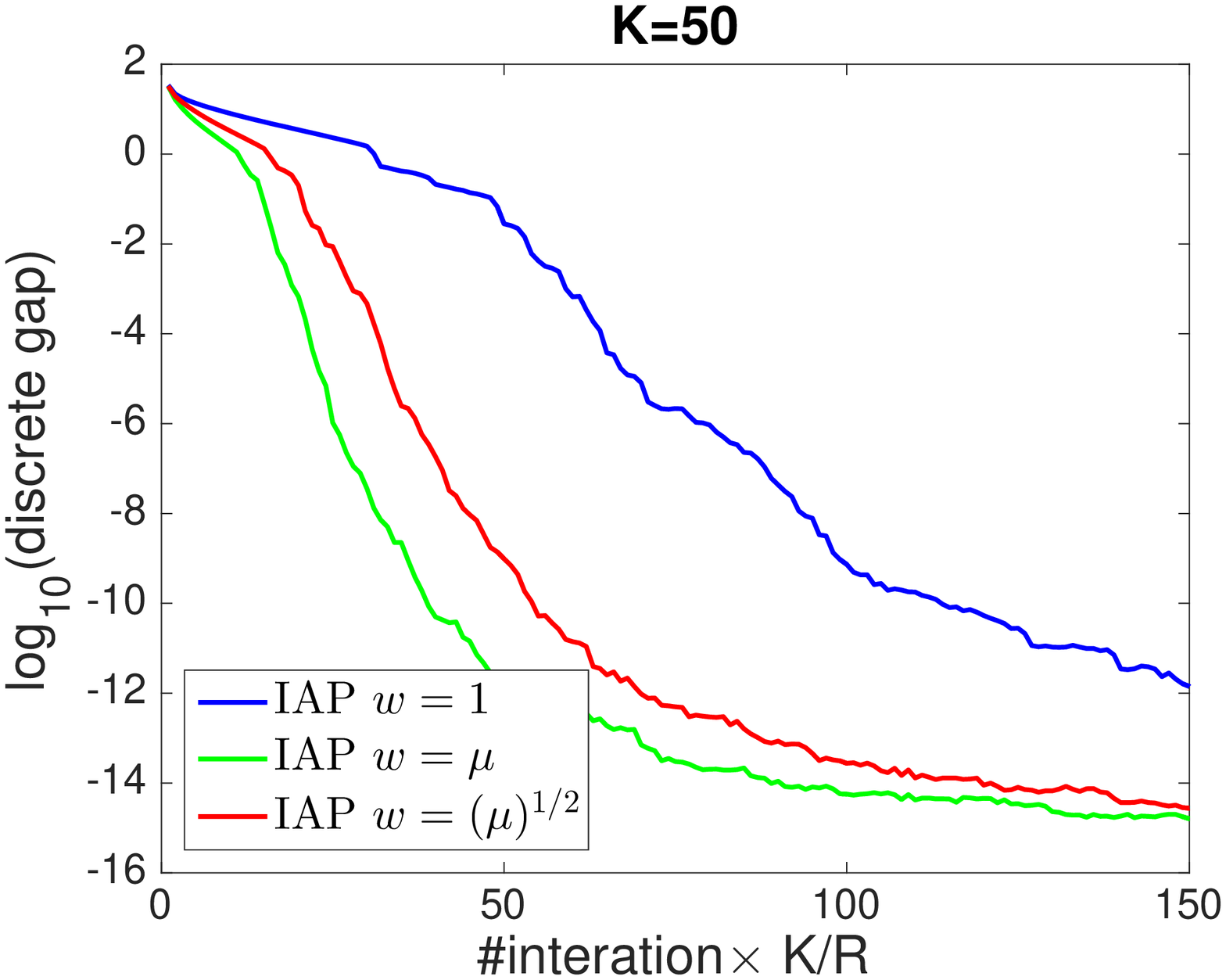}
\includegraphics[trim={0cm 0cm 0cm 0cm},clip,width=.24\textwidth]{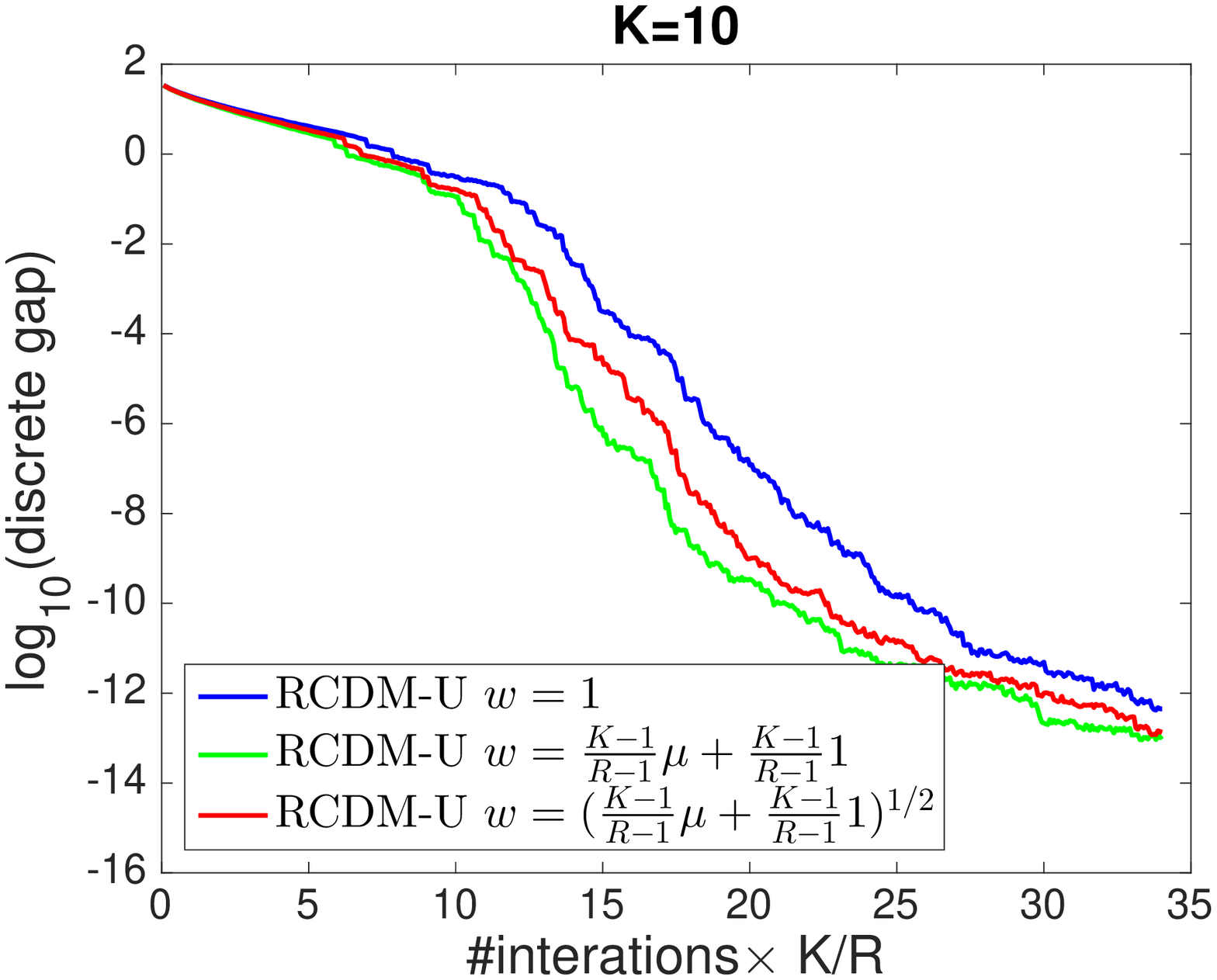}
\includegraphics[trim={0cm 0cm 0cm 0cm},clip, width=.24\textwidth]{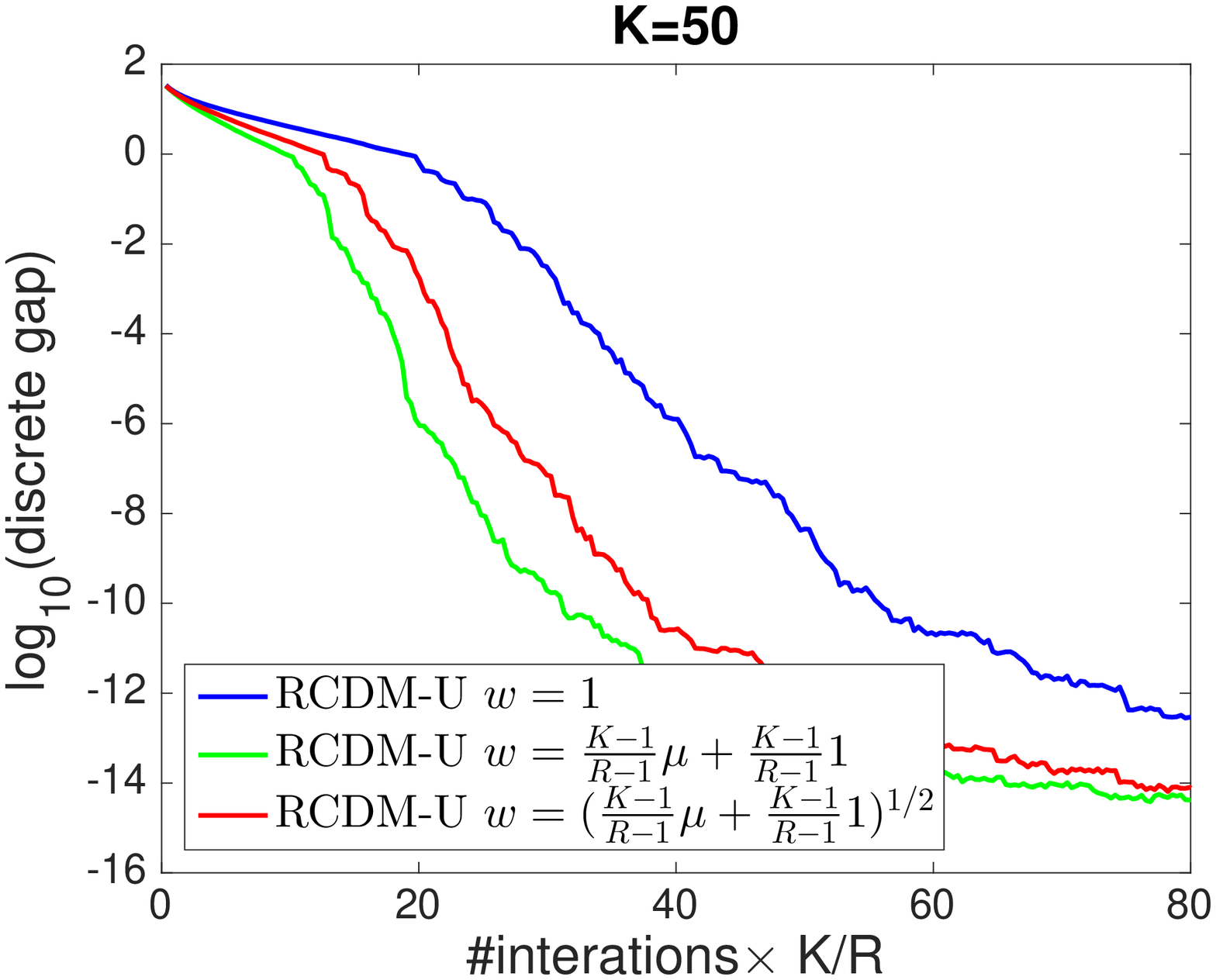}\\
\caption{Simulations for Algorithm 5 and 6: $\log_{10}(\text{discrete gap})$ vs (number of iterations $\times K/R$). }
\label{fig:weightedprim}
\end{figure}

In Table~\ref{tab:weightedprimalsim}, we also list the number of iterations needed by different solvers to obtain a solution for the discrete problem~\eqref{DSFM}. Again, the $w$ corresponding to the weighted proximal term + orthogonal projections case results in the smallest number of iterations, while the $w$ corresponding to the unweighted proximal term + oblique projections case results in the largest number of iterations. Note that as $K$ increases, the number of iterations $\times K/R$ in IAP does not change as IAP is fully parallelizable, while the number of operations in RCDM-U slightly increases due to the overlapping incidence sets of different submodular functions. 

\begin{table}[t]
\scriptsize{
\begin{tabular}{p{1.2cm}<{\centering}|p{3.0cm}<{\centering}| p{0.45cm}<{\centering}| p{0.45cm}<{\centering}| p{0.45cm}<{\centering}| p{0.45cm}<{\centering}| p{0.45cm}<{\centering}| p{0.45cm}<{\centering}| p{0.45cm}<{\centering}| p{0.45cm}<{\centering}|  p{0.45cm}<{\centering}| p{0.45cm}<{\centering}}
\hline
\multirow{2}{*}{Solvers} & \multirow{2}{*}{$w$} &  \multicolumn{2}{c|}{$K = 10$} &  \multicolumn{2}{c|}{$K = 20$}&  \multicolumn{2}{c|}{$K = 30$}&  \multicolumn{2}{c|}{$K = 40$}&  \multicolumn{2}{c}{$K = 50$} \\
\cline{3-12}
&& MN & MD& MN & MD& MN & MD& MN & MD& MN & MD\\
\hline
\multirow{3}{*}{IAP} & 1 & 109 & 103 & 109 & 103 & 109 & 103 & 109 & 103 & 109 & 103 \\
\cline{2-12}
& $\mu$ & 43 & 34 & 43 & 34 & 43 & 34 & 43 & 34 & 43 & 34 \\
\cline{2-12}
& $\mu^{1/2}$ & 59 & 50 & 59 & 50 & 59 & 50 & 59 & 50 & 59 & 50 \\
\hline
\multirow{3}{*}{RCDM-U} & 1 & 27 & 22 & 34 & 28 & 43 & 38 & 51 & 46 & 54 & 49 \\
\cline{2-12}
& $\frac{K-1}{R-1}\mu + \frac{R-K}{R-1}1$ & 22 & 17 & 25 & 20 & 29 & 24 & 32 & 24 & 33 & 25 \\
\cline{2-12}
&  $\left(\frac{K-1}{R-1}\mu + \frac{R-K}{R-1}1\right)^{1/2}$ & 25 & 19 & 28 & 23 & 33 & 28 & 37 & 31 & 38 & 32 \\
\hline
\end{tabular}}
\caption{The number of iterations $\times K/R$ needed to find an optimal solution to the discrete problem~\eqref{DSFM}. MN: mean; MD: median.}
\label{tab:weightedprimalsim}
\end{table}

\section{Supplementary experiments}\label{supexp}
\textbf{Semi-supervised learning over hypergraphs}. We also evaluate the proposed approaches over the 20Newsgroups from the University of California Irvine (UCI) data repository. This dataset is used as a benchmark example for evaluating semisupervised learning algorithms over hypergraphs~\cite{hein2013total,yadati2018hypergcn}. Here, for simplicity, we focused on binary classification tasks and thus paired the four 20Newsgroups classes, so that one group includes ``Comp.'' and ``Sci'', and the other one includes ``Rec.'' and ``Talk''. The 20Newsgroups dataset consists of categorical features and we adopt the same approach as the one described in~\cite{hein2013total} to construct hyperedges: each feature corresponds to one hyperedge and contributes one submodular function to the decomposition. Hence, 20Newsgroups contains $N=16242$ elements and $R=100$ submodular functions.
 
\begin{figure}[t]
\centering
\includegraphics[trim={0cm 0cm 0cm 0cm},clip,width=.24\textwidth]{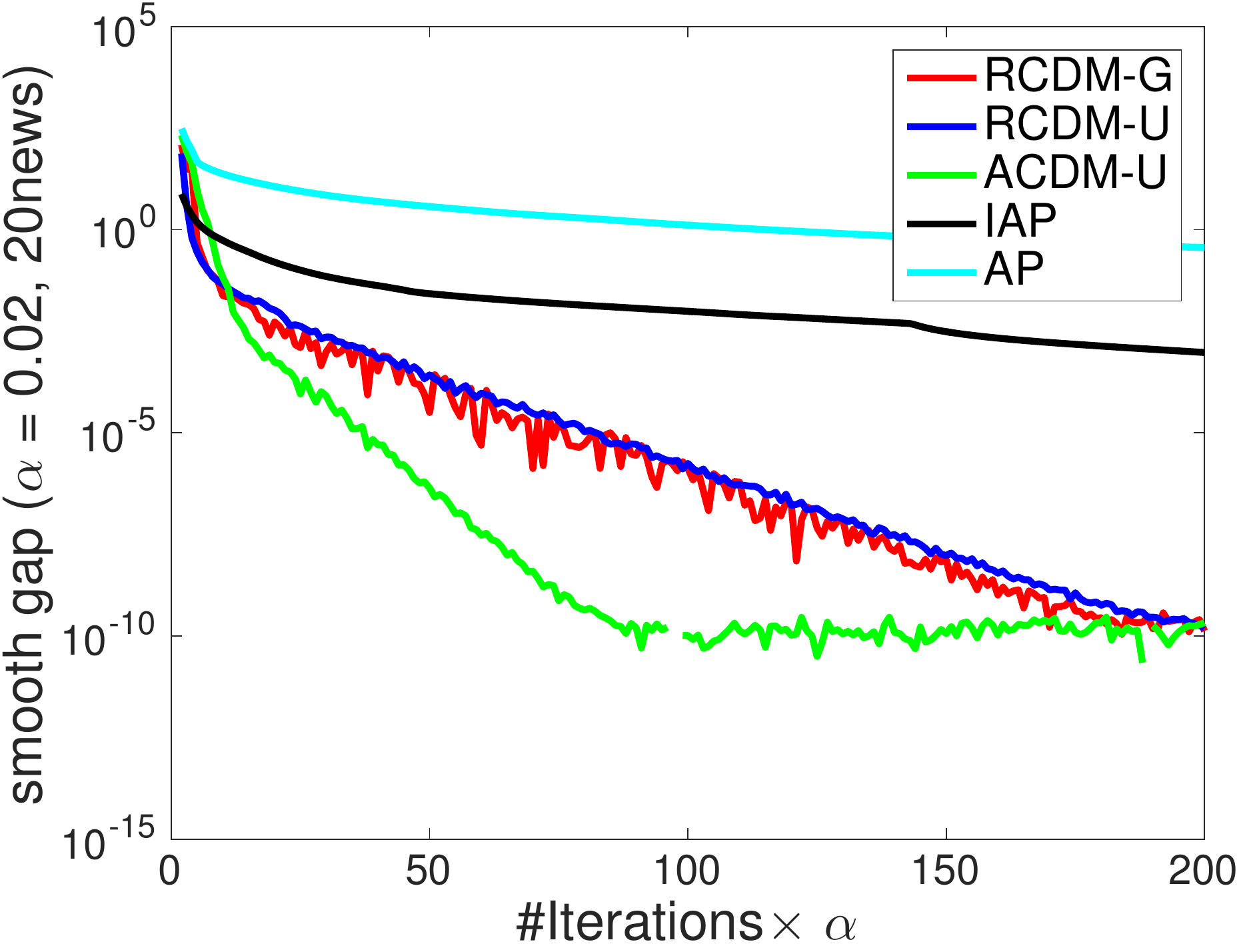}
\includegraphics[trim={0cm 0cm 0cm 0cm},clip, width=.24\textwidth]{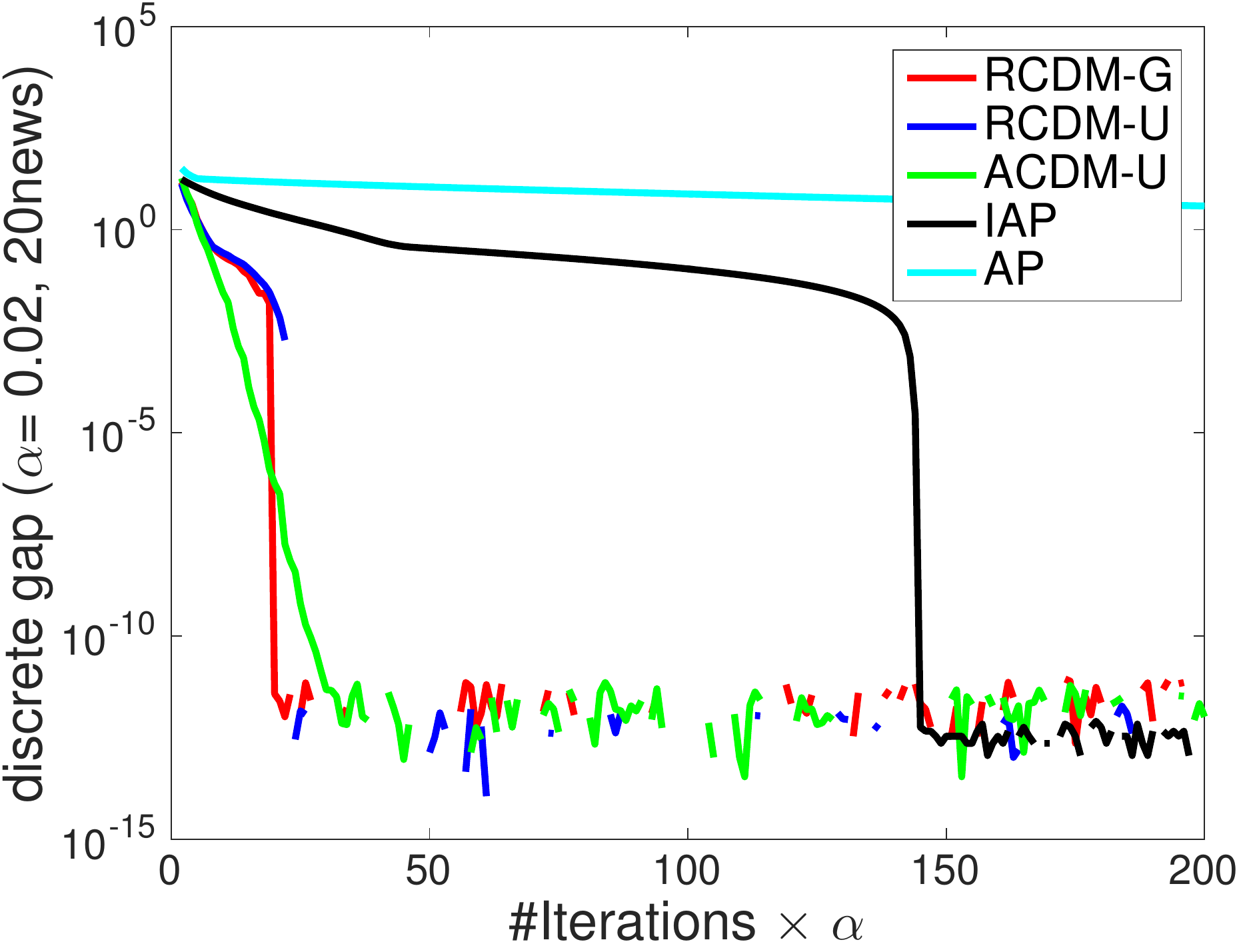}
\includegraphics[trim={0cm 0cm 0cm 0cm},clip,width=.24\textwidth]{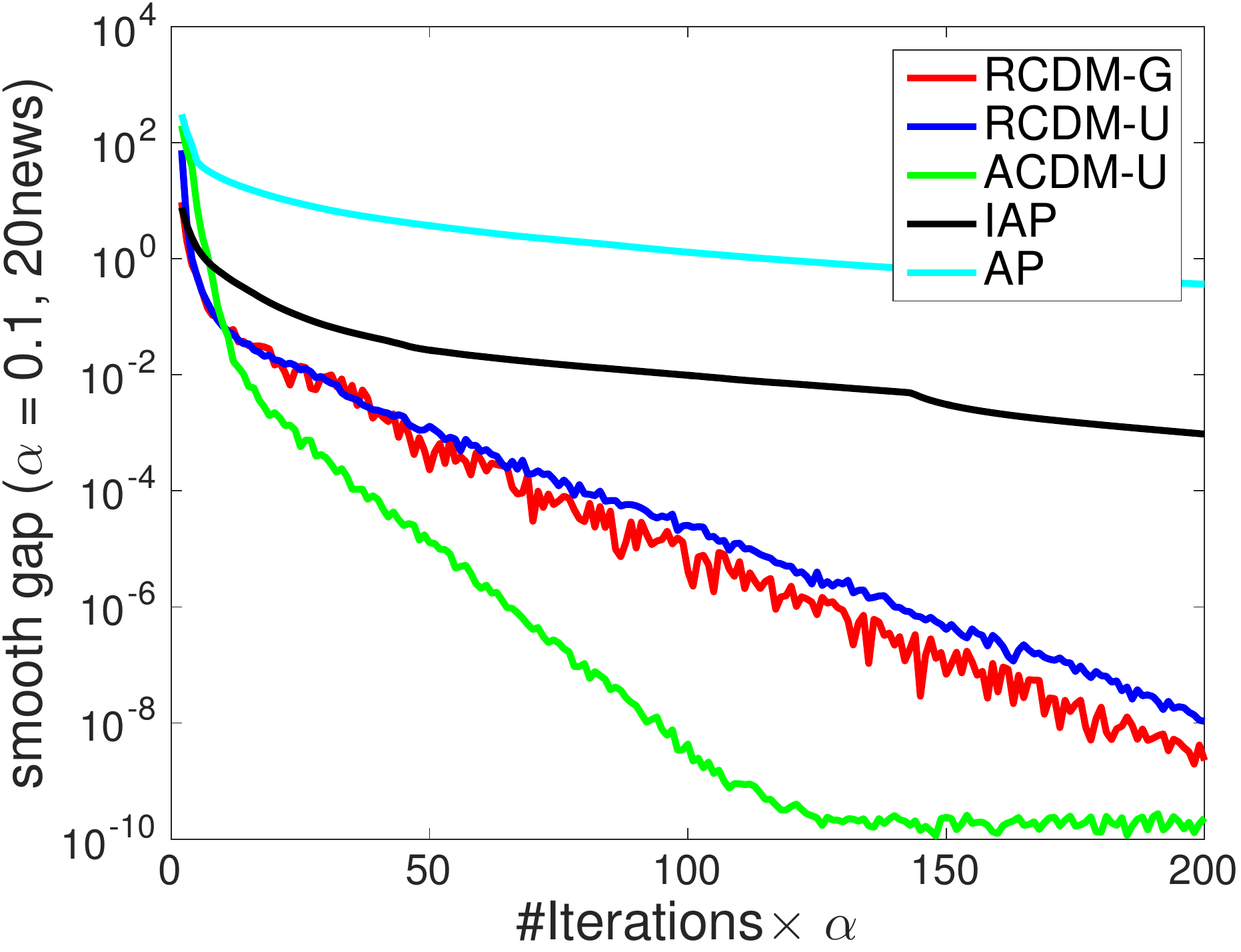}
\includegraphics[trim={0cm 0cm 0cm 0cm},clip, width=.24\textwidth]{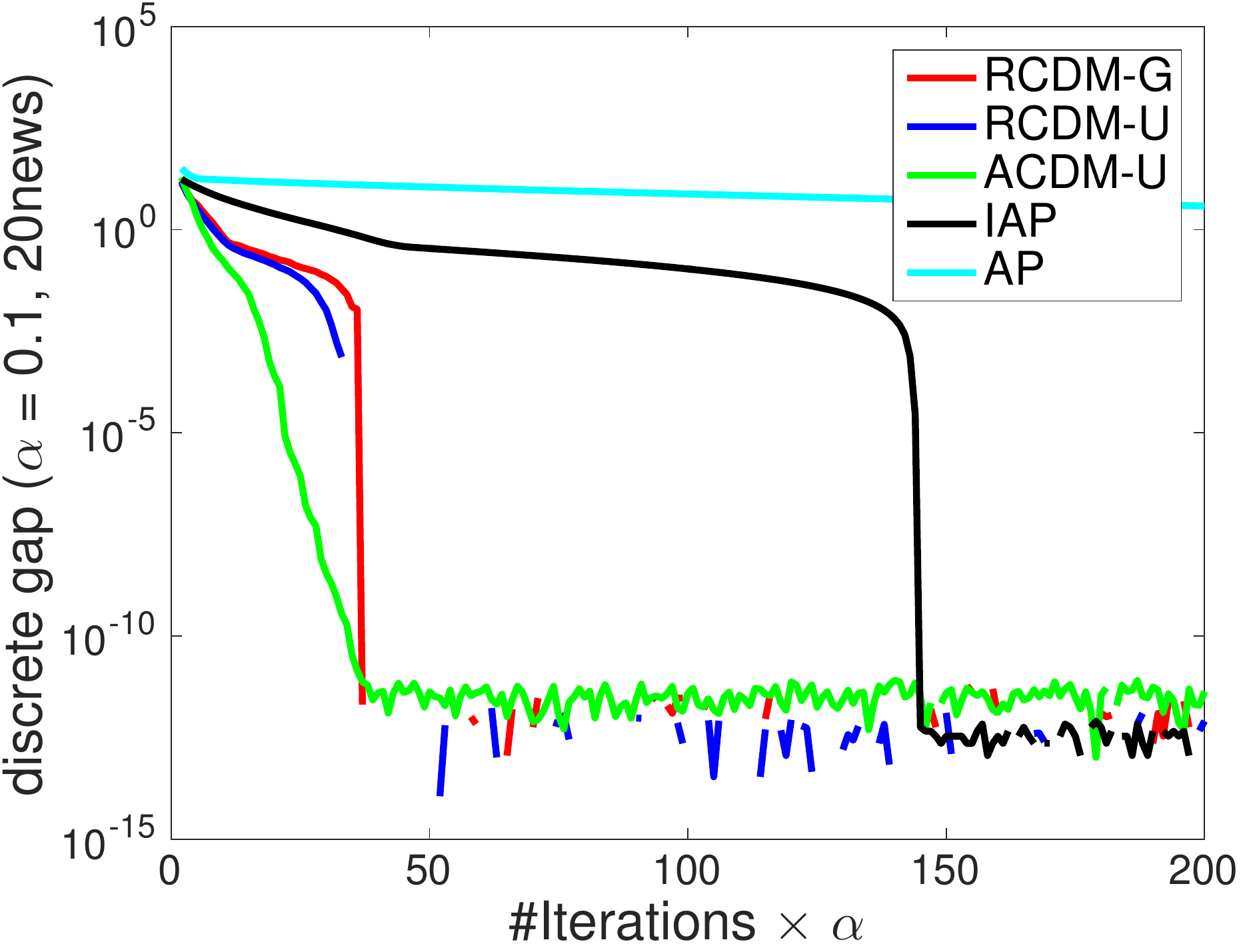}\\
\caption{20Newsgrounp: Smooth/discrete gap vs the (number of iterations $\times \alpha$). }
\label{20news}
\end{figure}

In the experiments for 20Newsgroups, we uniformly at random picked $200$ elements and set their corresponding components in $x_0$ of equation~\eqref{semisupervise} to the true labels and set all other entries to zero. Figure~\ref{20news} shows the results of the experiments pertaining to 20Newsgroup. We compared the convergence rate of different algorithms for different values of the parameter $\alpha\in\{0.02,0.1\}$.  The value on the horizontal axis, $\#$ iterations $\times \alpha$, equals the total number of projections, scaled by $R$. The results are averaged over $10$ independent experiments. Once again, we observe that CD-based methods outperform AP-based methods. ACDM-U offers the best performance among all CD-based methods and IAP significantly outperforms AP. Similarly, RCDM-G has better performance than RCDM-U, due to the use of the greedy algorithm for the sampling procedure.

\end{document}